\documentclass[10pt,journal,compsoc]{IEEEtran}
\ifCLASSOPTIONcompsoc
  \usepackage[nocompress]{cite}
\else
  \usepackage{cite}
\fi
\ifCLASSINFOpdf
\else
  \usepackage[dvips]{graphicx}
\fi
\usepackage{amsmath,amssymb,amsthm}
\ifCLASSOPTIONcompsoc
 \usepackage[caption=false,font=footnotesize,labelfont=sf,textfont=sf]{subfig}
\else
 \usepackage[caption=false,font=footnotesize]{subfig}
\fi

\usepackage{graphicx}
\usepackage{bm}
\usepackage{booktabs}
\usepackage[dvipsnames]{xcolor}
\usepackage{xfrac}
\usepackage{ragged2e} 
\usepackage{titletoc}
\usepackage{algorithm,algpseudocode}

\algnewcommand\algorithmicto{\textbf{to}}
\algrenewtext{For}[3]%
{\algorithmicfor\ $#1 \gets #2$ \algorithmicto\ $#3$ \algorithmicdo}

\DeclareMathOperator*{\E}{\mathbb{E}}
\DeclareMathOperator*{\rfield}{\mathbb{R}}
\newcommand{\B}{\mathcal{B}}

\newcommand{\setX}{\mathcal{X}}
\newcommand{\setF}{\mathcal{F}}
\newcommand{\setD}{\mathcal{D}}
\newcommand{\setO}{\mathcal{O}}
\newcommand{\setT}{\mathcal{T}}
\newcommand{\DL}{\setD_L}
\newcommand{\DU}{\setD_U}
\newcommand{\nuu}{{n_u}}
\newcommand{\nl}{{n_l}}
\newcommand{\nc}{{n_c}}
\DeclareMathOperator*{\argmax}{argmax}
\DeclareMathOperator*{\argmin}{argmin}

\newcommand{\norm}[1]{\left\lVert#1\right\rVert}
\newcommand{\abs}[1]{\left|#1\right|}
\newcommand{\phil}{\phi_{\frac{1}{2\kappa}}}
\newcommand{\inprod}[1]{\left\langle#1\right\rangle}
\newcommand{\xx}[1]{\theta_{#1}}
\newcommand{\xt}{\xx{t}}
\newcommand{\Gxy}[1]{\mathcal{L}(\xx{#1}, v_{#1})}
\newcommand{\Gxyt}{\Gxy{t}}
\newcommand{\gradGxy}[1]{\nabla_{\theta}\Gxy{#1}}
\newcommand{\gradGxyt}{\gradGxy{t}}
\newcommand{\hatx}[1]{\hat{\theta}_{#1}}
\newcommand{\hatxt}{\hatx{t}}
\newcommand{\inprodG}[1]{\inprod{g_{#1} - \gradGxy{#1}, \hatx{#1} - \xx{#1}}}
\newcommand{\lzo}{\ell_{0/1}}
\newcommand{\lce}{\ell_\text{CE}}
\newcommand{\lsce}{\ell_\text{CE}}
\def \Bse{\B_{sem}}
\def \Bsex{\Bse(x)}
\newcommand{\dinf}{d_\infty}
\newcommand{\fnew}{{\tilde{f}}}
\newcommand{\fj}{f^{(j)}}
\newcommand{\setFbu}{\setF_{\beta,\nu}}

\newcommand{\RL}{\widehat{R}_{\DL}(f)}
\newcommand{\RU}{\widehat{R}_{\DU}(f)}
\newcommand{\RB}{R^{0/1}_{\B\text{-robust}}(f)}

\newcommand{\ie}{\textit{i.e.}}
\newcommand{\eg}{\textit{e.g.}}

\theoremstyle{plain}
\newtheorem{thm}{Theorem}
\newtheorem{lem}{Lemma}
\newtheorem{prop}{Proposition}

\newtheorem{fact}{Fact}

\theoremstyle{definition}
\newtheorem{rem}{Remark}
\newtheorem{defn}{Definition}

\newcommand{\namedthm}[3]{\theoremstyle{plain}
   \newtheorem*{thm#1}{Restate of Theorem #1}\begin{thm#1}{(#3).}#2\end{thm#1}}
\newcommand{\namedthmn}[2]{\theoremstyle{plain}
   \newtheorem*{thmn#1}{Restate of Theorem #1}\begin{thmn#1}#2\end{thmn#1}}
\newcommand{\namedlem}[2]{\theoremstyle{plain}
   \newtheorem*{lem#1}{Restate of Lemma #1}\begin{lem#1}#2\end{lem#1}}
\newcommand{\namedprop}[3]{\theoremstyle{plain}
   \newtheorem*{prop#1}{Restate of Proposition #1}\begin{prop#1}{(#3).}#2\end{prop#1}}

\newcommand{\Eqref}[1]{Eq.~(\ref{#1})}
\newcommand{\Tabref}[1]{Table~\ref{#1}}
\newcommand{\Figref}[1]{Figure~\ref{#1}}
\newcommand{\Algref}[1]{Algorithm~\ref{#1}}
\newcommand{\Thmref}[1]{Theorem~\ref{#1}}
\newcommand{\Propref}[1]{Proposition~\ref{#1}}
\newcommand{\Factref}[1]{Fact~\ref{#1}}
\newcommand{\Lemref}[1]{Lemma~\ref{#1}}
\newcommand{\Defref}[1]{Definition~\ref{#1}}

\newcommand{\Apdref}[1]{Appendix~\ref{#1}}

\definecolor{ballblue}{HTML}{338EA7}
\definecolor{lightseagreen}{HTML}{759D39}
\definecolor{lightred}{HTML}{DD7769}
\definecolor{softred}{HTML}{FE8A71}
\definecolor{softblue}{HTML}{63ACE5}

\usepackage[colorlinks,linkcolor=ballblue,citecolor=ballblue]{hyperref}

\newcommand{\best}[1]{{\color{softred} \textbf{#1}}}
\newcommand{\secbest}[1]{{\color{softblue} \textbf{#1}}}

\newcommand\scalemath[2]{\scalebox{#1}{\mbox{\ensuremath{\displaystyle #2}}}}

\usepackage{url}

\begin{document}
%
\title{MaxMatch: Semi-Supervised Learning with\\ Worst-Case Consistency}
%
%
%
%

\author{Yangbangyan Jiang,
	Xiaodan~Li,
	Yuefeng~Chen,
	Yuan~He,
	Qianqian~Xu*,~\IEEEmembership{Senior Member,~IEEE,}
	Zhiyong~Yang,
	Xiaochun~Cao,~\IEEEmembership{Senior Member,~IEEE,}
	and~Qingming~Huang*,~\IEEEmembership{Fellow,~IEEE}
\IEEEcompsocitemizethanks{
\IEEEcompsocthanksitem Y. Jiang is with State Key Laboratory of Information Security, Institute of Information Engineering, Chinese Academy of Sciences, Beijing 100093, China, and also with School of Cyber Security, University of Chinese Academy of Sciences, Beijing 100049, China (E-mail: jiangyangbangyan@iie.ac.cn).\protect\\
\IEEEcompsocthanksitem X. Li, Y. Chen and Y. He are with the Security Department of Alibaba Group, Hangzhou 311121, China (E-mail: fiona.lxd@alibaba-inc.com, yuefeng.chenyf@alibaba-inc.com, heyuan.hy@alibaba-inc.com).\protect\\
\IEEEcompsocthanksitem Q. Xu is with the Key Laboratory of Intelligent Information Processing, Institute of Computing Technology, Chinese Academy of Sciences, Beijing 100190, China (E-mail: xuqianqian@ict.ac.cn).\protect\\
\IEEEcompsocthanksitem Z. Yang is with the School of Computer Science and Technology, University of Chinese Academy of Sciences, Beijing 101408, China (E-mail: yangzhiyong21@ucas.ac.cn).\protect\\
\IEEEcompsocthanksitem X. Cao is with School of Cyber Science and Technology, Shenzhen Campus, Sun Yat-sen University, Shenzhen 518107, China (E-mail: caoxiaochun@mail.sysu.edu.cn).\protect\\
\IEEEcompsocthanksitem Q. Huang is with the School of Computer Science and Technology, University of Chinese Academy of Sciences, Beijing 101408, China, also with the Key Laboratory of Big Data Mining and Knowledge Management (BDKM), University of Chinese Academy of Sciences, Beijing 101408, China, also with the Key Laboratory of Intelligent Information Processing, Institute of Computing Technology, Chinese Academy of Sciences, Beijing 100190, China, and also with Peng Cheng Laboratory, Shenzhen 518055, China (E-mail: qmhuang@ucas.ac.cn).\protect\\
\IEEEcompsocthanksitem *Corresponding author.
}
}

%
%

\markboth{IEEE Transactions on Pattern Analysis and Machine Intelligence}%
{Jiang \MakeLowercase{\textit{et al.}}: Bare Demo of IEEEtran.cls for Computer Society Journals}
%



\IEEEtitleabstractindextext{%
\begin{abstract}
\justifying
In recent years, great progress has been made to incorporate unlabeled data to overcome the  inefficiently supervised problem via semi-supervised learning (SSL). Most state-of-the-art models are based on the idea of pursuing consistent model predictions over unlabeled data toward the input noise, which is called \textit{consistency regularization}. Nonetheless, there is a lack of theoretical insights into the reason behind its success. To bridge the gap between theoretical and practical results, we propose a worst-case consistency regularization technique for SSL in this paper. Specifically, we first present a generalization bound for SSL consisting of the empirical loss terms observed on labeled and unlabeled training data separately. Motivated by this bound, we derive an SSL objective that minimizes the largest inconsistency between an original unlabeled sample and its multiple augmented variants. We then provide a simple but effective algorithm to solve the proposed minimax problem, and theoretically prove that it converges to a stationary point. Experiments on five popular benchmark datasets validate the effectiveness of our proposed method.
\end{abstract}

}

\maketitle

\IEEEdisplaynontitleabstractindextext

%
\IEEEpeerreviewmaketitle

\IEEEraisesectionheading{\section{Introduction}\label{sec:intro}}

%
%
%
%
\IEEEPARstart{I}{n} the past decades, deep neural networks have achieved great success on various tasks with the emergence of large-scale well-annotated datasets like ImageNet \cite{imagenet} or MSCOCO \cite{mscoco}. With sufficient labels, deep networks are able to capture the underlying patterns and generalize well to unseen data, leading to strong performance. Yet collecting such datasets requires numerous human annotations, which involves considerable time and effort in practice. This promotes the development of learning with less supervision, the representative paradigms of which are semi-supervised learning (SSL) and unsupervised learning. When limited labeled data are available, semi-supervised methods could take advantage of the valuable information contained in numerous unlabeled samples to improve generalization performance. Accordingly, there has been a growing trend in the applications of SSL ranging from image classification \cite{FixMatch}, object detection \cite{app1}, semantic segmentation \cite{app2}, image translation \cite{app3} to gait recognition \cite{app4}.

Among existing SSL frameworks, consistency regula-rization-based models have attracted much attention due to their surprising performance \cite{MT, VAT, MixMatch, FixMatch, UDA, ReMixMatch}. On image classification tasks, state-of-the-art models of this kind could even achieve comparable results compared with fully supervised settings. The key idea of such models lies in requiring the model's prediction to be invariant towards input noise over unlabeled samples. The main difference between these methods is the choice of the noise. For example, the early trial approach $\Pi$-Model \cite{Pi_model} simply minimizes the discrepancy between two predicted scores, with randomness from Dropout and Gaussian noise, of the same unlabeled data point. Then random flipping and cropping become the standard perturbation scheme, and the models are further equipped with a wide spectrum of techniques such as teacher-student model \cite{MT} and MixUp interpolation \cite{MixMatch}. After that, researchers find that a stronger noise injection strategy (\eg, aggressive data augmentation) which induces extra inductive bias could further improve the performance, resulting in much more competitive models like Unsupervised Data Augmentation (UDA) \cite{UDA}, ReMixMatch \cite{ReMixMatch} and FixMatch \cite{FixMatch}.

\begin{figure*}[t]
	\centering
	\includegraphics[width=0.97\linewidth]{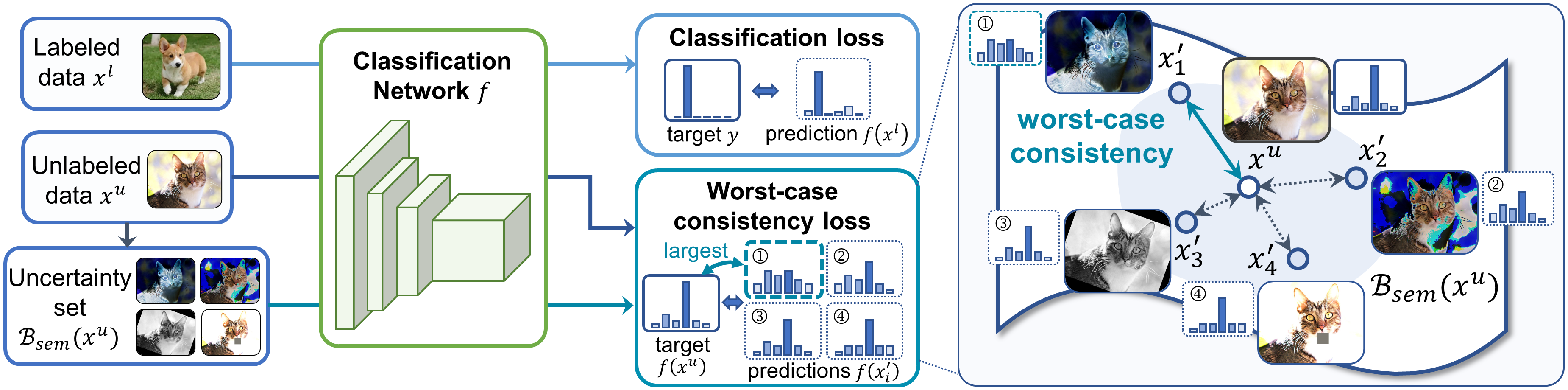}
	\caption{Overview of our proposed framework. Different from existing models, we consider the worst-case consistency over a class-invariant semantic uncertainty set $\Bsex$ around the original unlabeled data point $x$.}
	\label{fig:framework}
\end{figure*}

Despite the promising performance of the consistency regularization-based approaches, there are very few studies providing theoretical analysis to understand how such a technique could improve generalization performance of SSL models. As the first trial, \cite{bound2} analyzes the effect of consistency regularization only under the setting of self-training, which requires a pseudo-labeler and does not apply to a large number of methods like $\Pi$-Model, MixMatch and UDA. Different from these studies, inspired by a recent work \cite{bound}, we establish a generalization bound to explore the connection between consistency regularization and generalization performance. On top of this, we further propose a new consistency regularization term with a minimax formulation.

Specifically, we first study the upper bound of generalization error for the consistency regularization-based SSL framework by means of the Rademacher complexity. According to our theoretical analysis, the error bound could be decomposed into the sum of the worst-case consistency between unlabeled data and augmented samples around them, the empirical classification error on labeled data, and a term depending on the choice of the model. Then it could be further extended into a practical bound for deep convolutional networks. Next, this bound induces a new SSL objective which minimizes the largest inconsistency between an original unlabeled sample and its multiple augmented variants, as shown in \Figref{fig:framework}. To solve this problem, we present a simple solution, called \textit{MaxMatch}, which is proved to converge to a stationary point of our problem. Such worst-case consistency regularization could be easily plugged into existing methods. In summary, the contributions of this paper are three-fold:
\begin{itemize}
	\item We provide theoretical support for applying consistency regularization under the SSL setting, in terms of a generalization upper bound. Based on this, we further provide a practical bound for the widely-used deep convolutional networks.
	\item The proposed theoretical result shows that the worst-case consistency is directly related to the robust generalization error. Motivated by this, we propose a novel method called MaxMatch, which minimizes the largest inconsistency loss over an uncertainty set generated by data augmentation.
	\item We then propose a simple but effective optimization method for our proposed method. Furthermore, we show that the proposed optimization algorithm converges to a stationary point with high probability.
\end{itemize}
Practically, experiments on five popular benchmarks show that our proposed method could outperform state-of-the-art models under most settings.

\section{Related Work}\label{sec:related}
Semi-supervised learning (SSL) aims at utilizing the unlabeled data to improve the model performance together with labeled data. In this paper, we focus on SSL methods in which the objective consists of a supervised classification loss for labeled data and a semi-supervised loss for unlabeled data. The readers are referred to \cite{ssl-survey} for a more comprehensive overview.

\subsection{Consistency Regularization}
This popular research line of SSL methods pursues the consistency among the predictions for unlabeled samples with noises. This relies heavily on the \textit{smoothness assumption} that samples close to each other
are likely to share the same label, and the \textit{cluster assumption} that the decision boundary between classes
should lie in low-density regions \cite{cluster_assmp}. Generally speaking, when points around a data sample have similar score distributions, data samples in a high-density region are more likely to have the same label \cite{PL}.

Early models of this type build a teacher model and train the student model using the teacher's predictions over unlabeled data \cite{Gamma-model, TS, Pi_model}, where the teacher and student are with the same structure but different noises over input space or model space. By maximizing the consistency between the teacher and the student, these approaches could obtain better performance. Based on the teacher-student structure, random noises are firstly introduced into the teacher model to improve the student model's generalization performance. For example, in $\Pi$-Model \cite{Pi_model} where both the teacher and student are copies of a single model, the Dropout operation and input Gaussian noise provide the randomness to generate different predictions for the same unlabeled point. After that, the idea of ensembling is further employed for the teacher. With historical teacher models, Temporal Ensembling \cite{Pi_model} adopts the exponential moving average of their predictions as training targets, while Mean-Teacher \cite{MT} first calculates the exponential moving average over those historical parameters then uses the prediction of the averaged model for the supervision to the student. Beyond random noises, adversarial noises are also implemented by Virtual Adversarial Training (VAT) \cite{VAT} to enlarge the prediction deviation between the teacher and student as much as possible. Via adversarial training, the model's robustness towards input perturbation is enhanced. Besides, VAT adopts an entropy minimization term \cite{EntMin} to encourage a lower entropy of the model's prediction distribution and thus make each class well-separated.

Seeing the effectiveness of data augmentation in supervised learning, researchers turn to use more complicated input transformations instead of random noises for training. In this stage, the teacher-student structure is simplified to a single model, such that the model output is invariant for augmented copies of unlabeled samples. Random flipping and cropping (also called weak augmentation) are used as the common augmentation operations to obtain the training targets, while different advanced transformation strategies are incorporated for predictions required to be consistent to training targets. At first, to improve the smoothness of model predictions, Interpolation Consistency Training (ICT) \cite{ICT} and MixMatch \cite{MixMatch} both introduce the MixUp method \cite{mixup} that interpolates the inputs and targets to augment the training data. Meanwhile, Unsupervised Data Augmentation (UDA) \cite{UDA} is proposed to utilize a more aggressive input augmentation strategy, RandAugment \cite{randaug} that randomly apply some transformations to the data. Motivated by this, ReMixMatch \cite{ReMixMatch} proposes another augmentation strategy called CTAugment, which learns the combination of basic transformations on training data. Such augmentation strategies, often called strong augmentation, perform a series of transformations to produce samples with large modifications which provide additional inductive biases for the consistency regularization \cite{UDA}. Moreover, ReMixMatch further proposes a distribution alignment strategy to adjust guessed label distribution based on the empirical ground-truth class distribution and incorporates a self-supervised loss. With advanced augmentation strategies, these approaches further improve the performance on benchmarks by a large margin.

Based on these basic frameworks, many efforts have been devoted to improving the performance of SSL with the incorporation of strategies such as sample weighting \cite{weighting1}, self-supervised learning \cite{self_sup1, self_sup2}, negative sampling \cite{neg_sampl1}, weight perturbation \cite{perturb1}.

\subsection{Self-training}
Another major direction of SSL is \textit{self-training}, or pseudo-labeling, which iteratively makes predictions for unlabeled samples and takes the predicted labels as their targets to train the model. It is first proposed in word sense disambiguation \cite{self-train1}. Since then, self-training has been applied in a wide range of tasks, \eg, object detection \cite{self-train2}, image classification \cite{PL} and domain adaptation \cite{self-train3}. For image classification, various extensions have been proposed to improve the performance based on the vanilla pseudo-labeling model \cite{PL}. For example, \cite{noisy1} adds input noises into the student model while feeding clean images into the teacher. \cite{label_prop1,label_prop2,label_prop3} introduce label propagation for the pseudo-labeling. And \cite{curriculum} combines self-training with curriculum learning.

Compared with vanilla score distributions, training with pseudo-labels is shown to be equivalent to an implicit entropy regularization over score distributions \cite{PL}. Thus, it decreases the class overlap and makes the data points more separable. In fact, advanced consistency regularization methods such as MixMatch, UDA and ReMixMatch also borrow such wisdom. They implement a temperature sharpening technique to make the training targets closer to one-hot labels for unlabeled data.
Recently, FixMatch \cite{FixMatch} is proposed to combine pseudo-labeling and consistency regularization to boost performance. It achieves very promising performance with very few labels on existing benchmarks even with many training techniques removed (\eg, prediction sharpening and distribution alignment). After that, researchers have developed many variants by introducing techniques like feature template matching \cite{new1}, dynamic thresholding \cite{dash}, optimal transport \cite{sinkhorn}, alpha divergence \cite{alpha-div} and meta-learning \cite{metapseudo}.

Among these researches, the state-of-the-art model FixMatch is most relevant to our study since we instantiate our proposed method by introducing our worst-case consistency framework into it. ReMixMatch is also a related study that involves multiple augmented samples for an unlabeled data point, whereas it minimizes all the discrepancies unlike our worst-case way. Nevertheless, our method enjoys theoretical guarantees for both generalization error and convergence behavior, which does not necessarily hold for existing models including FixMatch and ReMixMatch.

It is noteworthy that both the proposed method and VAT \cite{VAT} share the same ultimate goal, \ie, making the predictor robust (or smooth) against random and local perturbation. However, the two methods mainly differ in the way to realize this goal.
VAT pursues the local smoothness in terms of the numerical neighborhood. It borrows the idea of adversarial training to smooth the model's most sensitive direction (\ie, the adversarial direction) in an $\epsilon$-radius $\ell_2$-norm ball.
On the contrary, the proposed method focuses on the robustness over semantic neighborhoods. Based on the expected robust classification risk, it establishes a generalization bound and then derives a minimax formulation that selects the most inconsistent one among some given augmented variants around each unlabeled input data and then minimizes such a worst-case consistency.
Note that in the proposed method, the neighborhood set consists of finitely many elements obtained by some random data transformations to preserve the semantic proximity (as illustrated in the next section). Hence, the model only needs to choose the worst-case semantically consistent sample among these candidates, which is rather simple but proved to converge to a stationary point under mild conditions. Yet VAT adopts the $\ell_2$-norm ball which has infinitely many elements and only conveys the numerical proximity. To find the adversarial direction, it has to solve the inner maximization problem by iteration-based methods. In this sense, the convergence of the optimization algorithm is hard to be guaranteed.

\section{Model Formulation}\label{sec:model}
In this section, we first provide a general robust SSL generalization bound to show how the generalization error is related to the empirical risks on both labeled and unlabeled sets. Based on the bound, we formulate our optimization objective. Meanwhile, we also extend it to provide a practical bound for convolutional neural networks.

\subsection{Generalization Bound for General Models}
In this paper, we focus on the image classification problem. A labeled data point is a tuple $(x,y)$, where $x \in \rfield^d$ is the high dimensional input feature of a raw image, while $y \in [\nc] = \{1,2,\cdots, \nc\}$ is the label of the sample. While an unlabeled data point purely contains the information about input feature $x$ with its label missing. In semi-supervised learning problems, the training data consists of a subset of labeled data $\DL=\{(x_i^l, y_i)\}_{i=1}^{\nl}$ and a subset of unlabeled dataset $\DU=\{x_j^u\}_{j=1}^\nuu$, with $\nl$ and $\nuu$ being the number of labeled and unlabeled samples, respectively. $\DL$ is drawn from the joint distribution of $(x,y)$, which is denoted as $\setD_{XY}$, while $\DU$ is drawn from the marginal distribution of input feature $\setD_{X}$. For a classifier $f: \rfield^d\mapsto\rfield^\nc$ with the output (score vector) denoted as $s$, its predicted label is given by $y(s)=\argmax_{i\in[\nc]} s_i$.

Our goal is to learn a classifier $f$ parameterized by $\theta$ with a small generalization error. Since usually a large proportion of the training data are unlabeled, we intuitively hope the prediction of $f$ is sufficiently stable under input transformations as most existing consistency regularization-based methods do \cite{bound2,Pi_model,MixMatch,UDA}. According to the basic \textit{smoothness and cluster assumptions} \cite{zhou2014semi}, when data points located within the same neighborhood have similar score distributions, samples lying in a high-density region are more likely to have the same label and thus a low-density separation between classes are easier to find \cite{PL}. To this end, we wish the trained model to minimize the expected robust classification risk instead of the original zero-one loss. Specifically, when input $x$ is an image, we hope a moderate change to $x$ will not change the prediction of $f$. Motivated by this, given an uncertainty set around $x$ as $\B(x)$, we aim at minimizing the following robust version of expected classification error over the entire data distribution $\setD_{XY}$.

\begin{defn}[Expected Robust Classification Risk]
	Define $\RB$ as the expected robust classification risk with the zero-one loss function:
	\begin{equation}
		\RB = \E_{(x, y) \sim \setD_{XY}} \left[\sup_{x'\in\B(x)}\lzo(y(f(x')), y)\right],
	\end{equation}
	where $\lzo\left(y', y\right)=\mathbb{I}\left[y'=y\right]$ is the zero-one loss function, and $\mathbb{I}[\cdot]$ denotes the indicator function, which has a value of 1 if the input condition holds and 0 otherwise. $\sup_{x\in\B} f(x)$ is the supremum of a function $f(x)$ over a set $\B$.
\end{defn}

Since the data distribution is not available, one cannot even calculate $\RB$, let alone its minimization. To tackle this issue, we propose an upper bound of $\RB$ as the function of the empirical loss terms observed on training data $\DL \cup \DU$. Through optimizing this bound, we could minimize $\RB$ indirectly. Considering that the zero-one loss cannot be directly optimized in practice due to its non-differentiability at 0, the empirical terms in the bound will be built based on surrogate losses. Here, we choose the cross-entropy loss to do so.

For a multi-class classification model, let its output be $s\in\rfield^{\nc}$, the cross-entropy loss (also known as multinomial logistic loss) for $s$ and the target label $y\in[\nc]$ is defined as:
\begin{equation*}
	\lce(s, y) = \log(\sum_{i\in[\nc]}\exp(s_i - s_y)).
\end{equation*}
In some cases, we might use a soft label as the learning target, denoted as $t\in\rfield^{\nc}$. Then the cross-entropy loss could be written as:
\begin{equation*}
	\lce(s, t) = -\sum_{i\in[\nc]}t_i \cdot \log s_i.
\end{equation*}

Under this setting, we first present the following proposition to bound the expected robust risk using two individual expected risks with the cross-entropy loss.

\begin{prop}
	\label{prop:risk}
	For any classifier $f$ with bounded output $s\in[\varepsilon, 1-\varepsilon]^\nc$, the following inequality holds:
	\begin{equation}
		\scalemath{1}{
		\begin{aligned}
		\RB \leq\ &C_1 \cdot \E_{x\sim\setD_{X}} \bigg[{\color{lightred} \sup_{x'\in\B(x)}} {\color{ballblue} \lce\left(f(x'), f(x)\right)}\bigg] \\
		&+\ C_2 \cdot\E_{(x,y)\sim\setD_{XY}} \bigg[\lce(f(x),y)\bigg],
		\end{aligned}
		}
	\end{equation}
	where $C_1 = \frac{1}{\nc\varepsilon \cdot \log \frac{1}{1-\varepsilon}}$, $C_2 = \frac{1}{\log 2}$, and $\varepsilon\in(0,0.5)$.
\end{prop}

\Propref{prop:risk} is the basis of our final generalization bound. It shows that the expected robust risk over the whole data distribution is upper bounded by the combination of 1) a robust expected risk over the marginal distribution and 2) a standard expected risk over the joint distribution. This allows us to further bound each of them using corresponding empirical risk. Moreover, from the first term, we can see that when the classifier output is stable towards any input uncertainty, then $\RB$ will also be small. Due to space limitations, we provide the proof in \Apdref{sec:proof-risk-prop}.

Following the convention of machine learning, we assume that the learned classifier $f$ is selected from a restricted class of functions called hypothesis class $\setF$ (say, the class of all deep neural network models). Given this restriction, the generalization error largely depends on the complexity of the hypothesis class. Before entering the main result, we first provide two necessary definitions of the complexity of $\setF$ in terms of the well-known Rademacher complexity \cite{foundml}.

\begin{defn}[\textbf{Supervised Rademacher Complexity}]
	Given a hypothesis space $\setF$ and a loss function $\ell$, we define $$(\ell\circ f)(x,y) = \ell(f(x), y)$$ and extend it to the hypothesis class $\setF$ by $\ell \circ \setF = \{\ell\circ f: f\in\setF\}$. Then the empirical Rademacher complexity of $\ell \circ \setF$ over a labeled dataset $\DL = \{(x_i^l, y_i)\}_{i=1}^\nl$ is defined as:
	\begin{equation*}
		Rad_{\DL}(\ell \circ \setF) = \E_{\sigma}\left[ \sup_{f\in\setF} \frac{1}{\nl} \sum_{i=1}^\nl \sigma_i \cdot \ell(f(x_i^l), y_i) \right],
	\end{equation*}
	where $\sigma=\{\sigma_1, \cdots, \sigma_{\nl}\}$ is a sequence of i.i.d Rademacher random variables with $\mathbb{P}[\sigma_i=1] = \mathbb{P}[\sigma_i=-1] = 0.5$.
\end{defn}

\begin{defn}[\textbf{Unsupervised Rademacher Complexity}]
	Define $$(\ell^\B\circ f)(x) = \sup_{x'\in\B(x)} \ell(f(x), f(x')).$$ The empirical Rademacher complexity of $\ell^\B \circ \setF$ over an unlabeled dataset $\DU = \{x_i^u\}_{i=1}^\nuu$ is defined as:
	\begin{equation*}
		\scalemath{0.95}{
		\begin{aligned}
			Rad_{\DU}&(\ell^\B \circ \setF) = \E_{\sigma}\left[ \sup_{f\in\setF} \frac{1}{\nuu} \sum_{i=1}^\nuu \sigma_i \sup_{x'\in\B(x_i^u)} \ell(f(x_i^u), f(x')) \right].
		\end{aligned}
		}
	\end{equation*}
\end{defn}

Now we are ready to present the main result as the following theorem. The proof is provided in \Apdref{sec:proof-general-bound}.

\begin{thm}[\textbf{General Robust Generalization Bound for SSL}]
	\label{thm:general_bound}
	Let $\setF$ be the hypothesis class. Under the same condition as \Propref{prop:risk}, for any function $f\in\setF$, the following inequality holds with probability at least $1-\delta$ over the random draw of $\DL$ and $\DU$:
	\begin{equation}
		\scalemath{1}{
		\begin{aligned}
		&\RB \\
		&~~~~\leq \underbrace{C_1 \cdot \frac{1}{\nuu} \sum_{i=1}^\nuu {\color{lightred} \sup_{x'\in\B(x_i^u)}} {\color{ballblue} \lce(f(x'), f(x_i^u))}}_{(1)} + \underbrace{\underset{\phantom{{\setF_{\setF}}_{j_k}}}{C_2} \cdot \RL}_{(2)} \\
		&~~~~~~~~ + \underbrace{2C_2 \cdot  Rad_{\DL}(\lce \circ \setF) + 2C_1 \cdot  Rad_{\DU}(\lce^\B \circ \setF)}_{(3)} \\
		&~~~~~~~~ \underbrace{+ 3C_2 \cdot \sqrt{\frac{\log\frac{4}{\delta}}{2\nl}} + 3C_1 \cdot \sqrt{\frac{\log\frac{4}{\delta}}{2\nuu}}}_{(3)},
		\end{aligned}
		}
	\end{equation}
	where $C_1,C_2$ are constants same as \Propref{prop:risk}, and
	\begin{equation*}
		\begin{aligned}
			\RL = \frac{1}{\nl} \sum_{i=1}^\nl \lce(f(x_i^l), y_i).
		\end{aligned}
	\end{equation*}
\end{thm}

\begin{rem}
	\Thmref{thm:general_bound} shows that the expected risk could be bounded by the sum of three components:
	\begin{enumerate}
		\item[(1)] The first term is defined on the unlabeled dataset $\DU$. Recall that $\B(x_i^u)$ is an uncertainty set. This term captures the inconsistency between predictions of original samples and samples around corresponding uncertain regions.
		\item[(2)] The second term is defined on the labeled dataset $\DL$, which is the empirical classification risk.
		\item[(3)] The third term is only relevant with the choice of the hypothesis class $\setF$ and the magnitude of $\nl, \nuu$. In general, this term vanishes when the training data is sufficiently large.
	\end{enumerate}
\end{rem}

\noindent\textbf{Relationship to consistency regularization}: \Thmref{thm:general_bound} indicates that the robustness of $f$ toward input uncertainty largely depends on the robustness toward such uncertainty over unlabeled data. This formulation coincides with the consistency regularization adopted in the state-of-the-art methods \cite{UDA,FixMatch,MixMatch,ReMixMatch}, which forces the prediction to be invariant toward noise/augmentation over input examples. The key difference from existing studies is that the {\color{lightred} supremum} in the generalization bound in our paper requires the uncertainty consistency to be validated for the worst-case element in set $\B$. In fact, when there is only one data point in $\B(x_i^u)$, the supremum is canceled hence term (1) degenerates to the vanilla consistency regularization. As a sequence, applying the consistency regularization over unlabeled data naturally minimizes the expected robust risk.

To realize our goal, we need to minimize the upper bound of $\RB$, \ie, all the three components, as much as possible. First, term (1) and (2) are both empirical risks that could be optimized following the standard risk minimization scheme. However, the supremum operation over $\B$ in term (1) might cause difficulties for calculation. Besides, the minimization of term (3) is further related to the choice of the model. In the following, we will analyze the optimization of term (1) and (3) in detail.

First, term (1) is the one we will mainly focus on in this paper:
\begin{equation}
\frac{1}{\nuu}\sum_{i=1}^\nuu {\color{lightred} \sup_{x'\in\B(x_i^u)}} {\color{ballblue} \lce(f(x'), f(x_i^u))},
\end{equation}
which is the performance bottleneck for SSL problems. In practice, $\B$ could be chosen as an $\ell_\infty$-norm ball with radius $\epsilon$: $\B_\infty^\epsilon(x) = \{x' \in\rfield^d | \norm{x'-x}_\infty \le \epsilon\}$. However, such sets contain infinitely many elements, making the supremum hard to calculate. Beyond computational issues, to represent the change of an original data $x$, the uncertainty set should also characterize the semantic proximity around $x$. In terms of this property, the norm ball sets merely convey numerical proximity. However, pursuing numerical proximity should not be our goal in the end, since for image data, the proximity is often expressed at the semantic level. As shown in \Figref{fig:example}{\color{ballblue}(a)}, the left image depicts a puppy, while the right depicts a bagel. The appearance distance between the two is pretty small, even though they belong to completely different semantic categories. A similar pair is displayed in \Figref{fig:example}{\color{ballblue}(b)}. Seeing the above-mentioned issues, we introduce a semantic uncertainty set $\Bse$ with finitely many elements. Specifically, given an original data point with feature $x$, $\Bse(x)$ is constructed by generating $K$ class-invariant augmentations of $x$. Following previous work along the direction of consistency regularization \cite{UDA, FixMatch}, each augmented data $x'_i$ is often obtained by performing an operation $\tau_i$ upon $x$, which is a combination of multiple atomic class-invariant data transformations such as solarization, rotation, and translation. Mathematically, $\Bse(x)$ could be defined as:
\begin{equation}
	\Bse(x) = \{x'_i: x'_i = \tau_i(x),~ i\in[K]\}.
\end{equation}
Please see Sec.\ref{sec:train} for more implementation details.

The resulted optimization problem for unlabeled data could then be written as:
\begin{equation}
	\label{eq:minimax}
	\min_\theta \  \frac{1}{n_u}\sum_{i=1}^{n_u}~ {\color{lightred} \max_{j\in[K]}}~ {\color{ballblue} \lce\big(f(x_{i,j}'), f(x_i^u)\big)},
\end{equation}
where $x'_{i,j}\in\Bse(x_i^u)$ is the $j$-th augmented data for the $i$-th unlabeled sample obtained by class-invariant data augmentation.

On the other hand, term (3) depends on the Rademacher complexity of the hypothesis class and the size of the training data. A number of theoretical analysis on the Rademacher complexity of deep neural networks have already been proposed \cite{dnn_rademacher1,dnn_rademacher2}, which suggest that the complexity enjoys a general form
\begin{equation*}
	\begin{aligned}
		Rad(\setF) &= g(\mathsf{Comp}, n),\\
		Rad(\ell \circ \setF) &= g(\mathsf{Comp}, n),
	\end{aligned}
\end{equation*}
where $\mathsf{Comp}$ is an abstract term only dependent on the size and architecture of the network. Specifically, it holds that $g(\mathsf{Comp}, n)\to 0$ when $n\to \infty$. In this sense, term (3) could not be reduced by simple optimization techniques. Practically, one can instead resort to increasing the size of training data to realize it. Consequently, we will ignore the impact of this term on our optimization problem.

\begin{figure}[t]
	\centering
	\includegraphics[width=1\columnwidth]{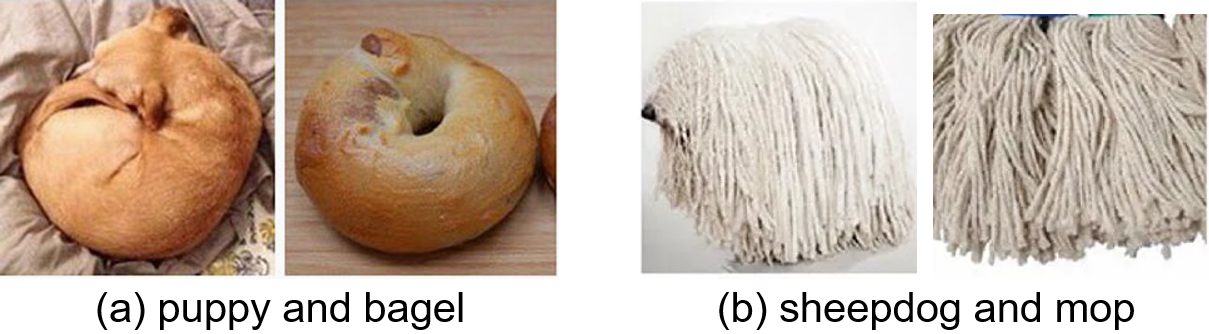}
	\caption{Images convey the numerical proximity but without semantic proximity.}
	\label{fig:example}
\end{figure}

Putting all these together, we have the overall objective as follows:
\begin{equation}
	\label{eq:minimax_all}
	\begin{split}
		\min_\theta \ \frac{1}{n_l}\sum_{i=1}^{n_l} \ell_i^l + \frac{\gamma}{n_u}\sum_{i=1}^{n_u} {\color{lightred} \max_{j\in[K]}} ~ {\color{ballblue} \ell_{i,j}^u},
	\end{split}
\end{equation}
where
\begin{equation*}
	\begin{split}
		\ell_i^l &= \lce\big(f(x_i^l), y_i\big), \\
		\ell_{i,j}^u &= \lce\big(f(x_{i,j}'), f(x_i^u)\big),
	\end{split}
\end{equation*}
and $\gamma>0$ is the trade-off hyperparameter. When $K=1$, this formulation degenerates to the vanilla consistency regularization framework.

\subsection{Practical Generalization Bound for Convolutional Neural Networks}
\Thmref{thm:general_bound} provides a generalization bound for general models. Based on this, now we zoom in on the popular deep convolutional neural network model, which is widely used in real-world applications, and see how its generalization error is bounded. All the proofs are provided in \Apdref{sec:proof-cnn-bound}.

First, we need to introduce this specific setting. We consider a deep convolutional network with its input constrained by $\norm{vec(x)}_2 \le \chi$ and a bounded $\nc$-dimensional output. The network consists of $L_c$ convolutional layers and $L_f$ fully-connected layers, \ie, $L_N = L_c + L_f$ layers in total. In $i$-th convolutional layer, the convolution kernel $U^{(i)}\in\rfield^{k_i\times k_i \times c_{i-1} \times c_i}$ is followed by a component-wise non-linear activation function and an optional pooling operation. Here we assume both the activation function and pooling operation are $1$-Lipschitz and nonexpansive (\eg, ReLU, tanh and sigmoid activation function, and max-pooling operation). We also use $op(U)$ to denote the operator matrix for the kernel $U$, which expresses the convolution as a matrix-vector product, \ie, $U(x)=op(U)vec(x)$. Meanwhile, the weight matrix of $i$-th fully-connected layer is denoted by $V^{(i)}$. We then denote all the network parameters by $\theta = (U^{(1)}, \cdots, U^{(L_c)}, V^{(1)}, \cdots, V^{(L_f)})$, and the total number of parameters by $W$. For two networks parameterized by $\theta$ and $\tilde{\theta}$ respectively, we define their distance as
\begin{equation*}
	d_N(\theta, \tilde{\theta}) = \sum^{L_c}_{i=1} \norm{op(U^{(i)}) - op(\tilde{U}^{(i)})}_2 + \sum^{L_f}_{i=1} \norm{V^{(i)} - \tilde{V}^{(i)}}_2.
\end{equation*}
We let all the initialized parameter matrices have operator norms at most $1+\nu$. Namely, for the initialized parameters $\theta_0 = (U^{(1)}_0, \cdots, U^{(L_c)}_0, V^{(1)}_0, \cdots, V^{(L_f)}_0)$, it subjects to $\Vert op(U^{(i)}_0) \Vert_2 \le 1+\nu, \forall i\in[L_c]$ and $\Vert V^{(j)}_0 \Vert_2 \le 1+\nu, \forall j\in[L_f]$.
We further assume that the distance between the learned parameters and the initialized ones $\theta_0$ is no larger than $\beta$. So the set of learned parameters is defined as
\begin{equation*}
	\setO_{\beta,\nu} = \Big\{\theta: d_N(\theta, \theta_0) \le \beta \Big\}.
\end{equation*}
And the corresponding set of CNNs parameterized by such learned parameters could be denoted as $\setFbu$. For any function $f\in\setFbu$, we could rewrite it in the form of component functions $f=(f^{(1)}, \cdots, \fj, \cdots, f^{(\nc)})$ where $\fj$ produces the output for the $j$-th class. Denote the class of the component functions as $\setFbu^0$, $f$ could be written as
\begin{equation*}
	f\in\setFbu= \underbrace{\setFbu^0 \times \cdots \times \setFbu^0}_{\nc}.
\end{equation*}

From \Thmref{thm:general_bound}, we see that to obtain our final result, it is necessary to bound the Rademacher complexity of the network. Unfortunately, there are two main difficulties in this process. On one hand, $f$ is a vector-valued function, which is more difficult to bound than scalar-valued functions. On the other hand, the compositional structure of neural networks makes it also hard to derive the Rademacher complexity upper bound. To deal with these issues, we adopt the covering number and chaining techniques in our derivation of the Rademacher complexity bound.

The derivation involves the Rademacher complexity for multi-output functions.
\begin{defn}[\textbf{Rademacher Complexity of Multi-output Functions}]
	Let $\setF$ be the hypothesis space for $\nc$-dimentional multi-output functions. The empirical Rademacher complexity of $\setF$ over a dataset $\setD=\{x_i\}^n_{i=1}$ is defined as:
	\begin{equation*}
		\begin{aligned}
			Rad_{n, \nc}&(\Pi\circ\setF) = \\
			&\E_\sigma \left[\sup_{f=(f^{(1)}, \cdots, f^{(\nc)})\in\setF} \frac{1}{n \cdot \nc} \sum^{\nc}_{j=1} \sum^n_{i=1} \sigma^{(j)}_i \cdot \fj(x_i)\right],
		\end{aligned}
	\end{equation*}
	where $\{\sigma^{(j)}_i\}_{i,j}$ is a sequence of independent Rademacher random variables with $\mathbb{P}[\sigma_i^{(j)}=1] = \mathbb{P}[\sigma_i^{(j)}=-1] = 0.5$.
\end{defn}

With this definition, we could then reach the following complexity bound of networks.
\begin{lem}
	\label{lem:cover_bound}
	The Rademacher complexity of multi-output deep convolutional networks could be bounded by
	\begin{equation}
		Rad_{n, \nc}(\Pi\circ\setFbu) \le \frac{4}{\sqrt{n \cdot \nc}} + \frac{12}{\sqrt{n \cdot \nc}} \sqrt{W \cdot \log(C_N \sqrt{n\nc})},
	\end{equation}
	where $C_N = 3\chi \cdot e^{\frac{\beta}{1+\nu}}$.
\end{lem}

We can see that this complexity is related to the size of both the dataset and the classes in the order of $O\left(\sqrt{\sfrac{W \log(n\nc)}{n\nc}}\right)$. When there are more available training data or more classes to learn, the Rademacher complexity will become smaller. Moreover, the result also implies that enlarging the number of parameters, the input scale and the distance from initialized parameters will increase the Rademacher complexity, which is consistent with intuition.

Meanwhile, our final result also relies on the Rademacher complexity bound for the component class $\setFbu^0$. Considering the semantic uncertainty set $\Bse$, we also need to take the input transformation into account when inducing the Rademacher complexity upper bound. This is based on the following definition.

\begin{defn}
	The empirical Rademacher complexity of single-output functions over a dataset $\setD=\{x_i\}^n_{i=1}$ with corresponding input transformations $\setT_\setD = \{\tau_i\}^n_{i=1}$ and a hypothesis space $\setF$ is defined as:
	\begin{equation}
		Rad_\setD(\setF\circ\setT_\setD) = \E_\sigma \bigg[\sup_{f\in\setF} \frac{1}{n} \sum^n_{i=1} \sigma_i \cdot f(\tau_i(x_i))\bigg].
	\end{equation}
\end{defn}

Similar to \Lemref{lem:cover_bound}, we have the Rademacher complexity bound related to $\setFbu^0$ as follows.

\begin{lem}
	\label{lem:b_cover_bound}
	Denote by $W_g$ the number of parameters for convolutional networks that replace the last layer of $f\in\setFbu$ with a single-output layer. Let the output of all the transformation functions $\tau$ be constrained by $\norm{vec(\tau(x))}\le\chi_\tau$, then the Rademacher complexity of $\setFbu^0$ without or with input transformation $\setT_\setD = \{\tau_i\}^n_{i=1}$ could be bounded by
	\begin{equation}
		\begin{aligned}
			Rad_\setD(\setFbu^0) &\le \frac{4}{\sqrt{n}} + \frac{12}{\sqrt{n}} \sqrt{W_g \cdot \log(C_N n)}, \\
			Rad_\setD(\setFbu^0\circ\setT_\setD) &\le \frac{4}{\sqrt{n}} + \frac{12}{\sqrt{n}} \sqrt{W_g \cdot \log(C_{N,\tau} n)},
		\end{aligned}
	\end{equation}
	respectively, where $C_N = 3\chi \cdot e^{\frac{\beta}{1+\nu}}$ and $C_{N,\tau} = 3\chi_\tau \cdot e^{\frac{\beta}{1+\nu}}$.
\end{lem}

Based on \Lemref{lem:cover_bound} and \ref{lem:b_cover_bound}, we could then provide the following generalization bound for the convolutional networks.

\begin{thm}[\textbf{Robust Generalization Bound for SSL on CNNs}]
	\label{thm:cnn_bound}
	Let $\setFbu$ be the hypothesis class of multi-output deep convolutional networks. Define the empirical risk on labeled and unlabeled set as
	\begin{equation*}
		\begin{aligned}
			\RL &= \frac{1}{\nl} \sum_{i=1}^\nl \lce(f(x_i^l), y_i), \\
			\RU &=  \frac{1}{\nuu} \sum_{i=1}^\nuu {\color{lightred} \sup_{x'\in\Bse(x_i^u)}} {\color{ballblue} \lsce(f(x'), f(x_i^u))}.
		\end{aligned}
	\end{equation*}
	Under the same condition as \Propref{prop:risk}, for any function $f\in\setFbu$, the following inequality holds with probability at least $1-\delta$ over the random draw of $\DL$ and $\DU$:
	\begin{equation}
		\scalemath{0.93}{
		\begin{aligned}
			&R^{0/1}_{\B\text{-robust}}(f) \leq\, C_1 \cdot \RU + C_2(1+\frac{C_0}{2}) \cdot \RL \\
			&~~~+ C_1 \biggl( \frac{16K\nc}{\sqrt{\nuu}}\cdot\frac{1-\varepsilon}{\varepsilon} \bigg[ 1 + 3 \sqrt{W_g \cdot \log(C_M \nuu)} \bigg] + 3\sqrt{\frac{\log \frac{4}{\delta}}{2\nuu}} \biggr) \\
			&~~~+ \frac{3C_0C_2}{2}\cdot \Psi_{\nl,\nc,\delta}(\setFbu)
		\end{aligned}
		}
	\end{equation}
	where $C_0>0$ is a universal constant, $C_1,C_2$ are constants same as \Propref{prop:risk}, $C_N = 3\chi \cdot e^{\frac{\beta}{1+\nu}}$, $C_M = 3\max\{\chi,\chi_\tau\} \cdot e^{\frac{\beta}{1+\nu}}$ and
	\begin{equation*}
		\scalemath{0.96}{
		\begin{aligned}
			\psi_{\nl, \nc}(\setFbu) =& \frac{4}{\sqrt{\nl\cdot \nc}} + \frac{12}{\sqrt{\nl\cdot \nc}} \sqrt{W \cdot \log(C_N \sqrt{\nl\nc})}, \\
			\Psi_{\nl, \nc,\delta}(\setFbu) =& 2 \left( \sqrt{\nc} \log^{3/2}(\nl \nc e) \cdot \psi_{\nl, \nc}(\setFbu) + \frac{1}{\sqrt{\nl}} \right) \\
			&+ \frac{\log(\nc e)}{\nl} \left(\log\frac{2}{\delta} + \log(\log\;\nl)\right).
		\end{aligned}
		}
	\end{equation*}
\end{thm}

\begin{rem}
	\Thmref{thm:cnn_bound} implies that the third term is of order $O\big(K\sqrt{\frac{W_g\log \nuu}{\nuu}}\big)$, and the fourth term is of order $\Psi_{\nl, \nc,\delta}(\setFbu)$, more specifically, in $O\big(\log^{2}(\nl \nc)\sqrt{\frac{W}{\nl}}\big)$. Therefore, the order of the generalization gap is
	\begin{equation*}
		O\left(K\sqrt{\frac{W_g\cdot\log \nuu}{\nuu}} + \log^{2}(\nl \nc)\sqrt{\frac{W}{\nl}}\right).
	\end{equation*}
\end{rem}

\section{Optimization}\label{sec:optim}
Intuitively, we could adopt a natural alternative optimization scheme to solve the minimax problem in (\ref{eq:minimax_all}). The corresponding algorithm, which is called \textit{MaxMatch}, is shown in \Algref{alg}. In each step, we sample a mini-batch of labeled and unlabeled data, respectively. First, we obtain $\Bse$ for each unlabeled data point by performing the transformation operation for $K$ times (\texttt{line 5}). We then fix the model parameters $\theta$, and solve the maximization problem
\begin{equation}
	\max_{j\in[K]} \ell_{i,j}^u
\end{equation}
independently for each unlabeled data by finding out the most inconsistent term among $\ell_{i,j}^u$ (\texttt{line 6}). With the choice of $K$ fixed, our algorithm then employs a gradient solver to minimize the classification loss and the consistency loss founded in the previous step (\texttt{line 7}).

\begin{algorithm}[t]
	\caption{MaxMatch}
	\begin{algorithmic}[1]
		\Require Labeled dataset $\mathcal{D_L}$, unlabeled dataset $\mathcal{D_U}$, labeled batch size $B_l$, unlabeled batch size $B_u$, number of transformation $K$, trade-off coefficient $\lambda$, learning rate $\eta$, max iteration $T$
		\Ensure Model parameter $\theta$
		\State Initialize $\theta\gets\theta_0$
		\For{t}{1}{T}
		\State Sample a mini-batch of labeled data $\mathcal{X}_l$
		\State Sample a mini-batch of unlabeled data $\mathcal{X}_u$
		\State Augment each unlabeled sample for $K$ times
		\State Find the largest consistency loss for each unlabeled sample $$\ell^u_i=\max_{j\in[K]} \lce\big(f(x_{ij}'), f(x_i^u)\big)$$
		\State Minimize the selected consistency loss together with the classification loss of labeled data with a gradient solver
		\begin{equation*}
		\begin{split}
		g_t &= \nabla_\theta \Big(\frac{1}{B_l}\sum_{i}^{B_l}\ell_i^l + \frac{\lambda}{B_u}\sum_{i}^{B_u} \ell^u_i\Big) \\
		\theta_{t} &\gets \theta_{t-1} - \eta \cdot g_t
		\end{split}
		\end{equation*}
		\EndFor
	\end{algorithmic}
	\label{alg}
\end{algorithm}

With the proposed algorithm, we need to find if it is able to converge to a local stationary point of (\ref{eq:minimax_all}). However, the indifferentiability of the maximum operation over these sampled data in (\ref{eq:minimax_all}) makes it difficult to analyze the convergence property of this problem. Fortunately, for each unlabeled data point $x^u_i$, if we introduce a continuous sample weight $v^u_{i,j} \in [0,1]$ for each of the $K$ augmented data in $\Bse$.
This minimax problem could be reformulated as a differentiable minimax problem. This is shown in \Factref{fact:reform}.

\begin{fact}
	\label{fact:reform}
	The optimization problem (\ref{eq:minimax_all}) is equivalent to the following minimax optimization problem:
	\begin{equation}
		\label{eq:reform_prob}
		\begin{split}
			\min_\theta \max_{\substack{v_{i,j}^u\in[0,1],\\ \sum_{j=1}^K v_{i,j}^u = 1}} ~\frac{1}{n_l}\sum_{i=1}^{n_l} \ell_i^l + \frac{\gamma}{n_u}\sum_{i=1}^{n_u}  \sum_{j=1}^K v_{i,j}^u \cdot \ell_{i,j}^u, \\
		\end{split}
	\end{equation}
	where $v_{i,j}^u$ is the weight for the $j$-th augmented variant of $i$-th unlabeled sample.
\end{fact}
For the sake of simplicity, we denote the set of $\{v^u_{i,j}\}_{i,j}$ as ${v}$ and define:
\begin{equation}
\mathcal{L}(\theta, v)\ = \frac{1}{n_l}\sum_{i=1}^{n_l} \ell_i^l + \frac{\gamma}{n_u}\sum_{i=1}^{n_u}  \sum_{j=1}^K v_{i,j}^u \cdot \ell_{i,j}^u.
\end{equation}

Moreover, we denote the feasible set as:
\begin{equation}
	\mathcal{V} = \bigg\{v_{i,j}^u: v_{i,j}^u\in[0,1], ~\sum_{j=1}^K v_{i,j}^u = 1\bigg\}.
\end{equation}

Since the weight $v_{i,j}^u$ is real-valued, we now have a differentiable objective for further analysis. To study the convergence behavior of (\ref{eq:reform_prob}), we need to clarify the notion of optimality first. From a game-theoretic perspective, \eqref{eq:minimax_all} could be regarded as a zero-sum game between two players, where the max player first maximizes the loss via choosing hard augmented examples in the uncertainty set, then the min player minimizes the maximized training loss via learning the model parameters $\theta$. In this sense, one might adopt the Nash equilibrium, following the recent studies in machine learning \cite{gan1,gan2,gan3}. However, the Nash equilibrium assumes that both players act simultaneously without the knowledge of the action taking by their opponent. This is not consistent with our setting, since the max player always acts first and the min player acts subsequently based on the max player's action.
In fact, problem (\ref{eq:reform_prob}) is an extensive-form game \cite{game} in the sense that, the model must first maximize the consistent loss for the given unlabeled sample $x^u_i$, then minimize the obtained maximum value. Since the classifier $f$ is non-convex with a deep neural network formulation, $\mathcal{L}(\cdot, v)$ is non-convex. Under this circumstance, the order of the actions does matter, since in general we have:
\begin{equation}
\max_{v \in \mathcal{V}} \min_{\theta} \mathcal{L}(\theta, v) \neq  \min_{\theta} \max_{v \in \mathcal{V}} \mathcal{L}(\theta, v).
\end{equation}

For such an objective, exchanging the minimization and maximization steps will cause totally different solutions.
Instead of the Nash equilibrium, we adopt the \emph{Stackelberg equilibrium} or what is called global minimax point in \cite{equilibrium} as our optimality condition. Different from the Nash equilibrium, the notion of global minimax point clearly considers the order of the two players:
\begin{defn}[Global minimax point]
	For a minimax problem $\min_{\theta\in\Theta}\max_{v\in\mathcal{V}}\mathcal{L}(\theta,v)$, $(\theta^\star, v^\star)$ is a global minimax point, if for any $(\theta,v)$, we have:
	\begin{equation}
		\mathcal{L}(\theta^\star, v) \le \mathcal{L}(\theta^\star, v^\star) \le \max_{v'\in\mathcal{V}} \mathcal{L}(\theta,v').
	\end{equation}
\end{defn}

In this definition, the global minimax point requires $\theta^\star$ to minimize $\max_{v \in \mathcal{V}}\mathcal{L}(\cdot, v)$. Since our goal is to minimize the training error, we should focus on the min player in this game. In this sense, we expect that:
\begin{equation}
\phi(\theta) = \max_{v \in \mathcal{V}} \mathcal{L}(\theta, v)
\end{equation}
to be as small as possible. Since $\phi(\theta)$ is non-convex, we then wish the algorithm to reach a stationary point where the gradient of $\phi$ is zero. To achieve this approximately, we resort to the $\epsilon$-stationary points of $\phi$ \cite{rafique2021weakly}, where the norm of the gradient is no larger than $\epsilon$.
\begin{defn}[$\epsilon$-stationary points of $\phi$]
$\tilde{\theta}$ is said to be an $\epsilon$-stationary point of $\phi$ if it satisfies that :
\begin{equation}
	\min_{g \in \partial\phi(\tilde{\theta})} \norm{g} \le \epsilon.
\end{equation}
where $\partial\phi(\tilde{\theta})$ is the sub-differential.
\end{defn}

Since the maximization operation inside $\phi$ makes it non-smooth, we use the sub-differential to characterize a surrogate of the gradient, see \cite{boyd2004convex} for more details. The non-smoothness of $\phi$ makes it difficult to carry out convergence analysis directly. Nonetheless, following the trick in \cite{equilibrium}, we can turn to the Moreau envelope as an auxiliary tool.

\begin{defn}[Moreau envelope]
	The Moreau envelope $\phi_{1/2\lambda}$ of a function $\phi$ is defined as follows:
	\begin{equation}
		\phi_{1/2\lambda}(\theta) := \min_{\theta'} \phi(\theta') + \frac{1}{2\lambda} \cdot \norm{\theta' - \theta}^2.
	\end{equation}
	where $\lambda>0$ is the hyperparameter.
\end{defn}

From \cite{equilibrium}, we have \Lemref{lem:phiLambda_to_phi} to relate the $\epsilon$-stationary point of $\phi_{1/2\lambda}$ to that of $\phi$. This requires the definitions for Lipschitz continuity.

\begin{defn}[$L$-Lipschitz]
	A function $f$ is called $L$-Lipschitz if there exists a constant $L\ge 0$ such that for any $x$ and $y$,
	\begin{equation*}
		\norm{f(x)-f(y)}\le L\norm{x-y}.
	\end{equation*}
\end{defn}

\begin{defn}[$\kappa$-gradient Lipschitz]
	A function $h$ is said to be $\kappa$-gradient Lipschitz if there exists a constant $\kappa>0$ such that for any $x,y\in dom(h)$,
	\begin{equation*}
		\norm{\nabla h(x) - \nabla h(y)}_2 \le \kappa\norm{x-y}_2.
	\end{equation*}
\end{defn}

\begin{lem}[\cite{equilibrium}]\label{lem:phiLambda_to_phi}
If $\mathcal{L}$ is $\kappa$-gradient Lipschitz, then for all $\lambda < 1/\kappa$, $\theta$ is an $\epsilon$-stationary point of $\phi_{1/2\lambda}$ implies that $\theta$ is an $\epsilon$-stationary point of $\phi$.
\end{lem}

Thanks to the lemma above, we only need to prove that the proposed algorithm could control $\norm{\nabla \phi_{1/2\lambda}}$. Accordingly, we have the following theorem.

\begin{thm}
	\label{thm:moreau_bound}
	Suppose $\ell^l$ is $\frac{L}{2}$-Lipschitz, $\ell^u$ is $\frac{L}{2\gamma}$-Lipschitz, and then $\mathcal{L}$ is $L$-Lipschitz and $\kappa$-gradient Lipschitz, the parameters $\theta$ are chosen from a compact set $\Theta$\footnote[1]{This could be realized by limiting the norm of weights (say, weight decay).}.
	Define $\phi(\cdot) := \max_v \mathcal{L}(\cdot,v)$ and $\phil$ is the Moreau envelope of $\phi$. Suppose $\max_{\theta \in \Theta} \big|\phi(\theta) - \min_{\theta\in \Theta} \phi(\theta)\big| \le B.$ Then with probability at least $1-\delta$, the following inequality holds for $\theta_t$s obtained from \Algref{alg}:
	\begin{gather}
	 \frac{1}{T+1} \sum_{t=0}^T \norm{\nabla\phil(\xt)}^2 \\
	 ~~\le 4L ~ \sqrt{\frac{\kappa\big(\phil(\theta_0) - \min_\theta \phi(\theta)\big)}{T}} + 8L ~ \sqrt{\frac{2B\kappa}{T+1}\log\frac{1}{\delta}}.
	\end{gather}
\end{thm}

\begin{proof}[Proof Sketch]
Since $\mathcal{L}$ is $\kappa$-gradient Lipschitz and $v_t$ is a maximizer for $\xt$, we have that any $\xt$ from \Algref{alg} and $\tilde{\theta}$ satisfy
\begin{equation}
	\begin{split}
		\inprod{\gradGxyt, \tilde{\theta} - \xt} \le \phi(\tilde{\theta}) - \phi(\xt) + \frac{\kappa}{2} \cdot \norm{\tilde{\theta} - \xt}^2.
	\end{split}
\end{equation}

Let
\begin{equation}
	\hatxt := \arg\min_\theta \phi(\theta) + \kappa \cdot \norm{\theta - \xt}^2.
\end{equation}

We have:
\begin{equation}
	\begin{split}
		&\phil(\theta_{t+1}) \le ~ \phil(\xt) + 2\eta\kappa \cdot \inprod{g_t - \gradGxyt, \hatxt - \xt} \\
		& ~~~~~~~~~~ + 2\eta\kappa \cdot \left( \phi(\hatxt) - \phi(\xt) + \frac{\kappa}{2} \cdot \norm{\hatxt - \xt}^2 \right) + \eta^2\kappa L^2,
	\end{split}
\end{equation}
where $g_t$ denotes the gradient estimated in the $t$-th step (obtained by \texttt{line 7} in Algorithm 1).

Taking a telescopic sum over $t$ and rearranging this inequality, we obtain
\begin{equation}\label{eq:telescopic_res}
	\begin{split}
		\frac{1}{T+1} \sum_{t=0}^T &\phi(\xt) - \phi(\hatxt) - \frac{\kappa}{2} \cdot \norm{\hatxt - \xt}^2 \\
		 \le &\frac{\phil(\theta_0) - \min_\theta \phi(\theta)}{2\eta\kappa T} + \frac{\eta L^2}{2} \\
		& + \frac{1}{T+1} \cdot \sum_{t=0}^T \inprod{g_t - \gradGxyt, \hatxt - \xt}.
	\end{split}
\end{equation}

Recall that $\mathcal{L}$ is $\kappa$-gradient Lipschitz. According to Fact~\ref{fact:weakcvx}, $\phi(\theta)$ is $\kappa$-weakly convex and $\phi(\theta) + \frac{\kappa}{2}\norm{\theta}^2$ is convex. Then, it is easy to see that $\Phi(\theta) = \phi(\theta) + \kappa \cdot \norm{\xt - \theta}^2$ is $\kappa$-strongly convex. Thus, we have
\begin{equation}\label{eq:moreau_ineq}
	\begin{split}
		\phi(\xt) - \phi(\hatxt) - \frac{\kappa}{2} \cdot \norm{\hatxt - \xt}^2 \ge ~ \frac{1}{4\kappa} \norm{\nabla\phil(\xt)}^2.
	\end{split}
\end{equation}

Since $\max_{\theta \in \Theta}\abs{\phi(\theta) - \min_{\theta\in \Theta} \phi(\theta)} \le B$, we have
\begin{equation}
	\begin{split}
		\phi(\hatxt) + \kappa \cdot \norm{\hatxt - \xt}^2 \le ~ \sqrt{\frac{\phi(\xt) - \min_\theta \phi(\theta)}{\kappa}} \le \sqrt{\frac{B}{\kappa}}.
	\end{split}
\end{equation}

Notice that $g_t$ is an unbiased estimation of $\gradGxyt$, namely we have $\E[g_t]=\gradGxyt$. Let $$G_t := \inprodG{t},$$ then we have
\begin{equation}
	\begin{split}
		\E\left[ G_t | G_{t-1}, G_{t-2}, \ldots, G_{1} \right] = 0, \\
		\E\left[ \sum_{t=0}^T G_t \right] = 0.
	\end{split}
\end{equation}
Namely, $\left\{ G_t \right\}_t$ is a martingale difference sequence. Then using Cauchy-Schwartz's inequality, we have
\begin{equation}
	\begin{split}
		\abs{G_t} \le \norm{g_t - \gradGxyt} \cdot \norm{\hatxt - \xt} \le 2L \sqrt{\frac{B}{\kappa}}.
	\end{split}
\end{equation}

Via Azuma's inequality, with probability at least $1-\delta$ we have
\begin{equation}\label{eq:azuma}
	\sum_{t=0}^T G_t \le 2L \cdot \sqrt{\frac{B}{\kappa} \cdot 2(T+1)\log\frac{1}{\delta}}.
\end{equation}

Substituting (\ref{eq:moreau_ineq}) and (\ref{eq:azuma}) into (\ref{eq:telescopic_res}), we have for any fixed $\eta>0$ the following inequality holds with probability at least $1-\delta$:
\begin{equation}\label{eq:ineq-2}
	\begin{split}
		\frac{1}{T+1} \sum_{t=0}^T \frac{1}{4\kappa} \norm{\nabla\phil(\xt)}^2 \le & \frac{\phil(\theta_0) - \min_\theta \phi(\theta)}{2\eta\kappa T} + \frac{\eta L^2}{2} \\
		& + 2L \cdot \sqrt{\frac{B}{\kappa} \cdot \frac{2}{T+1}\log\frac{1}{\delta}}.
	\end{split}
\end{equation}

To reach the tightest bound in this case, we can find $\eta^\star$ such that
\begin{equation}
	\begin{split}
		\eta^\star &= \argmin_{\eta>0} \frac{\phil(\theta_0) - \min_\theta \phi(\theta)}{2\eta\kappa T} + \frac{\eta L^2}{2} \\
		&= \sqrt{\frac{\phil(\theta_0) - \min_\theta \phi(\theta)}{\kappa L^2 T}}.
	\end{split}
\end{equation}

Plugging $\eta^\star$ into (\ref{eq:ineq-2}) proves the results. Due to space limitations, we provide the complete proof in \Apdref{sec:proof-moreau-bound}.
\end{proof}

\begin{rem}
	From \Thmref{thm:moreau_bound}, we see that under mild smoothness assumptions, the average norm of the gradient of the Moreau envelope decreases at a rate of $\sfrac{1}{\sqrt{T}}$. This implies the average norm of the gradient of $\phi$ will converge to $0$.
\end{rem}

\section{Instantiation of the Proposed Framework}\label{sec:train}
So far, we have introduced a general workflow of our proposed framework. Note that this worst-case consistency regularization framework could be easily applied to most existing consistency regularization-based models to enhance their performance, since what we need is only to extend their consistency loss into the worst-case form and then follow \Algref{alg} for training. Here we choose the state-of-the-art method FixMatch \cite{FixMatch} as a base model to instantiate our framework. According to FixMatch, an extra weak data augmentation (\eg, random flipping followed by random cropping) is carried out on both labeled and unlabeled data to mitigate overfitting. Besides, pseudo-labeling technique is adopted. In other words, the predicted labels of weak augmented unlabeled data, instead of the score distribution, are used as training targets. Following the convention, only samples with the largest predicted probability higher than a predefined threshold $\beta$ are considered in training. Then the loss of each labeled and unlabeled data becomes:
\begin{equation}\label{eq:implement}
	\begin{split}
	\ell^l_i &= \lce\big(f(\mathcal{A}_w(x_i^l)), y_i\big), \\
	\ell^u_i &= \mathbb{I}\Big(\max_k\tilde{s}_{i,k}^u > \beta\Big) \cdot \max_{j\in[K]} \lce\Big(f\big(x'_{i,j}\big), y(\tilde{s}_{i}^u)\Big),
	\end{split}
\end{equation}
where $\mathcal{A}_w(\cdot)$ denotes the weak augmentation function, $x'_{i,j} \in \Bse(\mathcal{A}_w(x_i^u))$ is the augmented variant for unlabeled data, $\tilde{s}_i^u = \tilde{f}\big(\mathcal{A}_w(x_i^u)\big)$ is the predicted score distribution with fixed model parameters.

\noindent\textbf{Time complexity analysis.}\quad For the sake of simplicity, we ignore the data loading and model initialization steps before training. At each iteration, the training process mainly consists of four steps: (1) data transformation, (2) feed forward, (3) back-propagation, and (4) optimization. Firstly, we denote the time costs of these steps in FixMatch as $t_{\text{data}}, t_{\text{ff}}, t_{\text{bp}}, t_{\text{opt}}$ respectively, and the number of total iterations as $T$. Then for MaxMatch with the same network architecture, due to involving $K$ random strong augmentations for each unlabeled data, the times spent on data transformation and feed forward are roughly increased by a factor of $K$. Nonetheless, the back-propagation and optimization processes of MaxMatch take basically the same time as FixMatch, since their numbers of parameters are identical. In summary, the total training time is at most $T(K(t_{\text{data}} + t_{\text{ff}}) + t_{\text{bp}} + t_{\text{opt}})$ for MaxMatch, and $T(t_{\text{data}} + t_{\text{ff}} + t_{\text{bp}} + t_{\text{opt}})$ for FixMatch. On the other hand, during the test phase, the time complexity of MaxMatch is exactly the same as FixMatch.

\section{Experiments}\label{sec:exp}

\begin{figure*}[t]
	\centering
	\subfloat[\#labels=40]{\label{fig:cifar10-40}\includegraphics[width=0.32\textwidth]{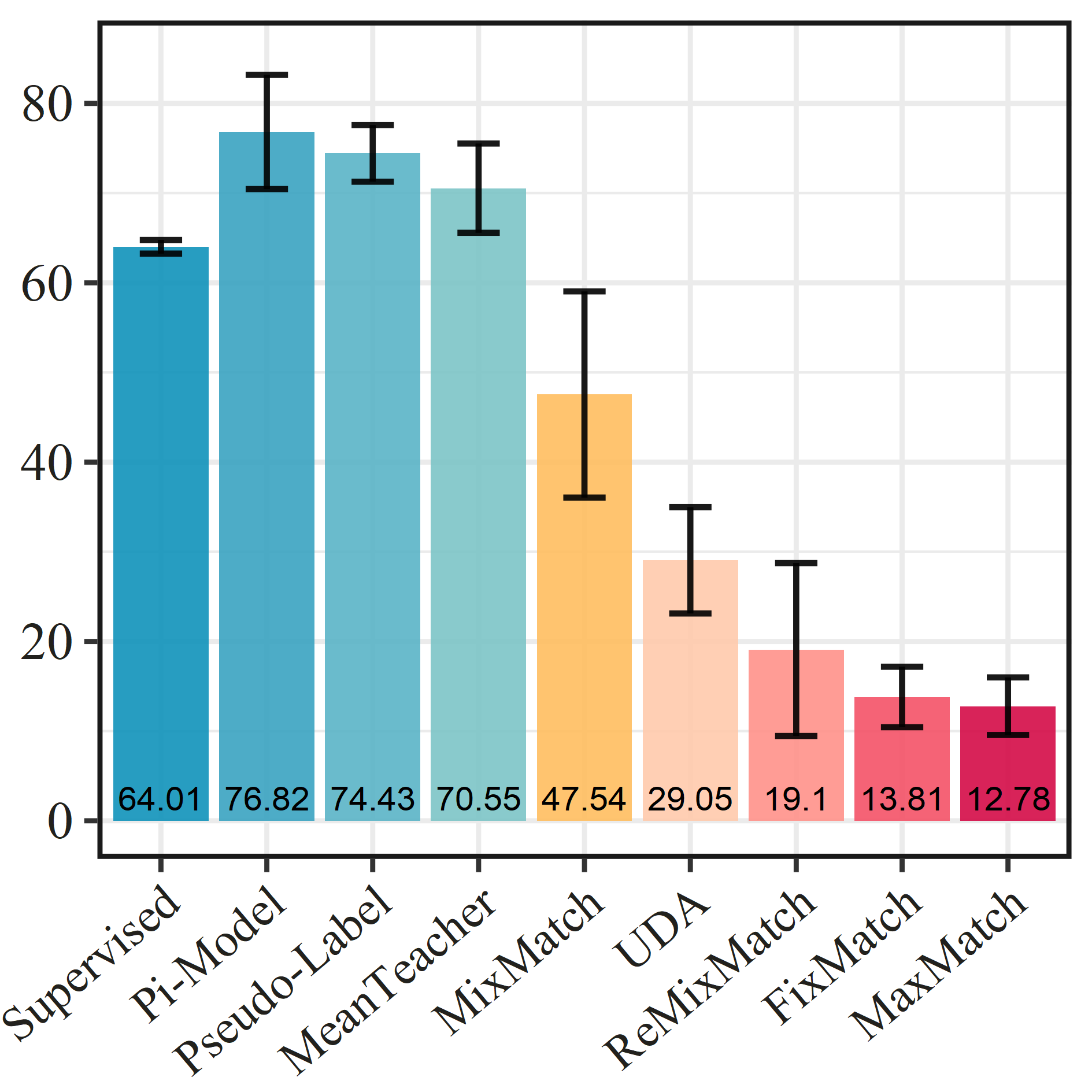}}~
	\subfloat[\#labels=250]{\label{fig:cifar10-250}\includegraphics[width=0.32\textwidth]{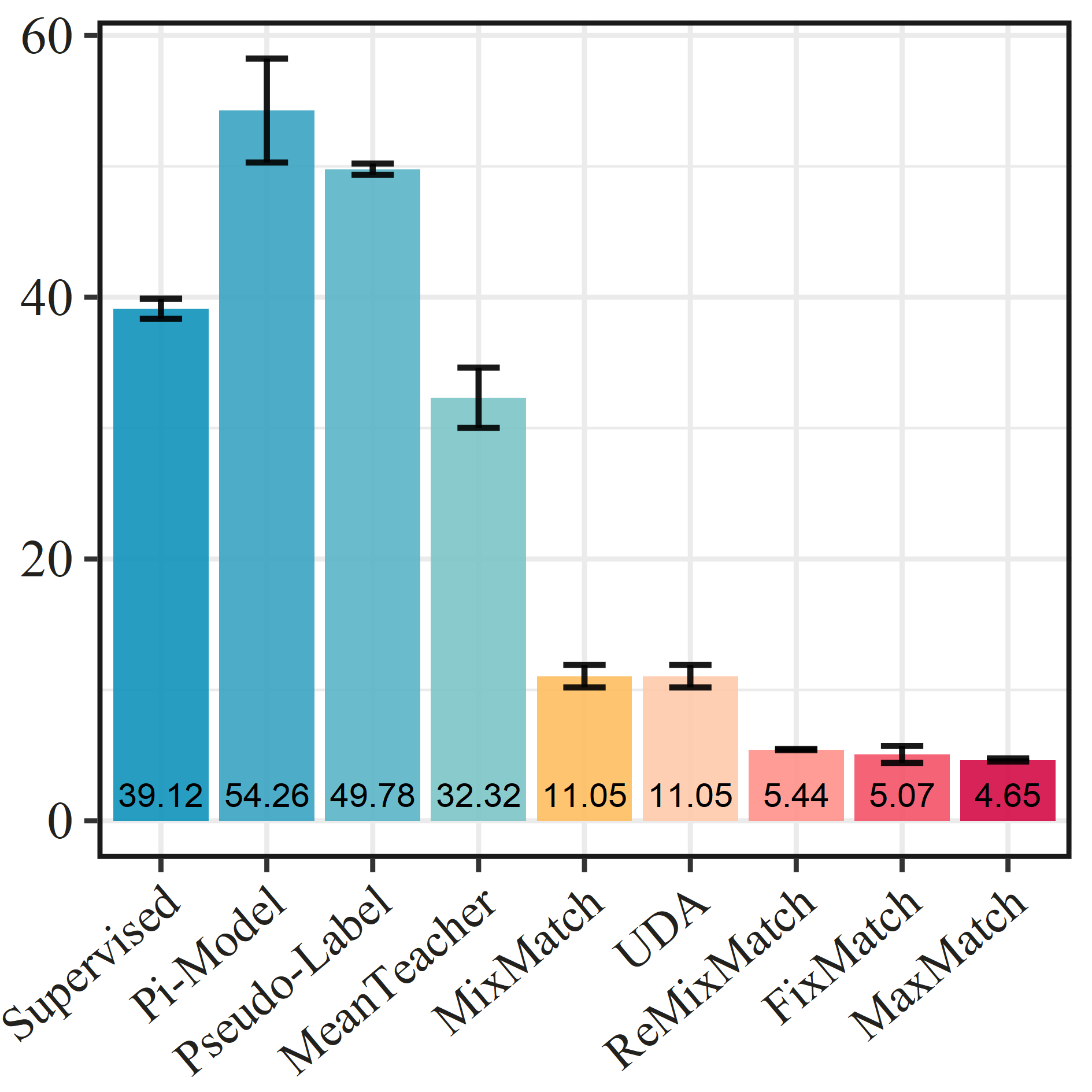}}~
	\subfloat[\#labels=4000]{\label{fig:cifar10-4000}\includegraphics[width=0.32\textwidth]{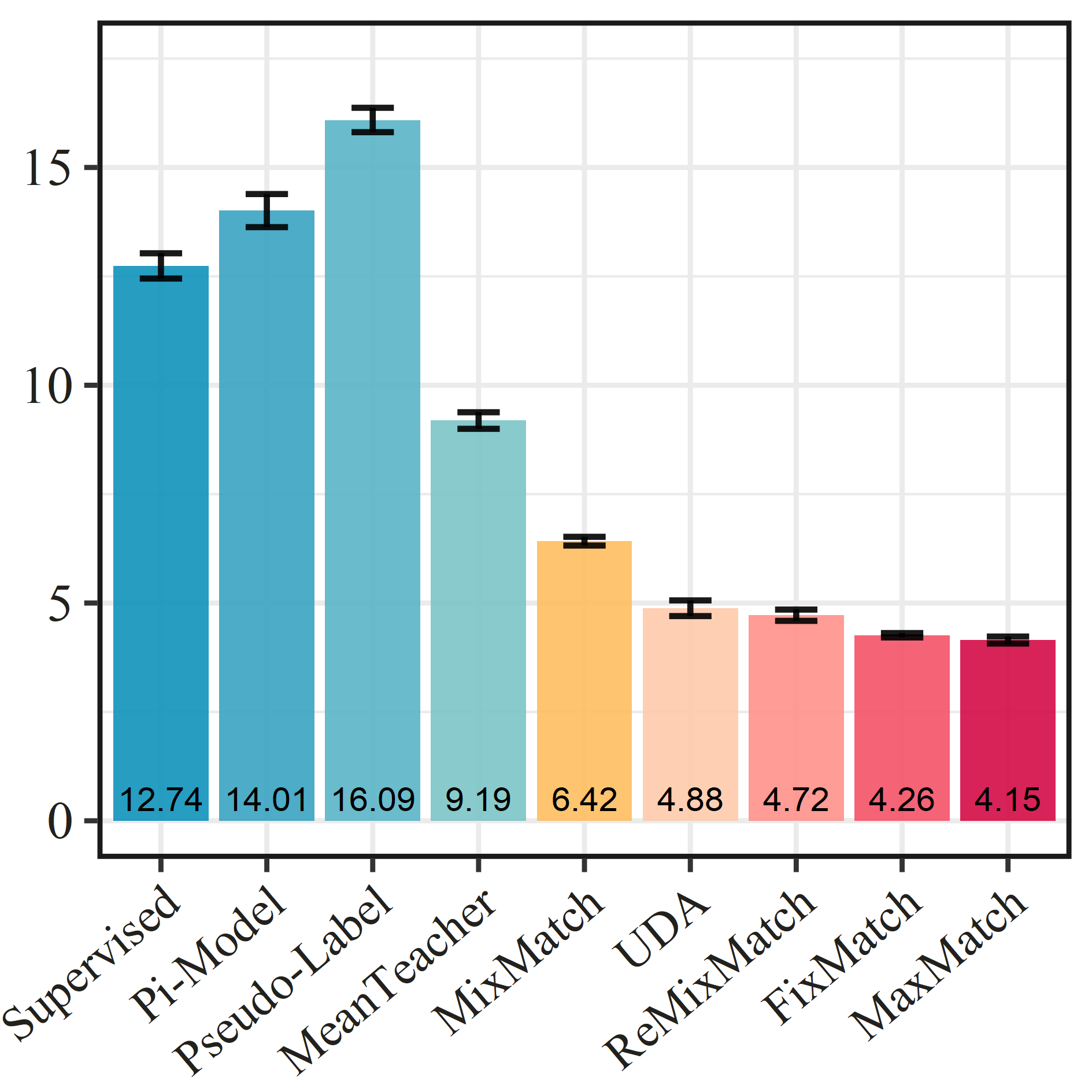}}
	\caption{Error rate (\%) over 5 different folds with varying labeled set size on CIFAR-10.}
	\label{fig:res_vary_sup_cifar10}
\end{figure*}

\begin{figure*}[t]
	\centering
	\subfloat[\#labels=400]{\label{fig:cifar10-400}\includegraphics[width=0.32\textwidth]{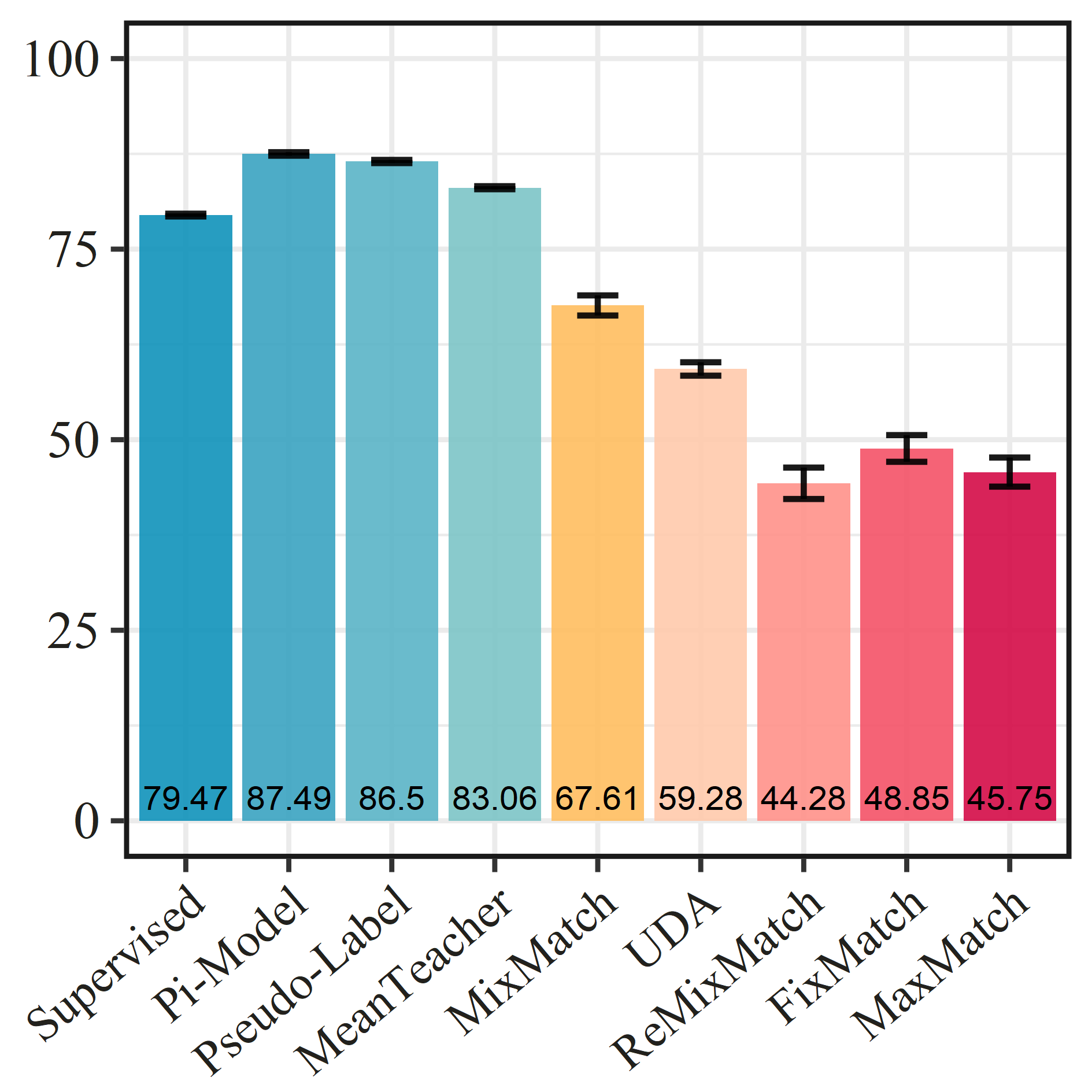}}~
	\subfloat[\#labels=2500]{\label{fig:cifar100-2500}\includegraphics[width=0.32\textwidth]{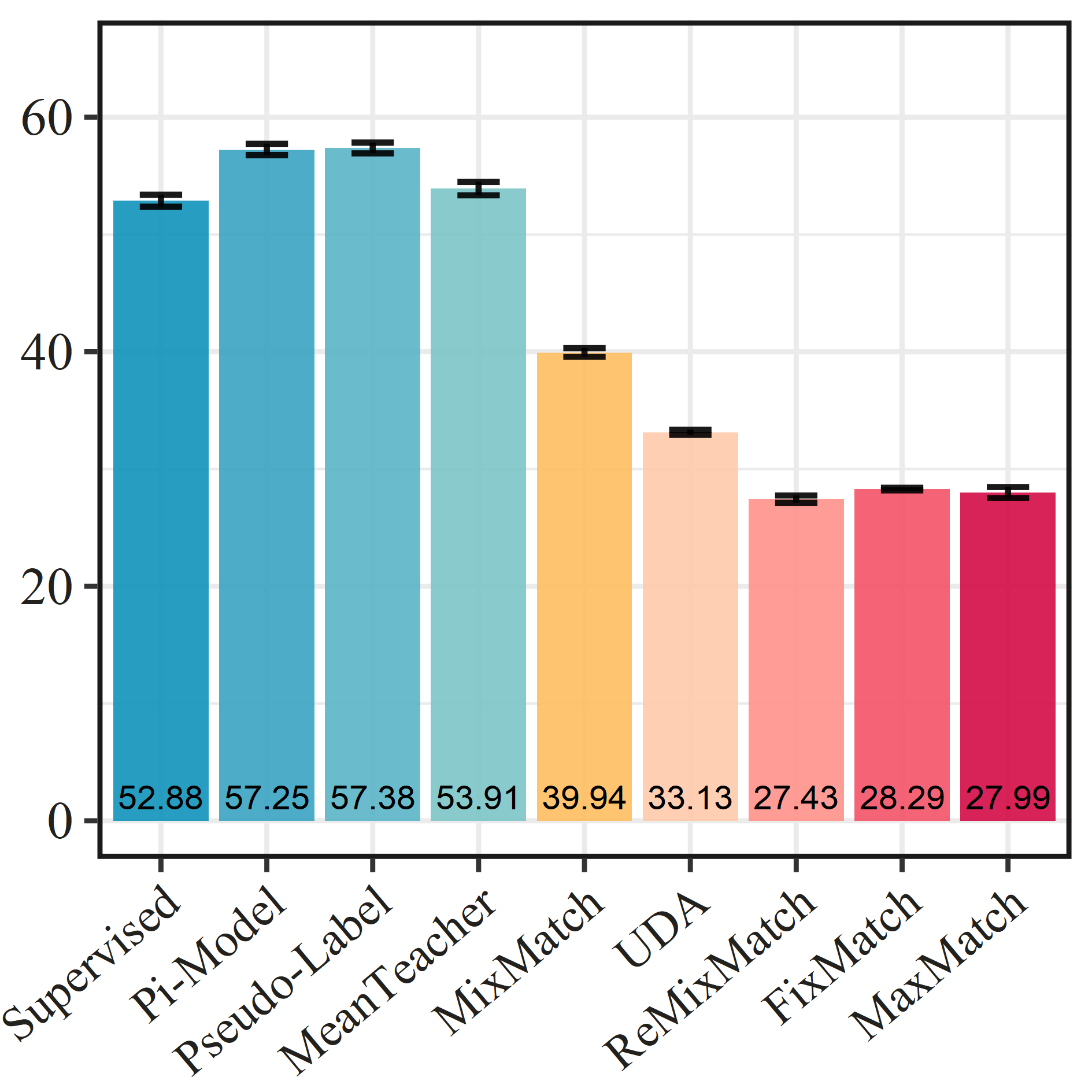}}~
	\subfloat[\#labels=10000]{\label{fig:cifar100-10000}\includegraphics[width=0.32\textwidth]{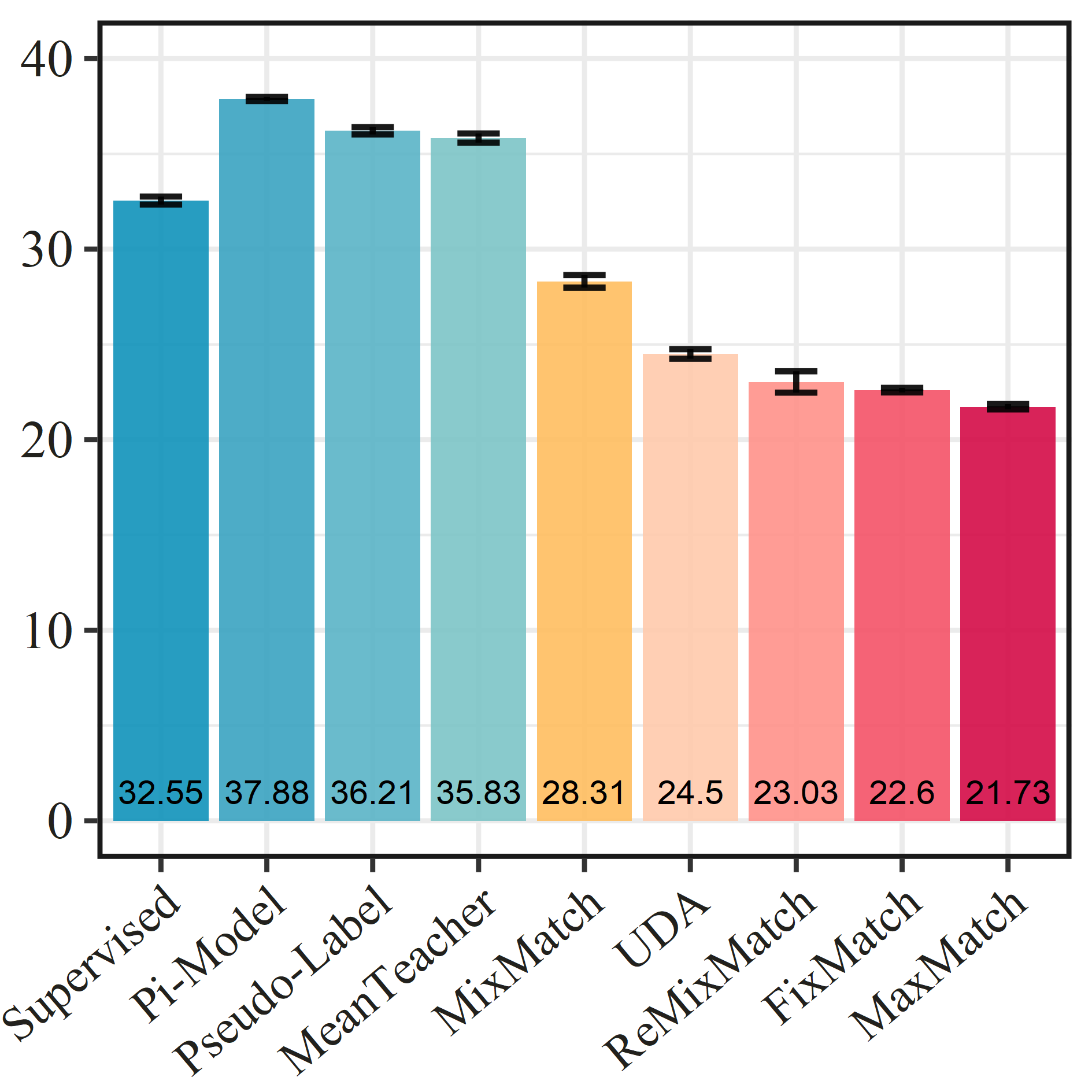}}
	\caption{Error rate (\%) over 5 different folds with varying labeled set size on CIFAR-100.}
	\label{fig:res_vary_sup_cifar100}
\end{figure*}

\subsection{CIFAR-10, CIFAR-100, SVHN and STL-10}
\noindent{\textbf{Datasets and Evaluation Protocols.}}\quad Following the standard SSL evaluation protocols~\cite{FixMatch, MixMatch, UDA}, we evaluate our method on four popular benchmark image classification datasets: CIFAR-10, CIFAR-100, SVHN and STL-10.
\begin{itemize}
  \item CIFAR-10 \cite{cifar} contains 10 classes of real-world objects (such as airplane, automobile and bird) with 6000 images per class. In the original dataset, there are 50,000 training images and 10,000 test ones, where the test set has exactly 1000 samples for each class. All the images are colored with a size of $32\times32$.
  \item Simliar to CIFAR-10, CIFAR-100 \cite{cifar} is with 100 classes with 600 images for each class. It includes 500 training images and 100 test images per class. More categories and fewer images per category make it more challenging than CIFAR-10.
  \item SVHN \cite{svhn} consists of more than 600,000 natural scene digit and number images with a size of $32\times32$ collected from Google Street View, including 73,257 digits for training, 26,032 for testing, and 531,131 as extra training data.
  \item STL-10 \cite{stl10} is composed of 5000 labeled training images and 8000 test ones from 10 classes, together with 100,000 unlabeled images. Different from other three datasets, the size of images in STL-10 is $96\times 96$, which is of higher resolution than previous datasets. Moreover, this dataset is more realistic since the unlabeled samples are drawn from a similar but different distribution from the labeled data.
\end{itemize}

\noindent Our evaluation is carried out with $4/25/400$ labeled samples per class on CIFAR-10, $4/25/100$ labeled samples per class on CIFAR-100 and SVHN, $100$ labeled samples per class on STL-10 following the setting of \cite{FixMatch}. For simplicity, we use `D@$N$' to denote the setting using $N$ labels on `D' dataset. Besides the selected labeled images, all the images in original training set are used as unlabeled samples for SSL training. And the original test set is used for performance evaluation. The overall classification error rate is adopted as the evaluation metric. Specifically, the mean and variance of error rates for $5$ different ``folds'' of labeled data are provided as \cite{FixMatch} reported.

\begin{figure*}[t]
	\centering
	\subfloat[\#labels=40]{\label{fig:svhn40}\includegraphics[width=0.32\textwidth]{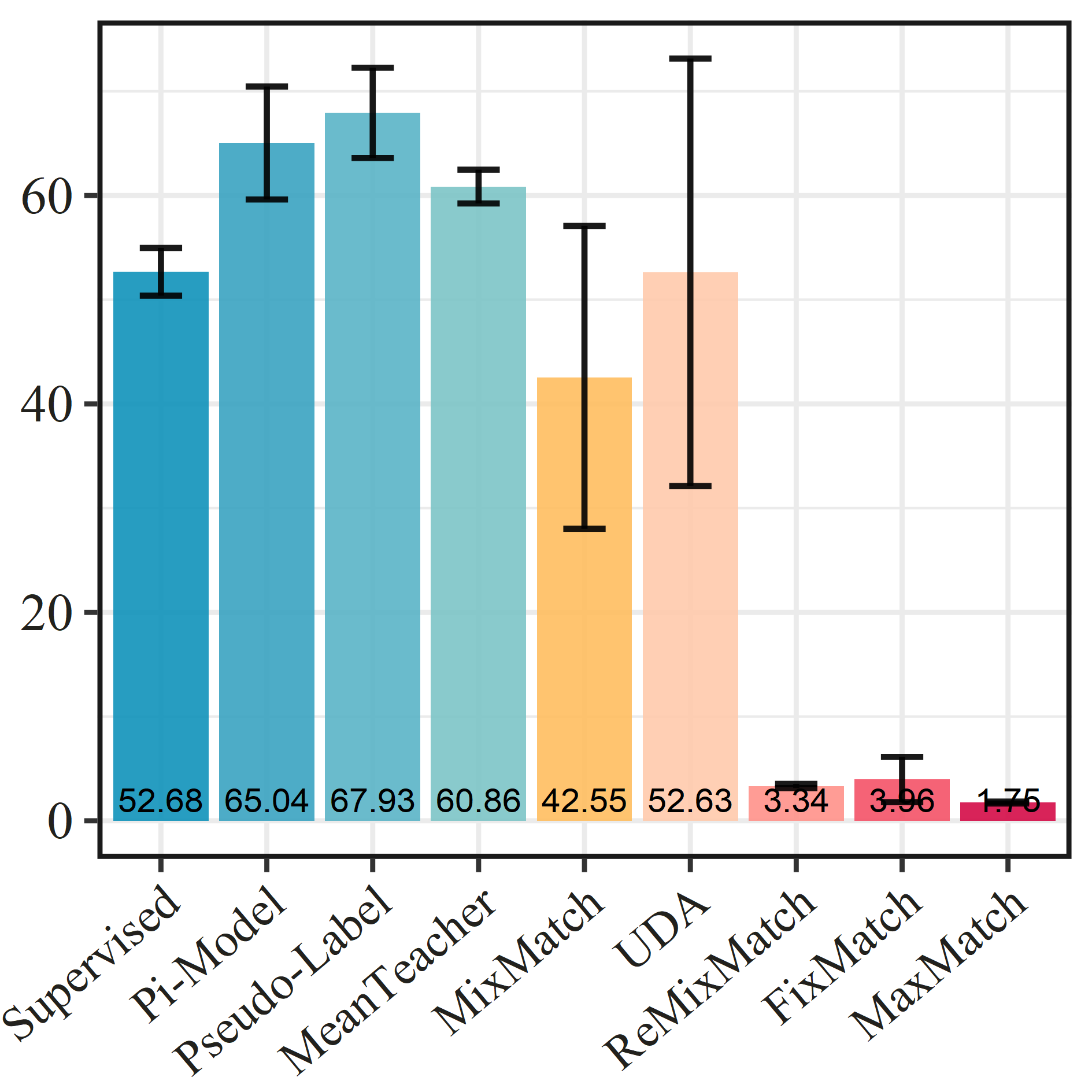}}~
	\subfloat[\#labels=250]{\label{fig:svhn250}\includegraphics[width=0.32\textwidth]{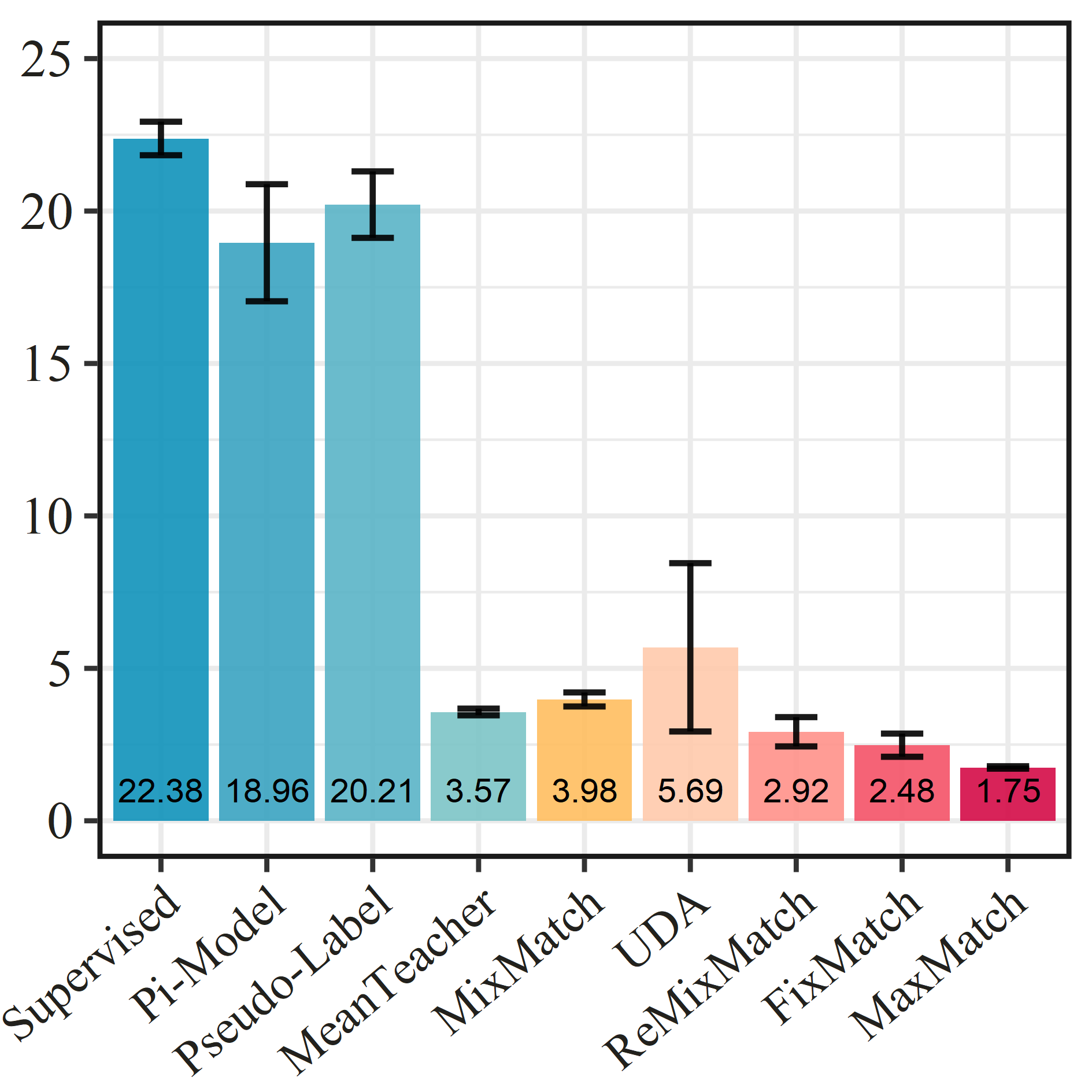}}~
	\subfloat[\#labels=1000]{\label{fig:svhn1000}\includegraphics[width=0.32\textwidth]{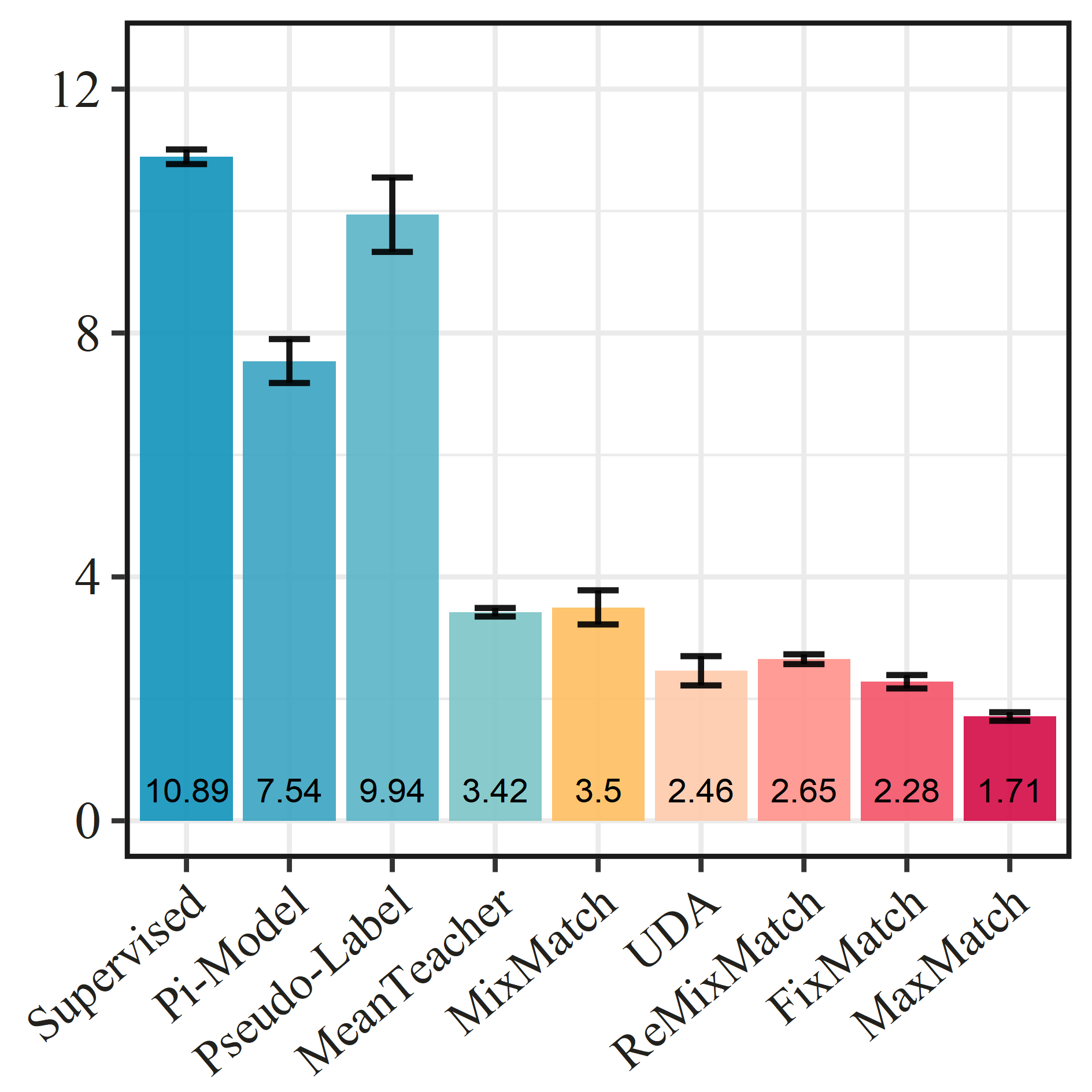}}
	\caption{Error rate (\%) over 5 different folds with varying labeled set size on SVHN.}
	\label{fig:res_vary_sup_svhn}
\end{figure*}

\begin{figure}[t]
	\centering
	\includegraphics[width=0.32\textwidth]{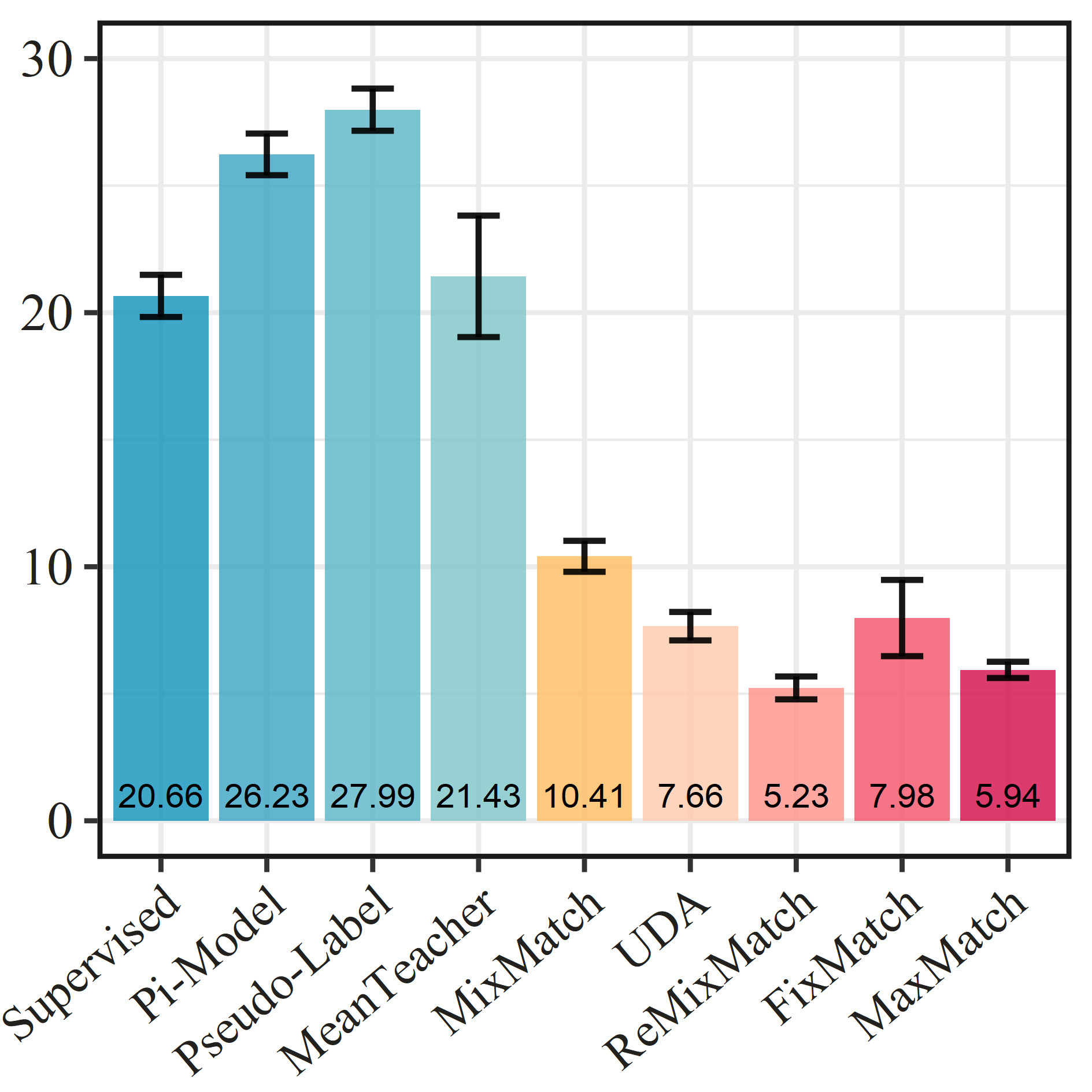}
	\caption{Error rate (\%) over 5 different folds with 1000 labeled samples on STL-10.}
	\label{fig:res_stl10}
\end{figure}

\noindent{\textbf{Competitors.}}\quad We adopt several common consistency-based SSL models as our competitors.
\begin{itemize}
	\item \textbf{Pseudo-Label} \cite{PL} is a standard self-training method that takes highly confident predicted labels of unlabeled data as their training targets in a mini-batch.
	\item \textbf{$\Pi$-Model} \cite{Pi_model} is a na\"{i}ve consistency regularization SSL method. It feeds an unlabeled sample twice with Gaussian noise and model dropout, and minimizes the difference between the predictions.
	\item \textbf{Mean Teacher} \cite{MT} has a teacher model which is the exponential moving average of past model parameters, and uses the teacher's predicted scores to guide the learning of the student.
	\item \textbf{MixMatch} \cite{MixMatch} adopts the average predictions of $K$ augmentations of an unlabeled sample as its soft-label, and introduces the MixUp mechanism to mix both labeled and unlabeled data for training.
	\item \textbf{UDA} \cite{UDA} introduces the stronger augmentation strategy and encourages the consistency between strongly augmented and weakly augmented data.
	\item \textbf{ReMixMatch} \cite{ReMixMatch} further enhances MixMatch by an improved augmentation scheme and a distribution alignment technique.
	\item \textbf{FixMatch} \cite{FixMatch}, the state-of-the-art model, integrates the pseudo-labeling with consistency regularization using strong augmentation.
\end{itemize}
Besides, we also implement models only trained on labeled data but equipped with a strong augmentation strategy (denoted as \textbf{Supervised}) as the baseline to show the improvement from unlabeled data.

\noindent{\textbf{Implementation details.}}\quad
Following \cite{FixMatch}, all the models are implemented with a Wide-ResNet-28-2 \cite{wrn} backbone for CIFAR-10 and SVHN, while with Wide-ResNet-28-8 for CIFAR-100 and Wide-ResNet-37-2 for STL-10, since the last two datasets are more challenging. In our model, we set the number of transformations for each unlabeled data as $K=3$, while other hyperparameters and training strategies are the same as FixMatch's. Specifically, we set the trade-off coefficient as $\lambda=1$, labeled batch size as $B_l=64$, unlabeled batch size as $B_u=448$, and the training epoch is $1024$. The SGD solver is adopted for minimization, with the initial learning rate $\alpha=0.03$, a momentum of $0.9$, and Nesterov acceleration enabled. We also use the cosine learning rate scheduling strategy. The weight decay parameter is $0.0005$ for CIFAR-10/SVHN/STL-10, and $0.001$ for CIFAR-100. Exponential moving average is employed with a ratio of $0.999$. The confidence threshold is $\beta=0.95$. RandAugment \cite{randaug} is adopted as the transformation strategy in our proposed model. Accordingly, the performance of both FixMatch and supervised model is reported when using RandAugment.


\noindent{\textbf{Performance comparison.}}\quad
The results of each method on CIFAR-10, CIFAR-100, SVHN and STL-10 are reported in \Figref{fig:res_vary_sup_cifar10}, \ref{fig:res_vary_sup_cifar100}, \ref{fig:res_vary_sup_svhn} and \ref{fig:res_stl10}, respectively. Overall, our proposed MaxMatch could outperform all the competitors under most settings except for CIFAR-100@$2500$, CIFAR-100@$400$ and STL-10@$1000$. Based on these results, we further have the following observations:
\begin{itemize}
	\item[(1)] Compared with the supervised baseline, the semi-supervised methods using strong augmentation strategies could improve the performance successfully. However, for those only adopting the weak augmentation (\ie, $\Pi$-Model, Pseudo-Label and Mean Teacher), without a relatively large amount of labeled data, the performance of the semi-supervised learning scheme is even worse than the supervised baseline. This shows the strength of advanced augmentation strategies.
	\item[(2)] MaxMatch achieves lower average error rates than its base model FixMatch over all the datasets, which implies that the proposed worst-case consistency regularization scheme is helpful to improve generalization performance.
	\item[(3)] With only 4 labels per category, MaxMatch could achieve a very low average error rate ($1.92\%$) on SVHN, and its variance ($0.12\%$) is much smaller than that of FixMatch ($2.17\%$). Under other two settings on this dataset, MaxMatch also has a smaller variance than FixMatch. These results indicate that MaxMatch is able to minimize the discrepancy between the hardest augmented data very well. In other words, the supremum over the uncertainty set is well-optimized; thus, MaxMatch successfully lowers the overall semi-supervised classification risk and makes the learning more stable.
	\item[(4)] Among the models involving strong augmentation, ReMixMatch and FixMatch are two most competitive methods. Although MaxMatch consistently outperforms FixMatch over all the settings, we see there are three cases (\eg, CIFAR-100@2500, CIFAR-100@400 and STL-10@1000) in which MaxMatch's performance is not better than ReMixMatch's. Note that FixMatch also performs worse than ReMixMatch under these settings. The reason might be that the three cases are much harder than others due to less labeled data but more diverse categories or out-of-distribution samples, while ReMixMatch equips with more complicated training techniques, such as the self-supervised loss and distribution alignment operation, to help mitigate the overfitting/underfitting issues. In fact, \cite{FixMatch} also verifies that when combined with distribution alignment, FixMatch could substantially outperform ReMixMatch in these cases. Nevertheless, the improvement over FixMatch through the worst-case consistency still makes MaxMatch more comparable to ReMixMatch.
\end{itemize}
In a nutshell, the quantitative results verify the effectiveness of the proposed worst-case consistency regularization scheme.

\begin{figure}[t]
	\centering
	\includegraphics[width=0.95\columnwidth]{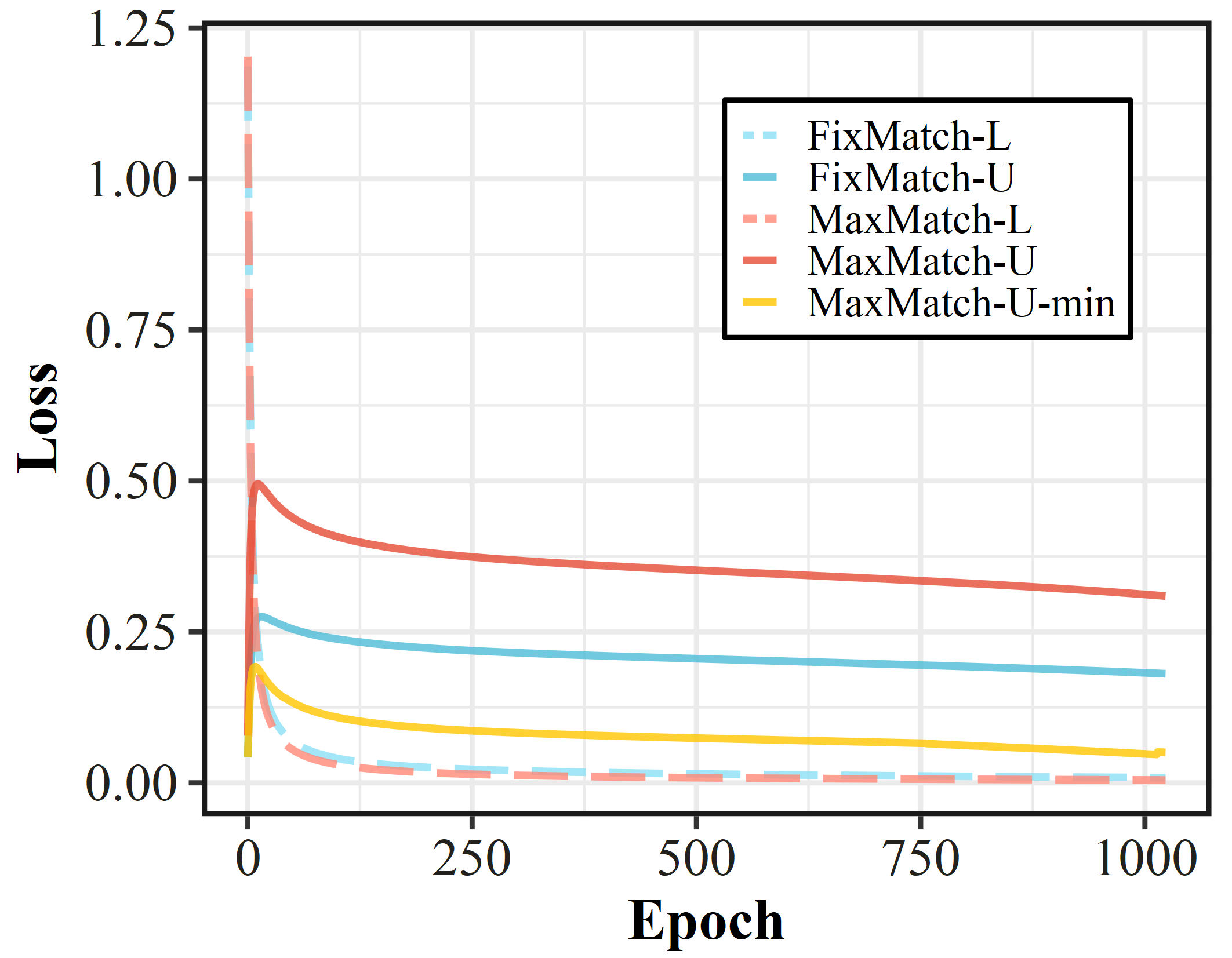}
	\caption{Values of loss terms on labeled and unlabeled data of MaxMatch and FixMatch using 4000 labels on CIFAR-10. The loss on labeled samples is denoted with the suffix \textbf{L}, and that on unlabeled samples is with the suffix \textbf{U}. Moreover, in contrast to the worst-case loss in MaxMatch, we also plot the best-case loss on unlabeled samples as \textbf{MaxMatch-U-min}.}
	\label{fig:loss}
\end{figure}

\noindent{\textbf{Change of losses.}}\quad To see the convergence behavior of our method, we also plot the change of the losses of our MaxMatch and FixMatch trained on CIFAR-10 dataset with 4000 labels in \Figref{fig:loss}. In this figure, the losses on labeled and unlabeled data for each method are illustrated in solid and dotted lines, respectively. For example, for MaxMatch, the labeled loss (denoted with the suffix \textbf{L}) is the first term in \Eqref{eq:minimax_all}, and the unlabeled one (denoted with the suffix \textbf{U}) is the second term. We can see that the two models share almost the same labeled loss which rapidly decreases to $0$ when the number of epochs increases. Meanwhile, the unlabeled losses first increase then slowly decrease. The fast increase at the beginning of the learning might be caused by the confidence thresholding technique controlled by $\beta$. Since the initial model cannot accurately and confidently predict the correct label (\ie, $\tilde{s}_{i,k}^u < \beta$ in \Eqref{eq:implement}), the unlabeled loss will be eliminated in this case. After several epochs, the model becomes stronger and makes more confident predictions, and thus the unlabeled loss suddenly gets large. As the training goes on, the unlabeled loss is minimized, which then stably decreases during training. Moreover, it could be observed that MaxMatch's unlabeled loss is consistently greater than FixMatch's. This is in accordance with our motivation of minimizing the worst-case consistency, which obviously causes a larger consistency loss in each step. Furthermore, the gap between the two unlabeled losses becomes smaller as the learning goes on, which validates the correctness of our algorithm. Besides, we also plot the best-case unlabeled loss (denoted with the suffix \textbf{U-min}) which is the minimal consistency loss among the $K$ augmented variants. We can find that over unlabeled data, the best-case loss of MaxMatch (`MaxMatch-U-min') is much less than the loss of FixMatch (`FixMatch-U'). Theoretically, for a random variable $X$ and a set of $K$ random variables $\tilde{X}_1,\cdots, \tilde{X}_K$ drawn from the same distribution, it apparently holds that $\E[\min_i \tilde{X}_i] \leq \E[X] \leq \E[\max_i \tilde{X}_i]$. As for our case, the consistency losses of different augmented variants could be regarded as random variables. Then the expected value of best-case (\ie, minimal) consistency losses among $K$ augmented variants is no larger than that of consistency losses using a single augmented variant.
With aggressive augmentation strategies like RandAugment, the worst-case (\textit{resp.} best-case) loss among multiple variants will be much larger (\textit{resp.} smaller) than the single variant's loss. This is in line with our expectation. Compared with the best-case loss, our worst-case loss can minimize the inconsistency for all augmented variants in the neighborhood set instead of the best one. In this sense, the worst-case is much more robust.

\begin{figure}[t]
	\centering
	\includegraphics[width=0.75\columnwidth]{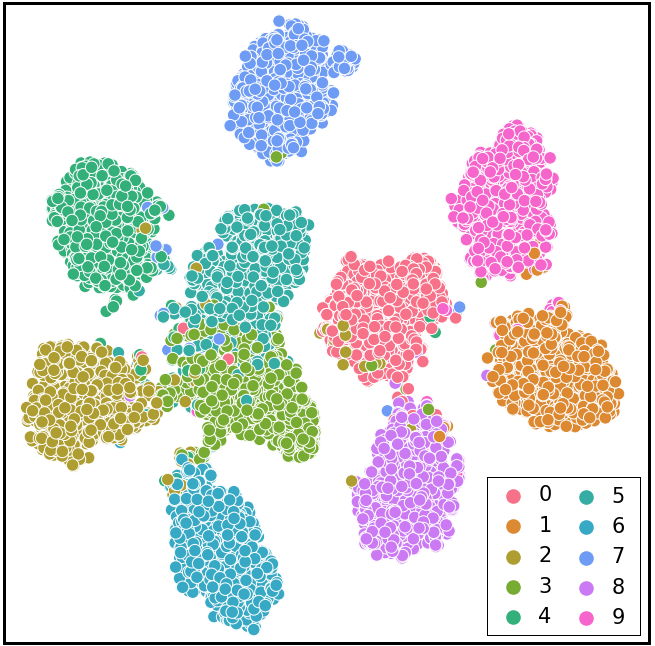}
	\caption{t-SNE visualization for feature embeddings of MaxMatch using 250 labels on CIFAR-10.}
	\label{fig:tsne}
\end{figure}

\noindent{\textbf{Visualization.}}\quad The feature embeddings obtained by MaxMatch on CIFAR-10 with $250$ labels are visualized in \Figref{fig:tsne} using t-SNE \cite{tsne}. Each class is plotted with a specific color. As shown in the figure, features scatter and gather distinctly. Also, the boundaries between a large proportion of classes are clear. With limited labeled data, MaxMatch is able to discriminate the data points as much as possible.

\subsection{ImageNet-1K}

\noindent{\textbf{Dataset and Evaluation Protocol.}}\quad ImageNet-1K \cite{imagenet} is a large-scale dataset contains over $1.28$ million data from $1000$ categories. The original dataset consists of $1,281,167$ training images and $50,000$ validation samples. Following \cite{FixMatch}, we take 10\% of the training images as labeled data, and the rest as unlabeled data for SSL training. Then we report the overall classification error rate on the original validation set for performance evaluation.

\noindent{\textbf{Competitors.}}\quad In addition to FixMatch, we compare MaxMatch with the representative self-supervised method MoCo V2 \cite{chen2020mocov2}. Specifically, this baseline is firstly pretrained by MoCo V2 using all the training data without labels, and then finetuned with only the 10\% labeled data. Besides, we also train FixMatch and MaxMatch based on the pretrained MoCo V2 model to show the effectiveness of the worst-case regularization with self-supervised pretraining. The Exponential Moving Average Normalization (EMAN) technique \cite{eman} is also incorporated for performance comparison.

\noindent{\textbf{Implementation details.}}\quad We conduct the ImageNet-1k experiments based on the PyTorch re-implementation in \cite{eman}, which also includes the pretraining and finetuning codes using MoCo V2. The experiments are run on a server with eight NVIDIA RTX 3090 GPUs. All the models adopt ResNet-50 \cite{resnet} as the network backbone, and a standard SGD optimizer with a momentum of $0.9$ for optimization. The hyperparameter setting of MoCo V2 and FixMatch follows \cite{eman}. For the proposed MaxMatch, all the hyperparameters are the same as FixMatch's, except for the number of transformations and initial learning rate.
For this much more complicated dataset, we randomly choose the number of variants $K$ from $\{1,2,3\}$ in each iteration to prevent the early-stage learning from being dominated by extreme hard examples and trapped in poor local optima.
The model is trained with a linear warmup to the initial learning rate $\alpha=0.02$ during the first $5$ epochs, and then the learning rate is adjusted with a cosine scheduler. We set the trade-off coefficient as $\lambda=10$, labeled batch size as $B_l=64$, unlabeled batch size as $B_u=320$, and the training epoch is $300$. The confidence threshold is $\beta=0.7$ and the decay ratio in the exponential moving average strategy is $0.999$. In terms of augmentation strategies, RandAugment is used for strong augmentation, and random horizontal flip is used for weak augmentation.

\begin{table}[t]
	\centering
	\caption{Error rate (\%) on ImageNet-1k. The lower, the better.}
	\begin{tabular}{llcc}
		\toprule
		Method & Pretrain & {top-1 err} & {top-5 err} \\
		\midrule
		FixMatch & None  & 28.74 & 10.94 \\
		MaxMatch & None  & \textbf{27.83} & \textbf{10.46} \\
		\midrule
		Finetune & MoCo v2 & 36.80  & 13.98 \\
		FixMatch & MoCo v2 & 28.25 & 10.52 \\
		MaxMatch & MoCo v2 & \textbf{28.09} & \textbf{10.49} \\
		\midrule
		FixMatch+EMAN & MoCo v2 & 26.56 & 9.69 \\
		MaxMatch+EMAN & MoCo v2 & \textbf{26.09} & \textbf{9.36} \\
		\bottomrule
	\end{tabular}%
	\label{tab:imagenet}%
\end{table}%

\begin{figure}[t]
	\centering
	\includegraphics[width=0.85\columnwidth]{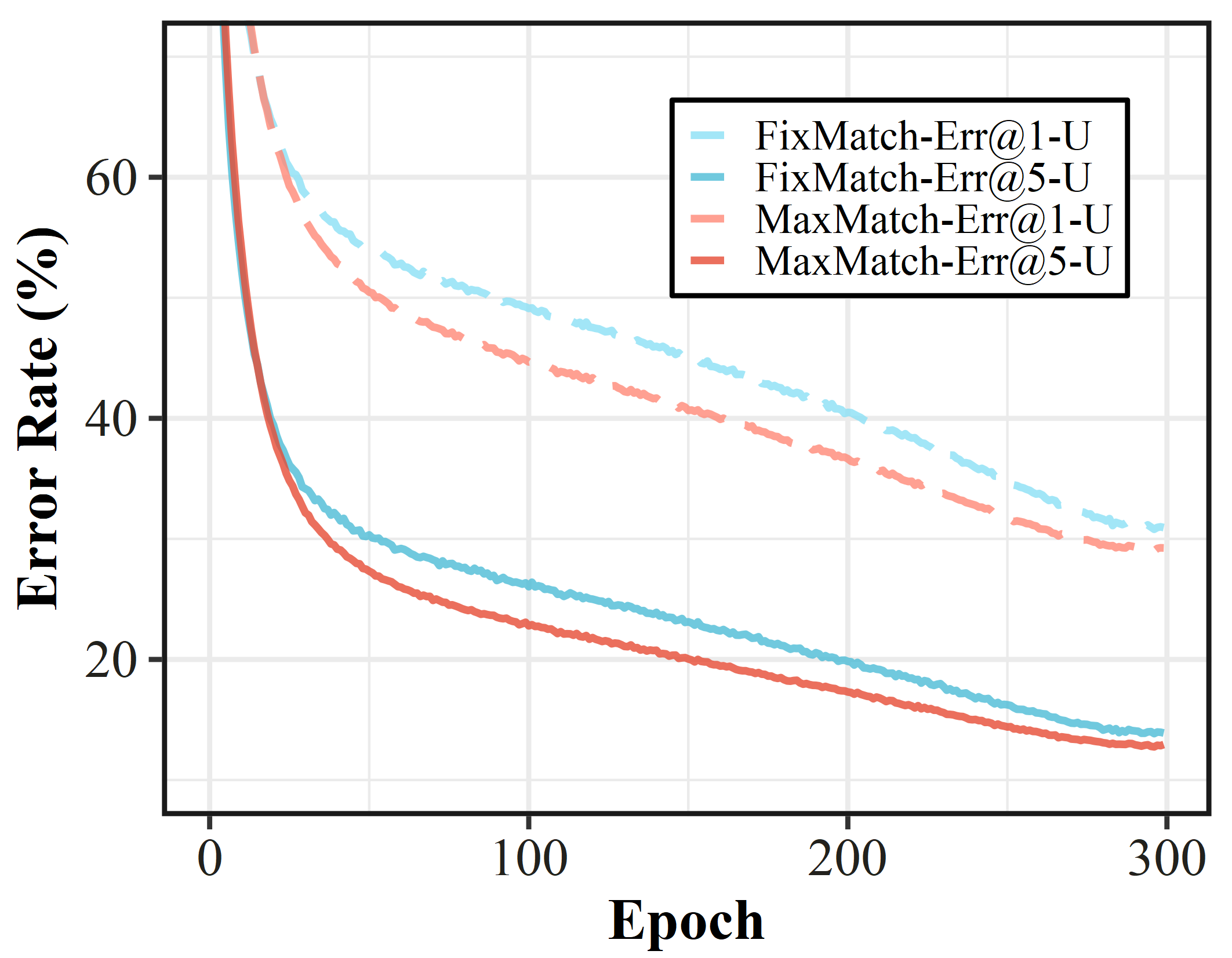}
	\caption{Change of error rates (\%) on unlabeled data of ImageNet-1k for MaxMatch and FixMatch trained from scratch.}
	\label{fig:imagenet_err_u}
\end{figure}

\noindent{\textbf{Performance comparison.}}\quad The top-1 and top-5 classification error rates are reported in \Tabref{tab:imagenet}. It could be observed that (1) all the semi-supervised models achieve higher performance than the model finetuned from MoCo V2, and (2) our proposed model outperforms FixMatch in terms of top-1 and top-5 error rates, whether they are initialized from self-supervised pretrained models or not. This validates the effectiveness of the proposed worst-case consistency regularization framework on large-scale data. Moreover, we plot the top-1 and top-5 error rates on unlabeled data for models trained from scratch in \Figref{fig:imagenet_err_u}. During the training, the error rates of MaxMatch are consistently lower than FixMatch. And the performance gap in top-1 error rates is slightly larger than top-5 error rates. Therefore, the worst-case consistency regularization is able to improve the performance over unlabeled data, especially in terms of the top-1 error rate.

Besides, we also provide the comparisons under different initial learning rates in \Tabref{tab:imagenet_lr}. As $\alpha$ increases from $0.01$ to $0.03$, the performance of FixMatch is continuously improved, while that of MaxMatch first improves and then slightly degenerates. Adopting the worst-case invariants will induce larger loss values on unlabeled data. Consequently, MaxMatch might require a smaller learning rate to achieve the best performance on such a challenging dataset. Nevertheless, with these different $\alpha$s, the proposed method still improves the performance of FixMatch in terms of top-1 and top-5 error rates, respectively. Moreover, the largest performance gain is achieved when $\alpha=0.02$. We further illustrate the top-1 test error rates for FixMatch and MaxMatch with these learning rates in \Figref{fig:imagenet_err_1}. It could be shown that the performance of MaxMatch is better than FixMatch in most cases. Therefore, the effectiveness of the proposed method is again verified.

\begin{table}[t]
	\centering
	\caption{Error rate (\%) with different initial learning rates ($\alpha$).}
		\begin{tabular}{lc|cc}
		\toprule
		Method & Initial lr & top-1 err & top-5 err \\
		\midrule
		FixMatch & 0.01  & 29.70 & 11.78 \\
		MaxMatch & 0.01  & \textbf{29.31} & \textbf{11.69} \\
		\midrule
		FixMatch & 0.02  & 28.94 & 11.30 \\
		MaxMatch & 0.02  & \textbf{27.83} & \textbf{10.46} \\
		\midrule
		FixMatch & 0.03  & 28.74 & 10.94 \\
		MaxMatch & 0.03  & \textbf{28.48} & \textbf{10.93} \\
		\bottomrule
		\end{tabular}%
	\label{tab:imagenet_lr}%
\end{table}%

\begin{figure}[t]
	\centering
	\includegraphics[width=0.95\columnwidth]{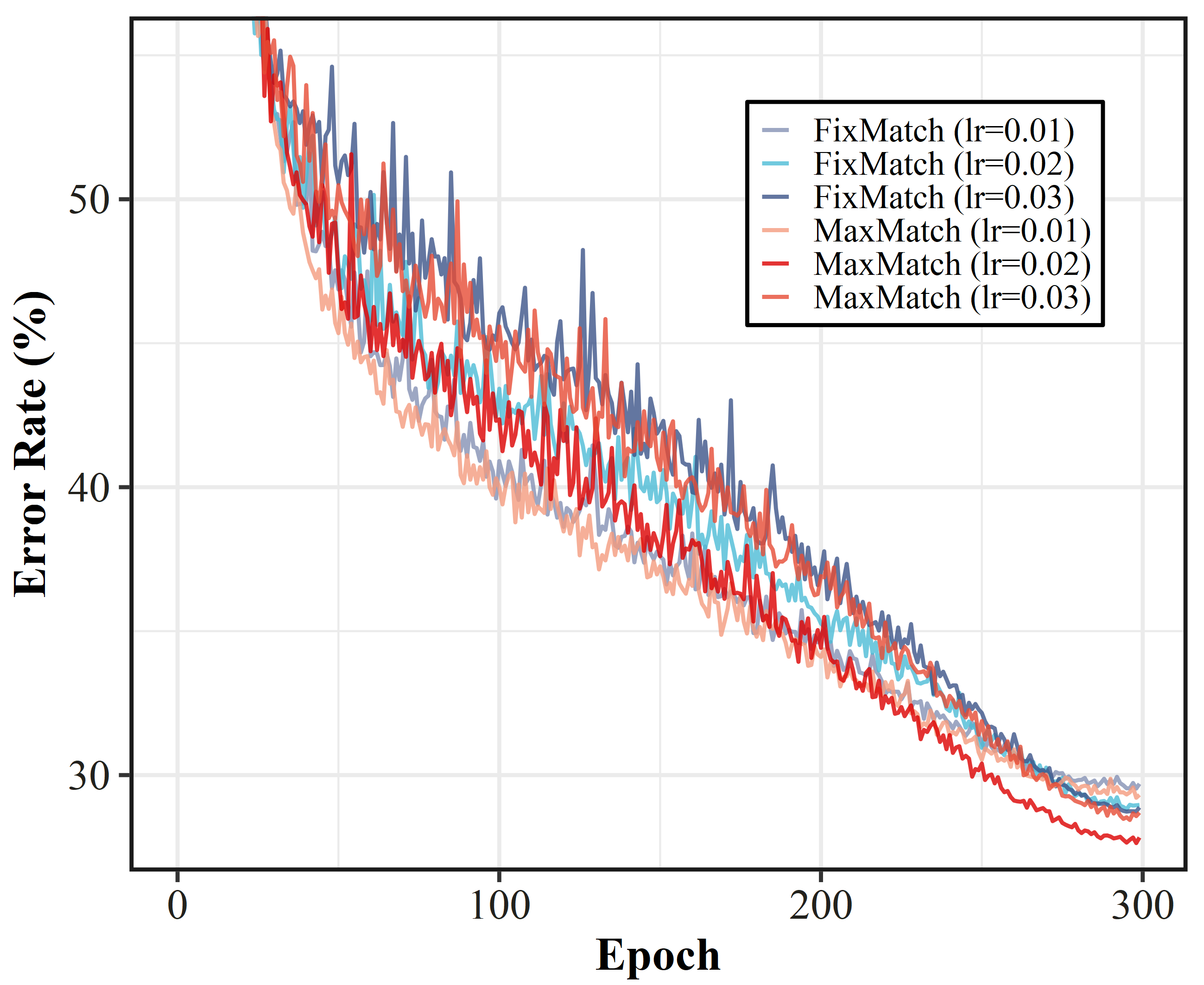}
	\caption{Change of top-1 test error rates (\%) on ImageNet-1k for MaxMatch and FixMatch with different learning rates. Results of MaxMatch are in shades of red, while those of FixMatch are in shades of blue.}
	\label{fig:imagenet_err_1}
\end{figure}

\subsection{Ablation Study and Sensivity Analysis}

\begin{table}[t]
	\centering
	\caption{Performance comparison between adopting RandAugment (RA) and CTAugment (CTA) on SVHN. The \best{best} and \secbest{second best} results are both highlighted.}
	\setlength{\tabcolsep}{1.2mm}{
	\begin{tabular}{l|ccc}
		\toprule
		& 40 & 250 & 1000 \\
		\midrule
		FixMatch (RA) & 3.96 $\pm$ 2.17   & 2.48 $\pm$ 0.38  & 2.28 $\pm$ 0.11 \\
		FixMatch (CTA) & 7.65 $\pm$ 7.65   & 2.64 $\pm$ 0.64  & 2.36 $\pm$ 0.19 \\
		\midrule
		MaxMatch (RA) & \best{1.75 $\pm$ 0.09} & \best{1.75 $\pm$ 0.04} & \best{1.71 $\pm$ 0.07} \\
		MaxMatch (CTA) & \secbest{1.92 $\pm$ 0.12}  & \secbest{1.97 $\pm$ 0.05} & \secbest{1.89 $\pm$ 0.05} \\
		\bottomrule
	\end{tabular}%
	}
	\label{tab:aug}%
\end{table}%

\noindent{\textbf{Effectiveness on different strong augmentation strategies.}}\quad In existing SSL methods, there are two popular strong augmentation strategies: RandAugment \cite{randaug} and CTAugment \cite{ReMixMatch}. Given a list of atomic transformations, RandAugment uniformly samples transformations to apply for input images at random, while CTAugment further learns magnitudes for each transformation during the training. Here we also implement MaxMatch with CTAugment (denoted as \textbf{MaxMatch (CTA)}) to see how it performs under different strong augmentation strategies, and then compare it with FixMatch's CTAugment variant (denoted as \textbf{FixMatch (CTA)}). The evaluation results on 5 folds of SVHN are shown in \Tabref{tab:aug}. First, we can see that for both FixMatch and MaxMatch, the RandAugment variant obtains lower error rates than CTAugment variant. Moreover, with varying labeled set size, MaxMatch consistently outperforms FixMatch using different strategies. Especially when only 40 labels available, MaxMatch (RA) achieves 2.21\% average performance improvement compared with FixMatch (RA), and MaxMatch (CTA) has a 5.73\% average performance gain to FixMatch (CTA). In these different settings, MaxMatch variants also show lower performance variance. This again demonstrates the effectiveness of the proposed framework.

\noindent{\textbf{Effect of $K$.}}\quad An important hyperparameter of the proposed method is $K$, which controls the uncertainty set size of the unlabeled data. Since every atomic transformation is uniformly sampled to apply in the strong augmentation strategy, a larger uncertainty set is more likely to include augmented variants more different from the original one. Namely, the inconsistency between the unlabeled sample and its augmented counterpart might be larger. To show the impact of $K$ on MaxMatch's performance, we evaluate our model with $K=\{1, 3, 5, 7\}$ on CIFAR-10 with 250/4000 labels and show the corresponding error rates in \Figref{fig:K}. Note that when $K=1$, the model is exactly FixMatch. The results show that enlarging $K$ will lead to a clear performance gain in most cases. Especially when $K=3$ for 250 labels and $K=5$ for 4000 labels, the model gets the largest performance improvement. However, the improvement decreases when $K$ becomes relatively large, \eg, $K=7$ for both label set sizes. Recall that MaxMatch focuses on the worst-case consistency, which might become harder to minimize if more and more augmented variants are involved. Thus, a moderate value such as $K=3$ would be a good choice in practice.

\begin{figure}[t]
	\centering
	\includegraphics[width=0.82\columnwidth]{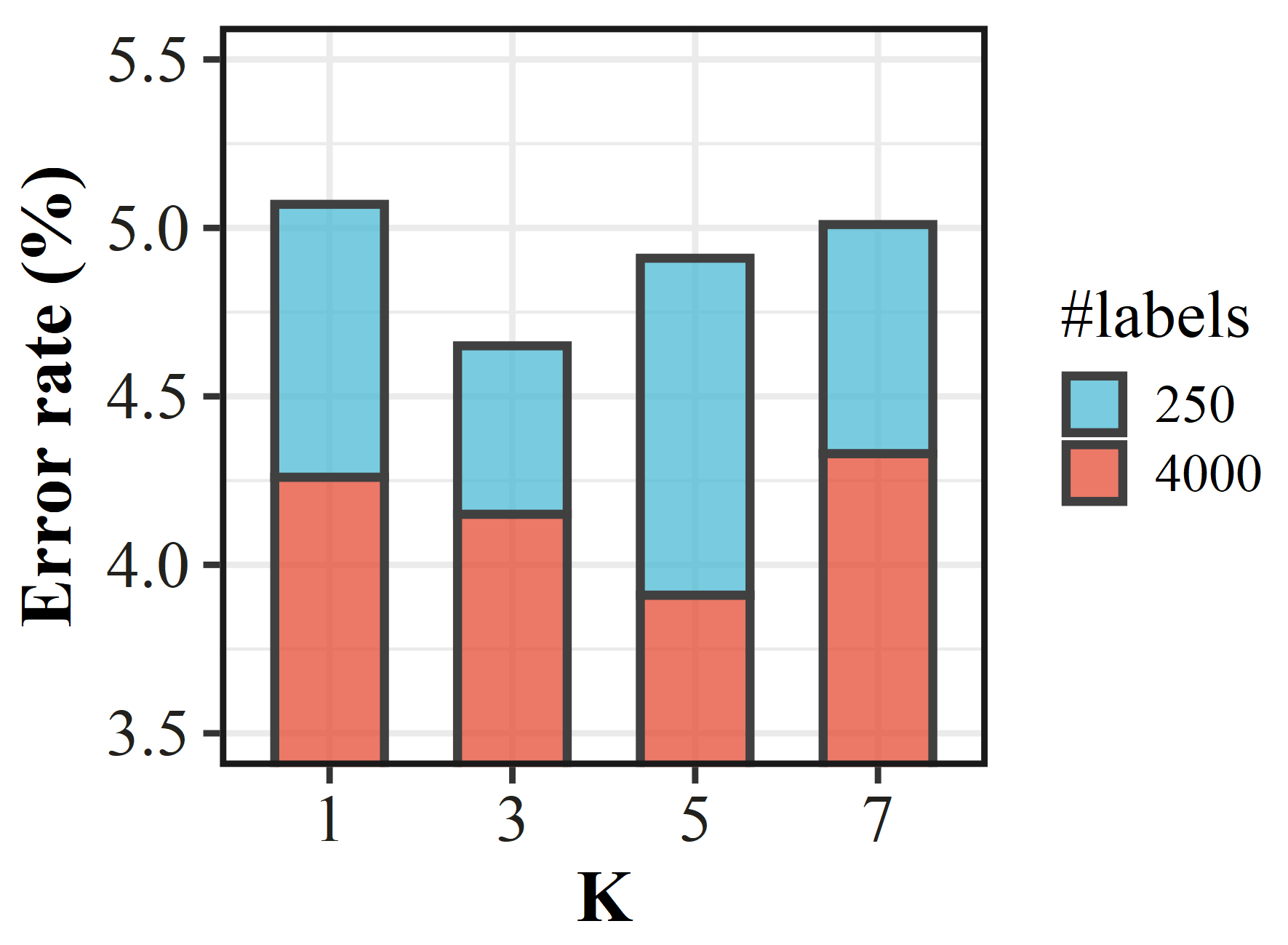}
	\caption{Error rate (\%) with varying $K$ for MaxMatch using 4000 and 250 labels on CIFAR-10.}
	\label{fig:K}
\end{figure}

\noindent{\textbf{Comparison with Augmentation Anchoring \cite{ReMixMatch}.}}\quad The Augmentation Anchoring technique and our proposed Worst-case Consistency both involve multiple augmented variants for an unlabeled sample. The main difference lies in that the former tries to minimize all the consistency losses between an unlabeled sample and its different augmented counterparts, while the latter only considers the largest inconsistency. At first glance, one might think that using all the augmented samples would be more helpful. In order to verify this, we compare MaxMatch with a variant of FixMatch which is equipped with Augmentation Anchoring (denoted as \textbf{FixMatch+Anchoring}) on SVHN when $K=3$. In other words, FixMatch+Anchoring is implemented by replacing the \textit{maximum} operation in MaxMatch with the \textit{mean} operation. As shown in \Figref{fig:anchor}, both of them achieve much lower error rates than FixMatch given different numbers of labels, since the diversity of augmented variants helps the network learn the semantic consistency. Furthermore, in two settings (SVHN@40 and SVHN@250), MaxMatch outperforms FixMatch+Anchoring. This indicates that in many cases, minimizing the worst-case consistency is sufficient to improve the model's generalization performance. Moreover, a major advantage of MaxMatch is that it enjoys a theoretical generalization performance guarantee while other methods do not. Therefore, the proposed method is of great potential for semi-supervised learning.

\begin{figure}[t]
	\centering
	\includegraphics[width=0.92\columnwidth]{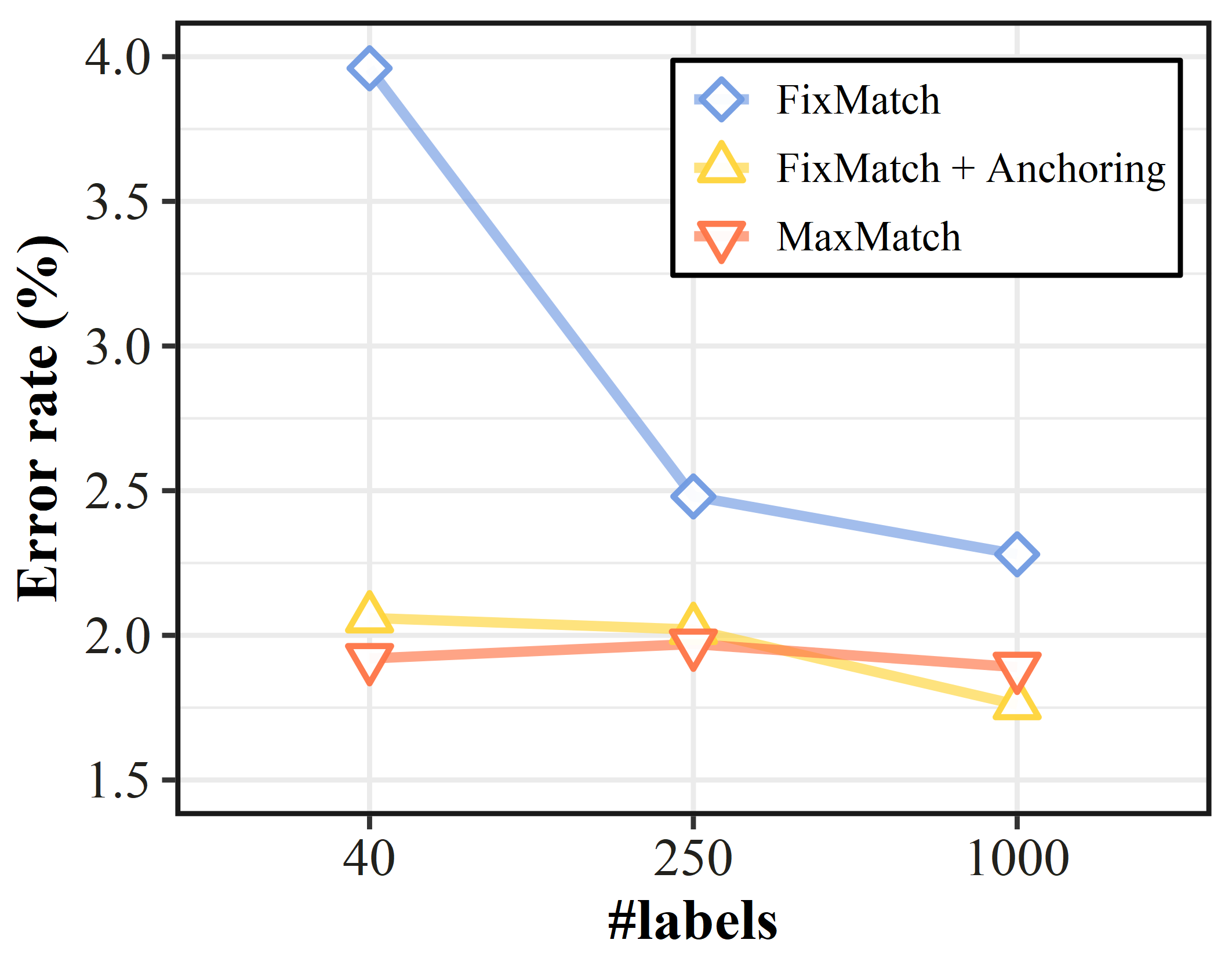}
	\caption{Performance comparison between Augmentation Anchoring \cite{ReMixMatch} and the proposed Worst-case Consistency when $K=3$ on SVHN.}
	\label{fig:anchor}
\end{figure}

\noindent{\textbf{Runtime analysis.}}\quad We also provide the training speed comparisons for FixMatch and MaxMatch (K=3) on CIFAR-10 and ImageNet-1k in \Tabref{tab:time}. It is indicated that the proposed model is about two times slower than FixMatch on CIFAR-10, and 1.5 times slower on ImageNet-1k. As we discussed before, the increased times are mainly caused by extra $K-1$ strong augmentations and the subsequent forward passes. Nevertheless, the costs during back-propagation and optimization are not affected.

\noindent{\textbf{Discussion.}}\quad As the runtime analysis shows, using multiple augmented variants for each unlabeled data will inevitably lead to extra time costs. Fortunately, according to our experimental results, a small $K$ such as $3$ might be sufficient for practical scenarios. This could mitigate this issue to some extent. Meanwhile, the proposed method is designed for the basic semi-supervised classification task. How to deal with more complicated semi-supervised learning (SSL) scenarios such as open-set SSL, noisy label SSL, SSL with distribution shift, and continuous SSL remains to be explored in future research.

\begin{table}[t]
	\caption{GPU hours per epoch during the training. The speed is tested on an NVIDIA RTX 3090 GPU for CIFAR-10 and eight NVIDIA RTX 3090 GPUs for ImageNet-1k.}
	\centering
	\begin{tabular}{l|cc}
		\toprule
		Dataset & FixMatch & MaxMatch \\
		\midrule
		CIFAR-10 & 0.069 & 0.147 \\
		ImageNet-1k & 3.016 & 4.512 \\
		\bottomrule
	\end{tabular}
	\label{tab:time}
\end{table}

\section{Conclusion}\label{sec:conclu}
In this work, we present a new consistency regularization-based semi-supervised learning framework called \textit{MaxMatch}. Given multiple augmented variants of an unlabeled sample, MaxMatch tries to minimize the worst-case consistency between the original sample and these augmented copies to make the model prediction to be robust towards data augmentation. Specifically, we first theoretically prove that the semi-supervised classification risk is bounded by the worst-case consistency regularization, the supervised classification risk, and the generalization gap, which is with magnitude $\tilde{O}(K\sqrt{\frac{W_g}{\nuu}}+\sqrt{\frac{W}{\nl}})$ for convolutional networks. We further derive a minimax optimization problem based on theoretical implications. Then we present an alternative optimization scheme to solve the problem, where its convergence property is guaranteed theoretically. Finally, empirical studies on CIFAR-10, CIFAR-100, SVHN, STL-10 and ImageNet-1k consistently show that our proposed method could reach promising performance in most cases.


%



\ifCLASSOPTIONcompsoc
  \section*{Acknowledgments}
\else
  \section*{Acknowledgment}
\fi
This work was supported in part by the National Key R\&D Program of China under Grant 2018AAA0102000, in part by National Natural Science Foundation of China: U21B2038, 61931008, 62025604, U1936208, 6212200758 and 61976202, in part by the Fundamental Research Funds for the Central Universities, in part by Youth Innovation Promotion Association CAS, in part by the Strategic Priority Research Program of Chinese Academy of Sciences, Grant No. XDB28000000, in part by the China National Postdoctoral Program for Innovative Talents under Grant BX2021298, and in part by China Postdoctoral Science Foundation under Grant 2022M713101.

\ifCLASSOPTIONcaptionsoff
  \newpage
\fi


\bibliographystyle{IEEEtran}
\bibliography{ref}

%
\begin{IEEEbiography}[{\includegraphics[width=1in,height=1.25in,clip,keepaspectratio]{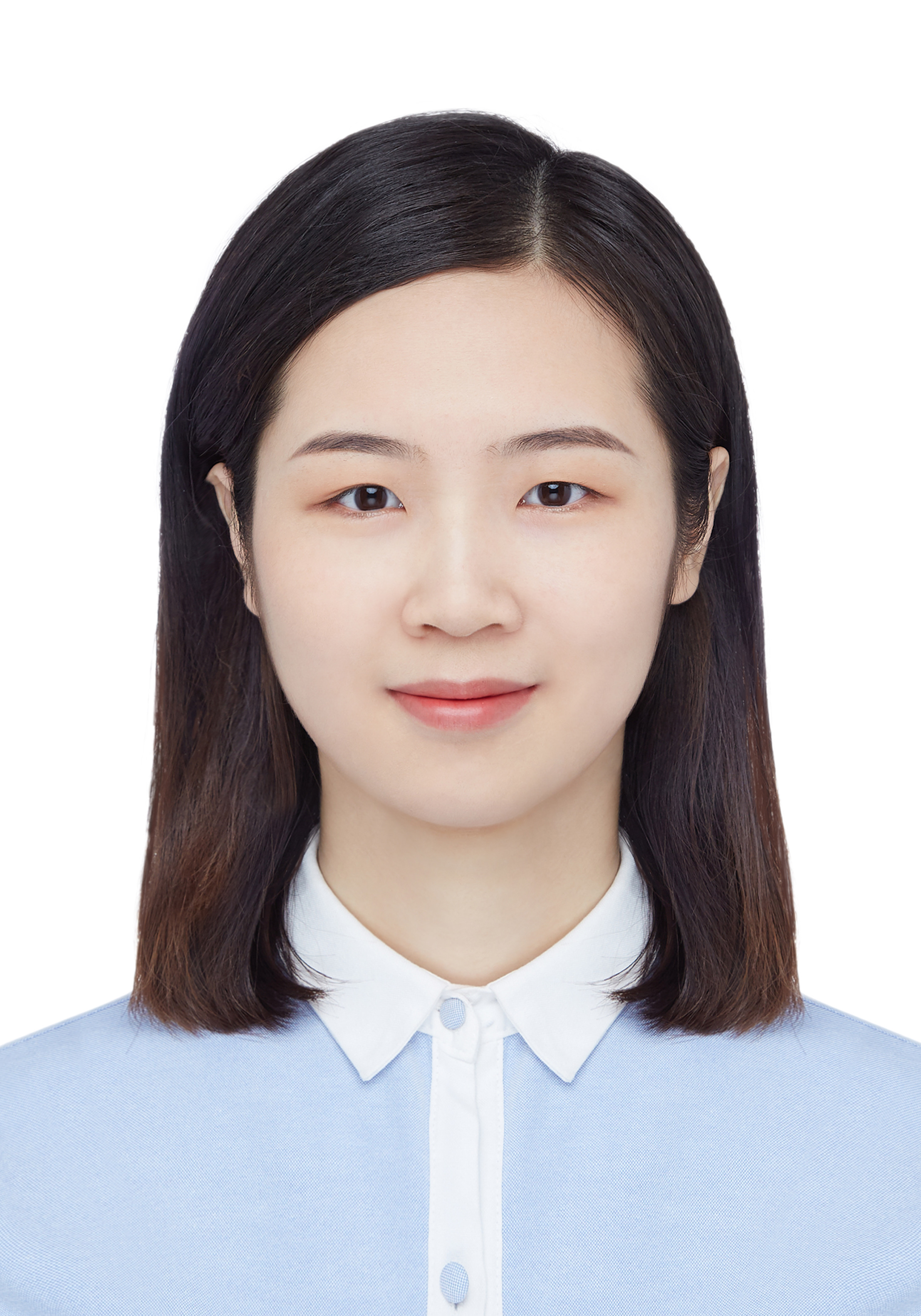}}]
	{\textbf{Yangbangyan Jiang} received the bachelor's degree in instrumentation and control from Beihang University in 2017. She is currently pursuing the Ph.D. degree with University of Chinese Academy of Sciences. Her research interests include machine learning and computer vision. She has authored or coauthored several academic papers in international journals and conferences including NeurIPS, CVPR, AAAI, ACM MM, {etc}. She served as a reviewer for several top-tier conferences such as ICML, NeurIPS, ICLR, CVPR, AAAI.}
\end{IEEEbiography}

\begin{IEEEbiography}[{\includegraphics[width=1in,height=1.25in,clip,keepaspectratio]{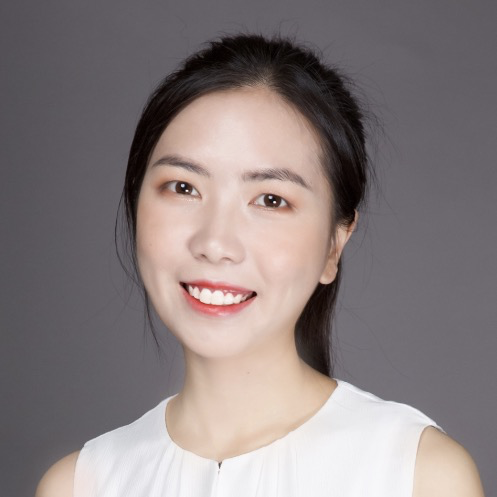}}]
	{\textbf{Xiaodan Li} received the M.S. degree from University of Science and Technology of China in 2019. She is currently an Algorithm Engineer in the Security Department of Alibaba Group, and working on trusted and safe artificial intelligence algorithms. Her research interests include adversarial machine learning and deepfake related topics. She has published several papers in top-tier conferences, such as CVPR/ACM Multimedia, etc.
	}
\end{IEEEbiography}

\begin{IEEEbiography}[{\includegraphics[width=1in,height=1.25in,clip,keepaspectratio]{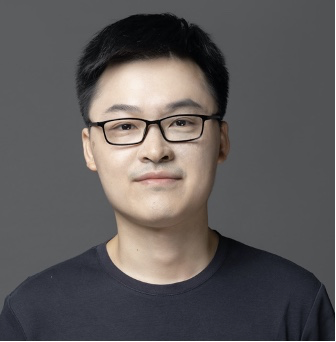}}]
	{\textbf{Yuefeng Chen} received the M.S. degree in computer science from Northwestern Polytechnical University in 2013. He is currently a Staff Engineer in the Security Department of Alibaba, and working on trusted and safe artificial intelligence algorithms. His current research focuses on adversarial machine learning and robust machine learning, including adversarial attack and defense, interpretable machine learning, deepfake related topics. He has published several papers in top-tier conferences, such as CVPR/ICCV/AAAI/ACM Multimedia, etc. He served as a reviewer for AAAI, CVPR, ACM Multimedia, etc.}
\end{IEEEbiography}

\begin{IEEEbiography}[{\includegraphics[width=1in,height=1.25in,clip,keepaspectratio]{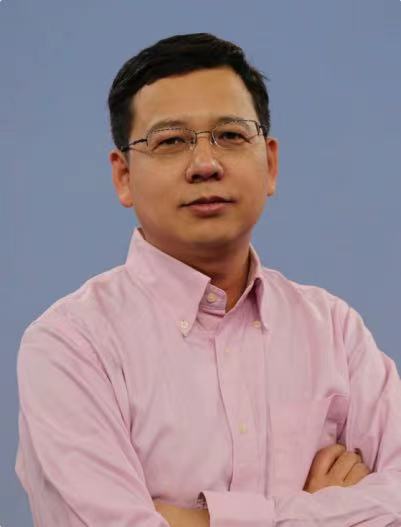}}]
	{\textbf{Yuan He} received his B.S. degree and Ph.D. degree from Tsinghua University, P.R. China. He is a Senior Staff Engineer in the Security Department of Alibaba Group, and working on artificial intelligence-based content moderation and intellectual property protection systems. Before joining Alibaba, he was a research manager at Fujitsu working on document analysis system. He has published more than 30 papers in computer vision and machine learning related conferences and journals including CVPR, ICCV, ICML, NeurIPS, AAAI and ACM MM. His research interests include computer vision, machine learning, and AI security.}
\end{IEEEbiography}

\begin{IEEEbiography}[{\includegraphics[width=1in,height=1.25in,clip,keepaspectratio]{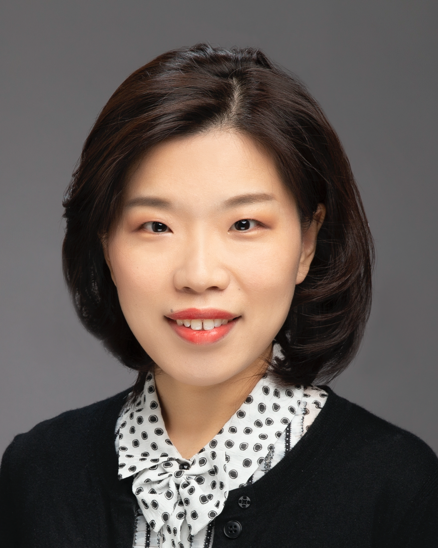}}]
    {\textbf{Qianqian Xu} received the B.S. degree in computer science from China University of Mining and Technology in 2007 and the Ph.D. degree in computer science from University of Chinese Academy of Sciences in 2013. She is currently an Associate Professor with the Institute of Computing Technology, Chinese Academy of Sciences, Beijing, China. Her research interests include statistical machine learning, with applications in multimedia and computer vision. She has authored or coauthored 60+ academic papers in prestigious international journals and conferences (including T-PAMI, IJCV, T-IP, NeurIPS, ICML, CVPR, AAAI, etc.) Moreover, she serves as an associate editor of IEEE Transactions on Circuits and Systems for Video Technology, IEEE Transactions on Multimed, ACM Transactions on Multimedia Computing, Communications, and Applications, and Multimedia Systems.}
\end{IEEEbiography}

\begin{IEEEbiography}[{\includegraphics[width=1in,height=1.25in,clip,keepaspectratio]{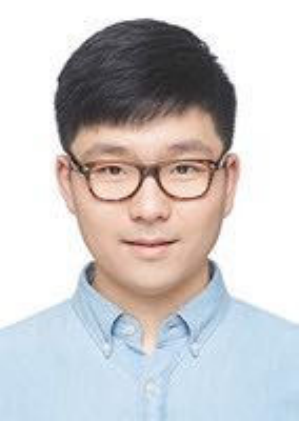}}]
	{\textbf{Zhiyong Yang} received the M.Sc. degree in computer science and technology from University of Science and Technology Beijing (USTB) in 2017, and the Ph.D. degree from University of Chinese Academy of Sciences (UCAS) in 2021. He is currently a postdoctoral research fellow with the University of Chinese Academy of Sciences. His research interests lie in machine learning and learning theory, with special focus on AUC optimization, meta-learning/multi-task learning, and learning theory for recommender systems. He has authored or coauthored 32 academic papers in top-tier international conferences and journals including T-PAMI/ICML/NeurIPS/CVPR. He served as a TPC member for IJCAI 2021 and a reviewer for several top-tier journals and conferences such as T-PAMI, TMLR, ICML, NeurIPS and ICLR.}
\end{IEEEbiography}

\begin{IEEEbiography}[{\includegraphics[width=1in,height=1.25in,clip,keepaspectratio]{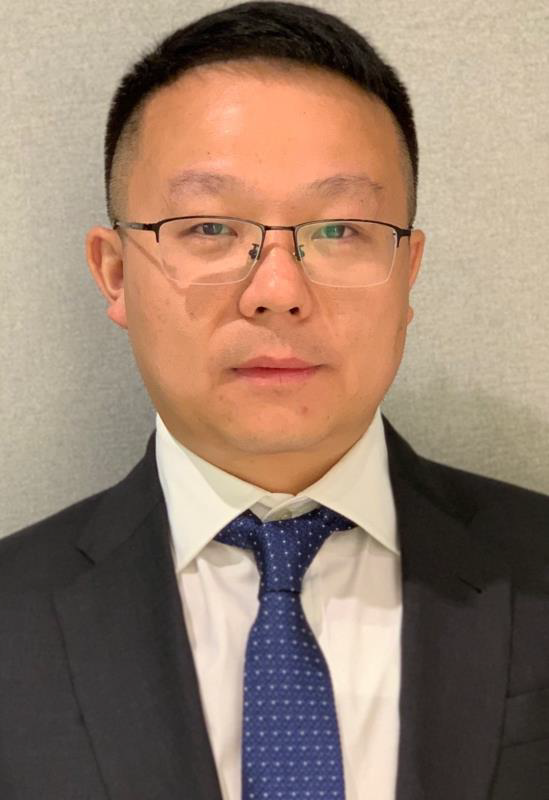}}]
    {\textbf{Xiaochun Cao} is a Professor of School of Cyber Science and Technology, Shenzhen Campus, Sun Yat-sen University. He received the B.E. and M.E. degrees both in computer science from Beihang University (BUAA), China, and the Ph.D. degree in computer science from the University of Central Florida, USA, with his dissertation nominated for the university level Outstanding Dissertation Award. After graduation, he spent about three years at ObjectVideo Inc. as a Research Scientist. From 2008 to 2012, he was a professor at Tianjin University. From 2012 to 2020, he was a professor at Institute of Information Engineering, Chinese Academy of Sciences. He has authored and coauthored over 200 journal and conference papers. In 2004 and 2010, he was the recipients of the Piero Zamperoni best student paper award at the International Conference on Pattern Recognition. He is a fellow of IET and a Senior Member of IEEE. He is an associate editor of IEEE Transactions on Image Processing, IEEE Transactions on Circuits and Systems for Video Technology and IEEE Transactions on Multimedia.}
\end{IEEEbiography}

\begin{IEEEbiography}[{\includegraphics[width=1in,height=1.25in,clip,keepaspectratio]{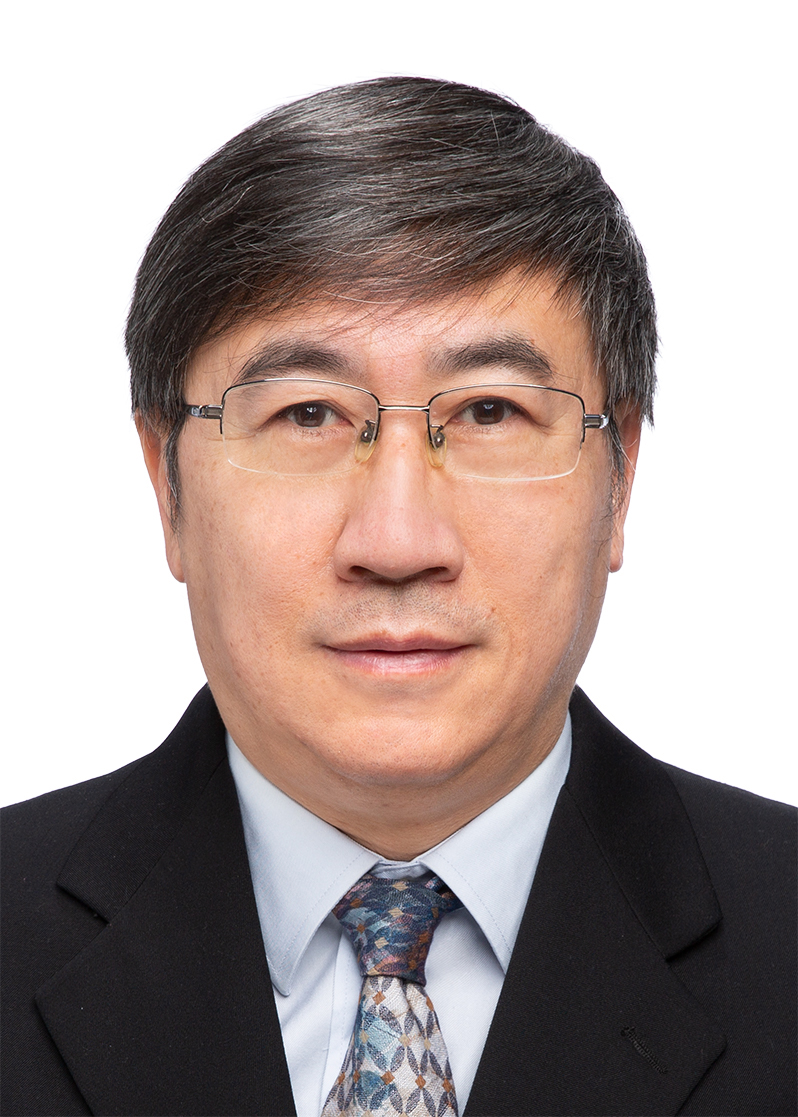}}]
    {\textbf{Qingming Huang} is a chair professor in University of Chinese Academy of Sciences and an adjunct research professor in the Institute of Computing Technology, Chinese Academy of Sciences. He graduated with a Bachelor degree in Computer Science in 1988 and Ph.D. degree in Computer Engineering in 1994, both from Harbin Institute of Technology, China. His research areas include multimedia computing, image processing, computer vision and pattern recognition. He has authored or coauthored more than 400 academic papers in prestigious international journals and top-level international conferences. He was the associate editor of IEEE Trans. on CSVT and is currently an associate editor of Acta Automatica Sinica and the reviewer of various international journals including IEEE Trans. on PAMI, IEEE Trans. on Image Processing, IEEE Trans. on Multimedia, etc. He is a Fellow of IEEE and has served as general chair, program chair, track chair and TPC member for various conferences, including ACM Multimedia, CVPR, ICCV, ICME, ICMR, PCM, BigMM, PSIVT, etc.}
\end{IEEEbiography}





\clearpage
\onecolumn
\appendices

\section{Preliminaries}

\begin{table}[htbp]
	\centering
	\caption{\label{tab:note}Notations and descriptions.}
	\begin{tabular}{l|p{0.78\columnwidth}}
		\toprule
		Notation & Description  \\
		\midrule
		\multicolumn{2}{c}{Basic} \\
		\midrule
		$\theta$ & All the model parameters \\
		$[n]$ & The set $\{1,\cdots,n\}$ \\
		$\mathbb{I}[\cdot]$ & Indicator function, which has a value of 1 if the input condition holds and 0 otherwise. \\
		$\nc$ & The number of classes \\
		$\nl$ & The number of labeled samples \\
		$\nuu$ & The number of unlabeled samples \\
		$\DL$ & Labeled set \\
		$\DU$ & Unlabeled set \\
		$\setD_X$ & Data distribution \\
		$\setD_{XY}$ & Joint distribution of data and labels \\
		$f/f_\theta$ & The model (parameterized by $\theta$), which has component functions $f^{(k)}$, \ie, $f=(f^{(1)}, \cdots, f^{(\nc)})$ \\
		$\norm{\cdot}_2$ & The operator norm of a matrix, or the $\ell_2$-norm of a vector \\
		$y(s)$ & Predicted label of a score vector, $y(s)=\argmax_{i\in[\nc]} s_i$ \\
		\midrule
		\multicolumn{2}{c}{CNN} \\
		\midrule
		$\chi$ & All the network input $x$ have operator norms at most $\chi$, \ie, $\norm{vec(x)}_2 \le \chi$. \\
		$\nu$ & All the initialized parameter matrices have operator norms at most $1+\nu$. \\
		$L_c$ & Total number of convolutional layers \\
		$L_f$ & Total number of fully-connected layers \\
		$L_N$ & Total number of network layers, \ie, $L_N = L_c + L_f$. \\
		$V^{(i)}$ & Weight matrix of $i$-th fully-connected layer \\
		$U^{(i)}$ & $i$-th convolutional kernel \\
		$op(U)$ & The operator matrix of a convolutional kernel, which expresses the convolution as a matrix-vector product. \\
		$\setO_{\beta,\nu}$ & The set of learned parameters where the distance between learned and initialized parameters is no larger than $\beta$. \\
		$\setFbu$ & The set of learned CNNs parameterized by parameters in $\setO_{\beta,\nu}$ \\
		$\setFbu^0$ & The component class of $\setFbu$ \\
		$W,W_g$ & Total number of parameters for networks in $\setFbu$, or for networks that replace the last layer of those in $\setFbu$ with a single-output layer \\
		$d_N(\theta,\tilde{\theta})$ & The distance between two networks, $d_N(\theta, \tilde{\theta}) = \sum^{L_c}_{i=1} \norm{op(U^{(i)}) - op(\tilde{U}^{(i)})}_2 + \sum^{L_f}_{i=1} \norm{V^{(i)} - \tilde{V}^{(i)}}_2$ \\
		\midrule
		\multicolumn{2}{c}{Loss, Risk \& Bound} \\
		\midrule
		$\tau$ & A transformation function \\
		$\setT$ & The set of input transformation functions \\
		$\setT_\setD$ & The set of input transformation functions for the dataset $\setD=\{x_i\}^n_{i=1}$ where each element $\tau_i$ is applied to $x_i$. \\
		$\chi_\tau$ & All the output of the transformation function have operator norms at most $\chi_\tau$, \ie, $\norm{vec(\tau(x))}_2 \le \chi_\tau$. \\
		$\lzo$ & The zero-one loss $\lzo(y', y) = \mathbb{I}\left[y'\neq y\right]$ \\
		$\lce$ & The cross-entropy loss of a vector based on the hard label $\lce(s, y) = \log(\sum_{i\in[\nc]}\exp(s_i - s_y))$, or based on the soft label $\lsce(s, t) = -\sum_{i\in[\nc]}t_i \cdot \log s_i$ \\
		$\ell^\B$ & $(\ell^\B\circ f)(x) = \sup_{x'\in\B(x)} \ell(f(x), f(x')).$ \\
		$\B(x)$ & The uncertainty set of a point $x$ \\
		$\sigma_i, \sigma_i^{(j)}$ & Rademacher random variables with $\mathbb{P}[\sigma_i=1] = \mathbb{P}[\sigma_i=-1] = 0.5$ and $\mathbb{P}[\sigma_i^{(j)}=1] = \mathbb{P}[\sigma_i^{(j)}=-1] = 0.5$. \\
		$\RL$ & Empirical risk with the surrogate loss on the labeled set, $\RL = \frac{1}{\nl} \sum_{i=1}^\nl \lce(f(x_i^l), y_i)$ \\
		$\RU$ & Empirical risk with the surrogate loss on the unlabeled set, $\RU =  \frac{1}{\nuu} \sum_{i=1}^\nuu \sup_{x'\in\B_K(x_i^u)} \lsce(f(x'), f(x_i^u))$ \\
		$\RB$ & The robust zero-one risk with $\B$, $\RB = \E_{(x, y) \sim \setD_{XY}} \left[\sup_{x'\in\B(x)}\ell^{0/1}(f(x'), y)\right]$ \\
		$\dinf(f,\fnew)$ & A pseudo-metric for two functions, $\dinf(f,\fnew) = \max_{z\in\mathcal{Z}_\setD} |f(z) - \fnew(z)|.$ \\
		$\mathcal{N}(\epsilon,T,d)$ & $\epsilon$-covering number of a set $T$ with respect to the metric $d$ \\
		$C_0, C_1, C_2$ & Universal constants \\
		$\Bse(x), K$ & The semantic uncertainty set of a point $x$ with $K$ elements \\
		$Rad_{\DL}(\ell \circ \setF)$ & Supervised Rademacher complexity of $\ell \circ \setF$ over a labeled dataset $\DL = \{(x_i^l, y_i)\}_{i=1}^\nl$: $Rad_{\DL}(\ell \circ \setF) = \E_{\sigma}[ \sup_{f\in\setF} \frac{1}{\nl} \sum_{i=1}^\nl \sigma_i \cdot \ell(f(x_i^l), y_i) ]$ \\
		$Rad_{\DU}(\ell^\B \circ \setF)$ & Unsupervised Rademacher complexity of $\ell^\B \circ \setF$ over an unlabeled dataset $\DU = \{x_i^u\}_{i=1}^\nuu$: $Rad_{\DU}(\ell^\B \circ \setF) = \E_{\sigma}[ \sup_{f\in\setF} \frac{1}{\nuu} \sum_{i=1}^\nuu \sigma_i \sup_{x'\in\B(x_i^u)} \ell(f(x_i^u), f(x')) ]$ \\
		$Rad_{n, \nc}(\Pi\circ\setF)$ & Empirical Rademacher complexity of multi-output functions, $Rad_{n, \nc}(\Pi\circ\setF) = \E_\sigma \left[\sup_{f=(f^{(1)}, \cdots, f^{(\nc)})\in\setF} \frac{1}{n \cdot \nc} \sum^{\nc}_{j=1} \sum^n_{i=1} \sigma^{(j)}_i \cdot \fj(x_i)\right]$ \\
		$Rad_\setD(\setF\circ\setT_\setD)$ & Empirical Rademacher complexity of single-output functions over a dataset $\setD=\{x_i\}^n_{i=1}$ with corresponding input transformations $\setT_\setD = \{\tau_i\}^n_{i=1}$: $Rad_\setD(\setF\circ\setT_\setD) = \E_\sigma [\sup_{f\in\setF} \frac{1}{n} \sum^n_{i=1} \sigma_i \cdot f(\tau_i(x_i))]$ \\
		\midrule
		\multicolumn{2}{c}{Optimization} \\
		\midrule
		$\mathcal{L}$ & Overall loss function \\
		$\phi(\cdot)$ & $\phi(\cdot) = \max_v\mathcal{L}(\cdot,v)$ \\
		$\phil$ & Moreau envelope of $\phi$ \\
		$B$ & $\max_{\theta \in \Theta}\abs{\phi(\theta) - \min_{\theta\in \Theta} \phi(\theta)} \le B$ \\
		\bottomrule
	\end{tabular}
\end{table}

\begin{lem}[\cite{nesterov2003introductory}]\label{lem:gradLip}
	For a function $h$ which is $\kappa$-gradient Lipschitz, then for any $x,y\in dom(h)$,
	\begin{equation*}
	\abs{h(x) - h(y) - \nabla h(y)^\top (x-y)} \le \frac{\kappa}{2}\norm{x-y}^2_2.
	\end{equation*}
\end{lem}

\begin{fact}[$\kappa$-weakly convex, \cite{equilibrium}]\label{fact:weakcvx}
	If function $f: \setX \times \mathcal{Y} \to \rfield$ is $\kappa$-gradient Lipschitz, then function $\phi(\cdot) := \max_{y\in\mathcal{Y}} f(\cdot, y)$ is $\kappa$-weakly convex; that is, $\phi(x) + \frac{\kappa}{2}\norm{x}^2$ is a convex function over $x$.
\end{fact}

\begin{defn}[$(B,d)$-Lipschitz parameterized, \cite{cnn_bound}]
	\label{def:lip-param}
	For $d\in \mathbb{N}$, a set $G$ of real-valued functions with a common domain $Z$, we say that $G$ is $(B,d)$-Lipschitz parameterized if there is a norm $\norm{\cdot}$ on $\rfield^d$ and a mapping $\psi$ from the unit ball w.r.t. $\norm{\cdot}$ in  $\rfield^d$ to $G$ such that, for all $\theta$ and $\theta'$ such that $\norm{\theta}\le 1$ and $\norm{\theta'}\le 1$ and all $z\in Z$,
	\begin{equation*}
		\abs{(\psi(\theta))(z) - (\psi(\theta'))(z)} \le B\norm{\theta - \theta'}.
	\end{equation*}
\end{defn}

\begin{defn}[Bounded Difference Property]
	\label{def:bdp}
	Given independent random variables $x_1, \cdots, x_n$, a function $f(x_1, \cdots, x_n)$ is said to have the bounded difference property if there exist non-negative constants $c_1, \cdots, c_n$, such that:
	\begin{equation*}
		\sup_{x_1, \cdots, x_n, x'_i} \abs{f(x_1, \cdots, x_n) - f(x_1, \cdots, x_{i-1}, x'_i, \cdots, x_n)} \le c_i, \forall 1\le i \le n.
	\end{equation*}
\end{defn}

\begin{lem}[Bounded Difference Inequality]
	\label{lem:bdi}
	Assume that $Z = f(x_1, \cdots, x_n)$ and $x_i$s satisfies the bounded difference property with constants $c_1, \cdots, c_n$. Denote $v = \frac{1}{4} \sum^n_{i=1} c_i^2$, then for all $\lambda>0$ it holds that:
	\begin{equation*}
		\log \E[\exp(\lambda(Z - \E[Z]))] \le \frac{\lambda^2 v^2}{2}.
	\end{equation*}
\end{lem}

\section{Proof of Proposition 1}\label{sec:proof-risk-prop}
\begin{lem}[Relationship between the zero-one loss and cross-entropy loss with hard labels]
	\label{lem:lce}
	For a multi-class classification model, let its output be $s\in\rfield^{\nc}$, the corresponding predicted label be $y(s)=\argmax_{i\in[\nc]} s_i,$ and the target label be $y\in[\nc]$. Denote the 0-1 loss by $$\lzo(y', y) = \mathbb{I}\left[y'\neq y\right],$$ and the cross-entropy loss (also known as multinomial logistic loss) by $$\lce(s, y) = \log(\sum_{i\in[\nc]}\exp(s_i - s_y)).$$ Then we have
	\begin{equation}
		\log 2 \cdot \lzo(y(s), y) \le \lce(s, y).
	\end{equation}
\end{lem}

\begin{proof}
	We give the proof in two cases.

	(1) Suppose $y(s)=y$. In this case, the 0-1 loss $\lzo(y(s), y)=0$. For the cross-entropy loss, when $i=y$, $\exp(s_i - s_y) = 1$, otherwise $\exp(s_i - s_y) \ge 0$. Thus, $$\lce(s, y) = \log(1 + c) \ge \log 2 \cdot \lzo(y(s), y) = 0$$ where $c \ge 0$.

	(2) Suppose $y(s)\neq y$. In this case, the 0-1 loss $\lzo(y(s), y)=1$. For the cross-entropy loss, there exist at least an $i\neq y$ (\eg, $y(s)$) such that $s_i-s_y>0$, \ie, $\exp(s_i - s_y) > 1$. Meanwhile, when $i=y$, $\exp(s_i - s_y) = 1$. Hence we have that $\lce(s, y) = \log(2 + c)$ where $c \ge 0$. In this case, it also holds that $$\log 2 \cdot \lzo(y(s), y) \le \lce(s, y).$$

	Since the two cases exhaust all possibilities, this completes the proof.
\end{proof}

\begin{lem}[Relationship between the zero-one loss and cross-entropy loss with soft labels]
	\label{lem:lsce}
	For a mutli-class classification model, let its output be $s\in[\varepsilon,1-\varepsilon]^{\nc}$ where $\varepsilon$ is a very small number within $(0, 0.5)$, and the corresponding soft label be $t\in[\varepsilon,1-\varepsilon]^{\nc}$. Writing the cross-entropy loss as $$\lce(s, t) = -\sum_{i\in[\nc]}t_i \cdot \log s_i,$$ then we have
	\begin{equation}
		\nc\varepsilon \cdot \log \frac{1}{1-\varepsilon} \cdot \lzo(s, t) \le \lce(s, t).
	\end{equation}
\end{lem}

\begin{proof}
	Since $\lzo \le 1$ and
	\begin{equation*}
		-\sum_{i\in[\nc]}t_i \cdot \log s_i \ge -\sum_{i\in[\nc]}\varepsilon \cdot \log (1-\varepsilon) = \nc\varepsilon \cdot \log \frac{1}{1-\varepsilon} > 0,
	\end{equation*}
	it is easy to see that
	\begin{equation*}
		-\sum_{i\in[\nc]}t_i \cdot \log s_i \ge \nc\varepsilon \cdot \log \frac{1}{1-\varepsilon} \cdot \lzo(s, t).
	\end{equation*}
\end{proof}

\namedprop{1}{
	For any classifier $f$ with bounded output $s\in[\varepsilon, 1-\varepsilon]^\nc$, the following inequality holds:
	\begin{equation}
		\scalemath{1}{
		\begin{aligned}
		&\RB \leq C_1 \cdot \E_{x\sim\setD_{X}} \bigg[{\color{lightred} \sup_{x'\in\B(x)}} {\color{ballblue} \lce\left(f(x'), f(x)\right)}\bigg] +~ C_2 \cdot\E_{(x,y)\sim\setD_{XY}} \bigg[\lce(f(x),y)\bigg],
		\end{aligned}
		}
	\end{equation}
	where $C_1 = \frac{1}{\nc\varepsilon \cdot \log \frac{1}{1-\varepsilon}}$, $C_2 = \frac{1}{\log 2}$, and $\varepsilon\in(0,0.5)$.
}{Robust Risk Decomposition}

\begin{proof}
	First of all, we have that:
	\begin{equation}
	\begin{split}
		\mathbb{I}\left[y(f(x'))\neq y\right] &\le \mathbb{I}\left[y(f(x'))\neq y(f(x))\right] + \mathbb{I}\left[y(f(x))\neq y\right] \\
		&\le \mathbb{I}\left[f(x')\neq f(x)\right] + \mathbb{I}\left[y(f(x))\neq y\right] \\
		&\overset{(*)}{\le} C_1 \cdot \lsce(f(x'), f(x)) + C_2 \cdot \lce(f(x), y),
	\end{split}
	\end{equation}
	where $(*)$ comes from \Lemref{lem:lce} and \ref{lem:lsce}.

	Taking the corresponding expectation,
	\begin{equation*}
		\begin{aligned}
			\RB = \E_{(x, y) \sim \setD_{XY}} \left[\sup_{x'\in\B(x)}\ell^{0/1}(y(f(x')), y)\right] &\leq \E_{(x,y)\sim\setD_{XY}} \bigg[\sup_{x'\in\B(x)} C_1 \cdot \lsce\left(f(x'), f(x)\right) + C_2 \cdot \lce(f(x),y)\bigg]\\
			&= \E_{x\sim\setD_{X}} \bigg[\sup_{x'\in\B(x)} C_1 \cdot \lsce\left(f(x'), f(x)\right)\bigg] + \E_{(x,y)\sim\setD_{XY}} \bigg[C_2 \cdot \lce(f(x),y)\bigg].
		\end{aligned}
	\end{equation*}
\end{proof}

\section{Proof of Theorem 1}\label{sec:proof-general-bound}
\begin{proof}
	From \Propref{prop:risk}:
	\begin{equation*}
		\RB \leq C_1 \cdot \underbrace{ \E_{x\sim\setD_{X}} \bigg[{\sup_{x'\in\B(x)}} {\lce\left(f(x'), f(x)\right)}\bigg]}_{(a)} +~ C_2 \cdot\underbrace{\E_{(x,y)\sim\setD_{XY}} \bigg[\lce(f(x),y)\bigg]}_{(b)},
	\end{equation*}
	where $C_1 = \frac{1}{\nc\varepsilon \cdot \log \frac{1}{1-\varepsilon}}$ and $C_2 = \frac{1}{\log 2}$.

	Obviously $\RU$ and $\RL$ are unbiased estimations of $(a)$ and $(b)$, \ie:
	\begin{equation*}
	\begin{split}
	(a) &= \mathbb{E}_{\DU}\left[\RU\right], \\
	(b) &= \mathbb{E}_{\DL}\left[\RL\right].
	\end{split}
	\end{equation*}

	Then following the generalization bound in \cite{foundml}, we have, with probability at least $1-\delta/2$, the following inequality holds:
	\begin{equation}
		(a) = \mathbb{E}_{\DU}\left[\RU\right] \le \RU + 2 Rad_{\DU}(\lce\circ\setF) + 3\sqrt{\frac{\log\frac{4}{\delta}}{2\nuu}},
	\end{equation}
	meanwhile, another inequality also holds with probability at least $1-\delta/2$:
	\begin{equation}
		(b) = \mathbb{E}_{\DL}\left[\RL\right] \le \RL + 2 Rad_{\DL}(\lce^\B\circ\setF) + 3\sqrt{\frac{\log\frac{4}{\delta}}{2\nl}},
	\end{equation}
	Putting all together, the proof is then finished.
\end{proof}

\section{Proof of Theorem 2}\label{sec:proof-cnn-bound}
\subsection{Upper Bound for Rademacher Complexity of Multi-output Functions}
\begin{lem}
	\label{lem:subgauss}
	Given an input feature set $\setD=\{x_i\}^n_{i=1}$, we first define $\mathcal{Z}_\setD = \{(x,y): x\in\setD, y\in[\nc]\}$ as the cartesian product between the input feature set and the label space. By rewriting $f$ as $f=(f^{(1)}, \cdots, f^{(\nc)})\in\setF$, for each $z=(x,y)\in\mathcal{Z}_\setD$ we have that $f(z) = f^{(y)}(x)$. Define $$\xi_f(\sigma) = \frac{1}{n \cdot \nc} \sum^{\nc}_{j=1} \sum^n_{i=1} \sigma^{(j)}_i \cdot \fj(x_i),$$
	and the pseudo-metric $$\dinf(f,\fnew) = \max_{z\in\mathcal{Z}_\setD} |f(z) - \fnew(z)|.$$
	Then for all $f,\fnew\in \setF$ and all $\lambda > 0$, the following holds:
	\begin{equation}
		\E_\sigma \left[\exp\big(\lambda C_S \cdot \big(\xi_f(\sigma) - \xi_\fnew(\sigma)\big)\big)\right] \le \exp(\frac{\lambda^2 \dinf(f,\fnew)^2}{2}),
	\end{equation}
	where $C_S = \sqrt{n \cdot \nc}$.
\end{lem}

\begin{proof}
	We first denote $\sigma$ as a set of Rademacher random variables with $\sigma = (\sigma^{(1)}_1, \cdots, \sigma^{(1)}_n, \cdots, \sigma^{(\nc)}_1, \cdots, \sigma^{(\nc)}_n)$, and $\tilde{\sigma}_{[i,j]}$ that replaces $\sigma^{(j)}_i$ in $\sigma$ with another Rademacher random variable $\tilde{\sigma}^{(j)}_i$.

	Then for all $f,\fnew\in\setF$, the difference between $\xi_f(\sigma) - \xi_\fnew(\sigma)$ and $\xi_f(\tilde{\sigma}_{[i,j]}) - \xi_\fnew(\tilde{\sigma}_{[i,j]})$ is bounded by:
	\begin{equation}
		\begin{aligned}
			\delta_{i,j} &= \abs{\bigl(\xi_f(\sigma) - \xi_\fnew(\sigma)\bigr) - \bigl(\xi_f(\tilde{\sigma}_{[i,j]}) - \xi_\fnew(\tilde{\sigma}_{[i,j]})\bigr)} \\
			&= \frac{1}{n\cdot \nc} \abs{(\sigma^{(j)}_i - \tilde{\sigma}^{(j)}_i) \cdot (\fj(x_i) - \fnew^{(j)}(x_i))} \\
			&\le \frac{2}{n\cdot \nc} \cdot \max_{x_i\in\setD_X} \abs{\fj(x_i) - \fnew^{(j)}(x_i)} \\
			&\le \frac{2}{n\cdot \nc} \cdot \max_{z\in\mathcal{Z}_\setD} \abs{f(z) - \fnew(z)} := \Delta.
		\end{aligned}
	\end{equation}
	Namely, $C_S\cdot(\xi_f(\sigma) - \xi_\fnew(\sigma))$ satisfies the bounded difference property (\Defref{def:bdp}) with the upper bound $C_S\cdot\Delta$. According to \Lemref{lem:bdi}, choosing $$v=\frac{1}{4}\sum^{\nc}_{j=1} \sum^n_{i=1} (C_S \cdot \Delta)^2 = \max_{z\in\mathcal{Z}_\setD} \abs{f(z) - \fnew(z)}$$ completes the proof.
\end{proof}

\setcounter{defn}{5}
\begin{defn}[$\epsilon$-cover and $\epsilon$-covering number]
	An $\epsilon$-cover of a subset $T$ of a pseudo-metric space $(X,d)$ is a set $\hat{T}\subset T$ such that for each $t\in T$ there is a $\hat{t}\in \hat{T}$ such that $d(t, \hat{t}) \le \epsilon$. The $\epsilon$-covering number of $T$, denoted by $\mathcal{N}(\epsilon, T, d)$, is the size of the smallest $\epsilon$-cover of $T$.
\end{defn}

\begin{lem}
	\label{lem:basic_cover_bound}
	Suppose we have a class of multi-output functions $f: \setX \mapsto [\varepsilon,1-\varepsilon]^{\nc}$, it holds that
	\begin{equation}
		Rad_{n, \nc}(\Pi\circ\setF) \le \inf_{\alpha>0} \Big(4\alpha + \frac{12}{\sqrt{n \cdot \nc}} \int^1_\alpha \sqrt{\log \mathcal{N}(\epsilon, \setF, \dinf)} d\epsilon \Big).
	\end{equation}
\end{lem}

\begin{proof}
	The proof follows a similar spirit from Theorem 5 in \cite{mauc} based on \Lemref{lem:subgauss}.
\end{proof}

\subsection{Upper Bound for Rademacher complexity of Multi-output Deep Convolutional Networks}
Now we move to bound the Rademacher complexity of multi-output deep convolutional networks.

\begin{lem}
	\label{lem:cnn_lip}
	For any $\theta, \tilde{\theta} \in \setO_{\beta,\nu}$ and any feasible input $x$,
	\begin{equation}
		\norm{f_\theta(x) - f_{\tilde{\theta}}(x)} \le \chi(1+\nu+\beta/L_N)^{L_N} d_N(\theta, \tilde{\theta}).
	\end{equation}
\end{lem}

\begin{proof}
	The proof is similar to Lemma 3.4 in \cite{cnn_bound}.
\end{proof}

\namedlem{1}{
	The Rademacher complexity of multi-output deep convolutional networks could be bounded by
	\begin{equation}
		Rad_{n, \nc}(\Pi\circ\setFbu) \le \frac{4}{\sqrt{n\cdot \nc}} + \frac{12}{\sqrt{n \cdot \nc}} \sqrt{W \cdot \log(C_N \sqrt{n\nc})},
	\end{equation}
	where $C_N = 3\chi \cdot e^{\frac{\beta}{1+\nu}}$.
}

\begin{proof}
	Based on \Lemref{lem:cnn_lip},
	\begin{equation*}
		\dinf(f_\theta, f_{\tilde{\theta}}) \le \sup \norm{f_\theta(x) - f_{\tilde{\theta}}(x)} \le \chi(1+\nu+\beta/L_N)^{L_N} d_N(\theta, \tilde{\theta}).
	\end{equation*}
	This implies that the convolutional network is $(\chi(1+\nu+\beta/L)^L, W)$-Lipschitz parameterized (\Defref{def:lip-param}). Applying Lemma A.8 in \cite{cnn_bound}, we have
	\begin{equation}
		\log \mathcal{N}(\epsilon, \setFbu, \dinf) \le W \cdot \log\frac{3\chi(1+\nu+\beta/L_N)^{L_N}}{\epsilon}.
	\end{equation}

	Based on \Lemref{lem:basic_cover_bound},
	\begin{equation}
		\begin{aligned}
			Rad_{n, \nc}(\Pi\circ\setFbu) &\leq \inf_{\alpha>0} \Big(4\alpha + \frac{12}{\sqrt{n\cdot \nc}} \int^1_\alpha \sqrt{\log \mathcal{N}(\epsilon, \setF, \dinf)} d\epsilon \Big) \\
			&\leq \inf_{\alpha>0} \Big(4\alpha + \frac{12}{\sqrt{n\cdot \nc}} \int^1_\alpha \sqrt{W \cdot \log\frac{3\chi(1+\nu+\beta/L_N)^{L_N}}{\epsilon}} d\epsilon \Big).
		\end{aligned}
	\end{equation}

	It is easy to see that the function to be integrated is monotonically decreasing within $[\alpha, 1]$; thus,
	\begin{equation}
		\int^1_\alpha \sqrt{W \cdot \log(\frac{3\chi(1+\nu+\beta/L_N)^{L_N}}{\epsilon}}) d\epsilon \le \sqrt{W \cdot \log\frac{3\chi(1+\nu+\beta/L_N)^{L_N}}{\alpha}}
	\end{equation}

	We could then choose $\alpha = \frac{1}{\sqrt{n\cdot \nc}}$ and obtain
	\begin{equation*}
		Rad_{n, \nc}(\Pi\circ\setFbu) \le \frac{4}{\sqrt{n\cdot \nc}} + \frac{12}{\sqrt{n\cdot \nc}} \sqrt{W \cdot \log(3\chi(1+\nu+\beta/L_N)^{L_N}\sqrt{n\nc})}.
	\end{equation*}
	Since $(1+\nu+\beta/L_N)^{L_N} \le e^{\frac{\beta}{1+\nu}}$, this completes the proof.
\end{proof}

\subsection{Proof of Theorem 2}
\begin{lem}[Vector Version of Talagrand's Contraction Lemma]
	\label{lem:talagrand}
	Let $\setF$ be a class of functions $f=(f^{(1)},\cdots,f^{({M})})$ which maps the input onto $\rfield^{M}$ over a set $\setD=\{x_i\}_{i=1}^n$, and $h: \rfield^{M}\mapsto\rfield$ is an $L$-Lipschitz function, then
	\begin{equation*}
		\E_\sigma \bigg[\sup_{f\in\setF} \sum_{i=1}^n \sigma_i h(f(x_i))\bigg] \le \sqrt{2}L \cdot \E_\sigma \bigg[\sup_{f\in\setF} \sum_{i=1}^n \sum_{j=1}^M \sigma_{i,j} \fj(x_i)\bigg],
	\end{equation*}
	where $\{\sigma_{i}\}_{i}$ and $\{\sigma_{i,j}\}_{i,j}$ are two sequences of independent Rademacher random variables.
\end{lem}

\begin{lem}
	\label{lem:log_lip}
	For $x,y\in[\varepsilon,1-\varepsilon]$ where $\varepsilon\in(0,0.5)$, $h(x,y)=-x\cdot\log y$ is $\sqrt{2}\frac{1-\varepsilon}{\varepsilon}$-Lipschitz continuous.
\end{lem}

\begin{proof}
	Since $x,y\in[\varepsilon,1-\varepsilon]$, we can see that both the partial derivatives of $h$ are bounded:
	\begin{equation*}
		\begin{aligned}
			L_x = \sup_{x\in[\varepsilon,1-\varepsilon]} \norm{\nabla_x h(x, y)} &= \sup_{x\in[\varepsilon,1-\varepsilon]} \norm{-\log y} = \log \frac{1}{\varepsilon} \\
			L_y = \sup_{x\in[\varepsilon,1-\varepsilon]} \norm{\nabla_y h(x, y)} &= \sup_{x\in[\varepsilon,1-\varepsilon]} \norm{-\frac{x}{y}} = \frac{1-\varepsilon}{\varepsilon}.
		\end{aligned}
	\end{equation*}
	Then
	\begin{equation*}
		\begin{aligned}
			\abs{h(x_1, y_1) - h(x_2, y_2)} &\le \abs{h(x_1, y_1) - h(x_1, y_2) + h(x_1, y_2) - h(x_2, y_2)} \\
			&\le \abs{h(x_1, y_1) - h(x_1, y_2)} + \abs{h(x_1, y_2) - h(x_2, y_2)} \\
			&\overset{(**)}{\le} L_x \cdot \abs{y_1 - y_2} + L_y \cdot \abs{x_1 - x_2} \\
			&\le \max\Big\{\log \frac{1}{\varepsilon}, \frac{1-\varepsilon}{\varepsilon}\Big\} \cdot (\abs{y_1 - y_2} + \abs{x_1 - x_2}) \\
			&\overset{(***)}{=} \frac{1-\varepsilon}{\varepsilon} \cdot (\abs{y_1 - y_2} + \abs{x_1 - x_2}) \\
			&\overset{(****)}{\le} \sqrt{2}\frac{1-\varepsilon}{\varepsilon} \cdot \sqrt{(x_1-x_2)^2+(y_1-y_2)^2} \\
			&= \sqrt{2}\frac{1-\varepsilon}{\varepsilon} \cdot \norm{(x_1,y_1)-(x_2,y_2)},
		\end{aligned}
	\end{equation*}
	where $(**)$ is based on the mean value theorem; $(***)$ is due to $\log \frac{1}{\varepsilon} - \frac{1-\varepsilon}{\varepsilon} = \log\frac{\exp(\frac{1}{\varepsilon}-1)}{\varepsilon} > 0$ when $\varepsilon\in(0, 0.5)$; and $(****)$ comes from the Cauchy-Schwartz inequality.
\end{proof}

\namedlem{2}{
	Denote by $W_g$ the number of parameters for convolutional networks that replace the last layer of $f\in\setFbu$ with a single-output layer. Let the output of all the transformation functions $\tau$ be constrained by $\norm{vec(\tau(x))}\le\chi_\tau$, then the Rademacher complexity of $\setFbu^0$ without or with input transformation $\setT_\setD = \{\tau_i\}^n_{i=1}$ could be bounded by
	\begin{equation}
		\begin{aligned}
			Rad_\setD(\setFbu^0) &\le \frac{4}{\sqrt{n}} + \frac{12}{\sqrt{n}} \sqrt{W_g \cdot \log(C_N n)}, \\
			Rad_\setD(\setFbu^0\circ\setT_\setD) &\le \frac{4}{\sqrt{n}} + \frac{12}{\sqrt{n}} \sqrt{W_g \cdot \log(C_{N,\tau} n)},
		\end{aligned}
	\end{equation}
	respectively, where $C_N = 3\chi \cdot e^{\frac{\beta}{1+\nu}}$ and $C_{N,\tau} = 3\chi_\tau \cdot e^{\frac{\beta}{1+\nu}}$.
}

\begin{proof}
	The proof follows the similar spirit of \Lemref{lem:cover_bound}.
\end{proof}

We could then provide the following generalization bound.

\setcounter{thm}{1}
\namedthm{2}{
	Let $\setFbu$ be the hypothesis class of multi-output deep convolutional networks. Define the empirical risk on labeled and unlabeled set as
	\begin{equation*}
		\begin{aligned}
			\RL &= \frac{1}{\nl} \sum_{i=1}^\nl \lce(f(x_i^l), y_i), \\
			\RU &=  \frac{1}{\nuu} \sum_{i=1}^\nuu \sup_{x'\in\Bse(x_i^u)} \lsce(f(x'), f(x_i^u)).
		\end{aligned}
	\end{equation*}
	For any function $f\in\setFbu$, the following inequality holds with probability at least $1-\delta$ over the random draw of $\DL$ and $\DU$:
	\begin{equation}
		\scalemath{1}{
		\begin{aligned}
		\RB \leq\, & C_1 \cdot \RU + C_2(1+\frac{C_0}{2}) \cdot \RL \\
		&+ C_1 \biggl( \frac{16K\nc}{\sqrt{\nuu}}\cdot\frac{1-\varepsilon}{\varepsilon} \bigg[ 1 + 3 \sqrt{W_g \cdot \log(C_M \nuu)} \bigg] + 3\sqrt{\frac{\log \frac{4}{\delta}}{2\nuu}} \biggr) + \frac{3C_0C_2}{2}\cdot \Psi_{\nl,\nc,\delta}(\setFbu),
		\end{aligned}
		}
	\end{equation}
	where $C_0>0$ is a universal constant, $C_1 = \frac{1}{\nc\varepsilon \cdot \log \frac{1}{1-\varepsilon}}$, $C_2 = \frac{1}{\log 2}$, $C_N = 3\chi \cdot e^{\frac{\beta}{1+\nu}}$, $C_M = 3\max\{\chi,\chi_\tau\} \cdot e^{\frac{\beta}{1+\nu}}$ and
	\begin{equation*}
		\begin{aligned}
			\psi_{\nl, \nc}(\setFbu) &= \frac{4}{\sqrt{\nl\cdot \nc}} + \frac{12}{\sqrt{\nl\cdot \nc}} \sqrt{W \cdot \log(C_N \sqrt{\nl\nc})}, \\
			\Psi_{\nl, \nc,\delta}(\setFbu) &= 2 \left( \sqrt{\nc} \cdot \log^{3/2}(\nl \nc e) \cdot \psi_{\nl, \nc}(\setFbu) + \frac{1}{\sqrt{\nl}} \right) + \frac{\log(\nc e)}{\nl} \left(\log\frac{2}{\delta} + \log(\log\ \nl)\right).
		\end{aligned}
	\end{equation*}
}{Robust Generalization Bound for SSL using CNNs}

\begin{proof}
	According to \Propref{prop:risk},
	\begin{equation}
		\scalemath{1}{
		\begin{aligned}
		&\RB \leq C_1 \cdot \underbrace{ \E_{x\sim\setD_{X}} \bigg[\sup_{x'\in\B(x)} \lsce\left(f(x'), f(x)\right)\bigg]}_{(a)} + C_2 \cdot \underbrace{\E_{(x,y)\sim\setD_{XY}} \bigg[\lce(f(x),y)\bigg]}_{(b)},
		\end{aligned}
		}
	\end{equation}

	Obviously $\RU$ and $\RL$ are unbiased estimations of $(a)$ and $(b)$, \ie,
	\begin{equation*}
	\begin{split}
	(a) &= \E_{\DU}\left[\RU\right], \\
	(b) &= \E_{\DL}\left[\RL\right].
	\end{split}
	\end{equation*}

	Since $s\in[\varepsilon,1-\varepsilon]^{\nc}$, we have $range(\lce) \subseteq [\log((\nc-1)e^{2\varepsilon-1} + 1), \log((\nc-1)e^{1-2\varepsilon} + 1)]$. Then following Theorem 1 of \cite{multioutput_bound} and \Lemref{lem:cover_bound}, we have that, with probability at least $1-\frac{2}{\delta}$, the following inequality holds:
	\begin{equation}
		\begin{aligned}
			(b) = \E_{\DL}\left[\RL\right] &\le \RL + C_0 \left( \sqrt{\RL \cdot \Psi_{\nl,\nc,\delta}(\setFbu)} + \Psi_{\nl,\nc,\delta}(\setFbu) \right), \\
			&\le \RL + C_0 \left( \frac{1}{2}\Big(\RL + \Psi_{\nl,\nc,\delta}(\setFbu)\Big) + \Psi_{\nl,\nc,\delta}(\setFbu) \right), \\
			&= (1+\frac{C_0}{2})\cdot\RL + \frac{3C_0}{2}\cdot \Psi_{\nl,\nc,\delta}(\setFbu),
		\end{aligned}
	\end{equation}
	where $C_0>0$ is a universal constant and
	\begin{equation*}
		\begin{aligned}
			\psi_{\nl, \nc}(\setFbu) &= \frac{4}{\sqrt{\nl\cdot \nc}} + \frac{12}{\sqrt{\nl\cdot \nc}} \sqrt{W \cdot \log(C_N \sqrt{\nl\nc})}, \\
			\Psi_{\nl, \nc,\delta}(\setFbu) &= 2 \left( \sqrt{\nc} \cdot \log^{3/2}(\nl \nc e) \cdot \psi_{\nl, \nc}(\setFbu) + \frac{1}{\sqrt{\nl}} \right) + \frac{\log(\nc e)}{\nl} \left(\log\frac{2}{\delta} + \log(\log\ \nl)\right),
		\end{aligned}
	\end{equation*}
	note that here we also use $\log((\nc-1)e^{1-2\varepsilon} + 1) < \log(\nc e)$.

	Meanwhile, for (a), we first have:
	\begin{equation}
		\begin{aligned}
			Rad_{\DU}(\lsce\circ\setFbu) &= \E_\sigma \bigg[\sup_{f\in\setFbu} \frac{1}{\nuu} \sum_{i=1}^\nuu \sigma_i \lsce\left(f(x'), f(x_i)\right)\bigg] \\
			&= \E_\sigma \bigg[\sup_{f\in\setFbu} \frac{1}{\nuu} \sum_{i=1}^\nuu \sigma_i \lsce\left(f(\tau_i(x_i)), f(x_i)\right)\bigg] \\
			&= \E_\sigma \bigg[\sup_{f=(f^{(1)}, \cdots, f^{(\nc)})\in\setFbu} \frac{1}{\nuu} \sum_{i=1}^\nuu \sigma_i \bigg(\sum_{j=1}^{\nc} h\left(\fj(\tau_i(x_i)), \fj(x_i)\right)\bigg)\bigg], \\
			&\le \sum_{j=1}^{\nc} \E_\sigma \bigg[\sup_{\fj\in\setFbu^0} \frac{1}{\nuu} \sum_{i=1}^\nuu \sigma_i h\left(\fj(\tau_i(x_i)), \fj(x_i)\right)\bigg], \\
		\end{aligned}
	\end{equation}
	where $h(x,y)=-x\log y$.

	We could apply \Lemref{lem:talagrand} combined with \Lemref{lem:log_lip} and obtain:
	\begin{equation}
		\begin{aligned}
			\sum_{j=1}^{\nc} \E_\sigma \bigg[\sup_{\fj\in\setFbu^0} \frac{1}{\nuu} \sum_{i=1}^\nuu \sigma_i h\left(\fj(\tau_i(x_i)), \fj(x_i)\right)\bigg] &\le \sum_{j=1}^{\nc} 2\frac{1-\varepsilon}{\varepsilon} \E_\sigma \bigg[\sup_{\fj\in\setFbu^0} \frac{1}{\nuu} \sum_{i=1}^\nuu \big( \sigma_{i,1} \fj(\tau_i(x_i)) + \sigma_{i,2} \fj(x_i) \big) \bigg]\\
			&= 2\frac{1-\varepsilon}{\varepsilon} \sum_{j=1}^{\nc} \Big[ Rad_{\DU}(\setFbu^0\circ\setT_\setD) + Rad_{\DU}(\setFbu^0) \Big].
		\end{aligned}
	\end{equation}

	Here $\setFbu^0$ could be characterized as a class of convolutional networks that have the same architecture as $\setFbu$ except that the output is a scalar used for binary classification. Then using \Lemref{lem:cover_bound} and \ref{lem:b_cover_bound}, it holds that
	\begin{equation}
		\begin{aligned}
			Rad_{\DU}(\lsce\circ\setFbu) &\le 2\nc\frac{1-\varepsilon}{\varepsilon} \Big[ Rad_{\DU}(\setFbu^0\circ\setT_\setD) + Rad_{\DU}(\setFbu^0) \Big] \\
			&\le 2\nc\frac{1-\varepsilon}{\varepsilon} \bigg[ \Big(\frac{4}{\sqrt{\nuu}} + \frac{12}{\sqrt{\nuu}} \sqrt{W_g \cdot \log(C_{N,\tau} \nuu)}\Big) + \Big(\frac{4}{\sqrt{\nuu}} + \frac{12}{\sqrt{\nuu}} \sqrt{W_g \cdot \log(C_N \nuu)}\Big) \bigg] \\
			&\le 2\nc\frac{1-\varepsilon}{\varepsilon} \bigg[ \Big(\frac{4}{\sqrt{\nuu}} + \frac{12}{\sqrt{\nuu}} \sqrt{W_g \cdot \log(C_M \nuu)}\Big) + \Big(\frac{4}{\sqrt{\nuu}} + \frac{12}{\sqrt{\nuu}} \sqrt{W_g \cdot \log(C_M \nuu)}\Big) \bigg] \\
			&= \frac{16\nc}{\sqrt{\nuu}}\cdot\frac{1-\varepsilon}{\varepsilon} \bigg[ 1 + 3 \sqrt{W_g \cdot \log(C_M \nuu)} \bigg],
		\end{aligned}
	\end{equation}
	where $C_N = 3\chi \cdot e^{\frac{\beta}{1+\nu}}$, $C_{N,\tau} = 3\chi_\tau \cdot e^{\frac{\beta}{1+\nu}}$ and $C_M = \max\{C_N, C_{N,\tau}\}$.

	Denote $(\lsce^{\Bse}\circ f)(x) := \sup_{x'\in\Bse(x)} \lsce(f(x'), f(x))$. Then according to Lemma 9.1 of \cite{foundml}, we have that:
	\begin{equation}
		Rad_{\DU}(\lsce^{\Bse}\circ\setFbu) \le K \cdot Rad_{\DU}(\lsce\circ\setFbu).
	\end{equation}

	Following the generalization bound (Theorem 3.3) in \cite{foundml}, we finally have the upper bound for (a) that, with probability at least $1-\frac{\delta}{2}$, the following holds:
	\begin{equation}
		\begin{aligned}
			(a) = \E_{\DU}\left[\RU\right] &\le \RU + 2 Rad_{\DU}(\lsce^{\Bse}\circ\setFbu) + 3\sqrt{\frac{\log \frac{4}{\delta}}{2\nuu}}, \\
			&\le \RU + \frac{16K\nc}{\sqrt{\nuu}}\cdot\frac{1-\varepsilon}{\varepsilon} \bigg[ 1 + 3 \sqrt{W_g \cdot \log(C_M \nuu)} \bigg] + 3\sqrt{\frac{\log \frac{4}{\delta}}{2\nuu}}.
		\end{aligned}
	\end{equation}

	Putting all these together, the proof is then finished.
\end{proof}

\section{Proof of Theorem 3}\label{sec:proof-moreau-bound}
Here we first restate Theorem 3 as follows.
\namedthmn{3}{
	Suppose $\ell^l$ is $\frac{L}{2}$-Lipschitz, $\ell^u$ is $\frac{L}{2\gamma}$-Lipschitz, and then $\mathcal{L}$ is $L$-Lipschitz and $\kappa$-gradient Lipschitz, The parameters $\theta$ are chosen from a compact set $\Theta$\footnote[1]{This could be realized by limiting the norm of weights (say, weight decay).}. Define $\phi(\cdot) := \max_v \mathcal{L}(\cdot,v)$ and $\phil$ is the Moreau envelope of $\phi$. Suppose $\max_{\theta \in \Theta}\abs{\phi(\theta) - \min_{\theta\in \Theta} \phi(\theta)} \le B$. Then with probability at least $1-\delta$, the following inequality holds for $\theta_t$s obtained from \Algref{alg}:
	\begin{equation*}
	\begin{aligned}
	\frac{1}{T+1} \sum_{t=0}^T \norm{\nabla\phil(\xt)}^2 \le  4L ~ \sqrt{\frac{\kappa\big(\phil(\theta_0) - \min_\theta \phi(\theta)\big)}{T}} + 8L ~ \sqrt{\frac{2B\kappa}{T+1}\log\frac{1}{\delta}}
	\end{aligned}
	\end{equation*}
}

\begin{proof}
Since $\mathcal{L}$ is $\kappa$-gradient Lipschitz and $v_t$ is a maximizer for $\xt$, we have that any $\xt$ from \Algref{alg} and $\tilde{\theta}$ satisfy
\begin{equation}
	\begin{split}
		\phi(\tilde{\theta}) \ge &~\mathcal{L}(\tilde{\theta}, v_t)\\
	 \ge &~\Gxyt + \inprod{\gradGxyt, \tilde{\theta} - \xt} - \frac{\kappa}{2} \cdot \norm{\tilde{\theta} - \xt}^2 \\
	 \ge &~\phi(\xt) + \inprod{\gradGxyt, \tilde{\theta} - \xt} - \frac{\kappa}{2} \cdot \norm{\tilde{\theta} - \xt}^2
	\end{split}
\end{equation}
where the second line follows from \Lemref{lem:gradLip} in the sense that if $\abs{h(x) - h(y) - \nabla h(y)^\top (x-y)} \le \frac{\kappa}{2}\norm{x-y}^2_2$, then $-\left(h(x) - h(y) - \nabla h(y)^\top (x-y)\right) \le \frac{\kappa}{2}\norm{x-y}^2_2$.

Thus,
\begin{equation}\label{eq:ineq1}
	\begin{split}
		\inprod{\gradGxyt, \tilde{\theta} - \xt} \le \phi(\tilde{\theta}) - \phi(\xt) + \frac{\kappa}{2} \cdot \norm{\tilde{\theta} - \xt}^2
	\end{split}
\end{equation}

Let
\begin{equation}
	\hatxt := \arg\min_\theta \phi(\theta) + \kappa \cdot \norm{\theta - \xt}^2
\end{equation}

We have:
\begin{equation}
	\begin{split}
		\phil(\theta_{t+1}) = &~ \min_{\theta'} \phi(\theta') + \kappa \cdot  \norm{\theta' - \theta_{t+1}}^2 \\
		\le &~ \phi(\hatxt) + \kappa \cdot \norm{\hatxt - \theta_{t+1}}^2
		=  \phi(\hatxt) + \kappa \cdot \norm{\hatxt - (\xt - \eta g_t)}^2 \\
		= &~ \phi(\hatxt) + \kappa \cdot \norm{\hatxt - \xt}^2 + 2\eta\kappa \cdot \inprod{g_t, \hatxt - \xt} + \eta^2\kappa \cdot \norm{g_t}^2 \\
		= &~ \phi(\hatxt) + \kappa \cdot \norm{\hatxt - \xt}^2 + 2\eta\kappa \cdot \inprod{g_t - \gradGxyt, \hatxt - \xt} + 2\eta\kappa \cdot \inprod{\gradGxyt, \hatxt - \xt} + \eta^2\kappa \cdot \norm{g_t}^2 \\
		\le &~ \phil(\xt) + 2\eta\kappa \cdot \inprod{g_t - \gradGxyt, \hatxt - \xt} + 2\eta\kappa \cdot \left( \phi(\hatxt) - \phi(\xt) + \frac{\kappa}{2} \cdot \norm{\hatxt - \xt}^2 \right) + \eta^2\kappa L^2
	\end{split}
\end{equation}
where $g_t$ denotes the gradient estimated in the $t$-th step (obtained by \texttt{line 7} in Algorithm 1), and the last line follows from (\ref{eq:ineq1}) and the definition of $L$-Lipschitz.

Taking a telescopic sum over $t$, we obtain
\begin{equation}
	\begin{split}
		\phil(\theta_T) \le \phil(\theta_0) + \eta^2\kappa L^2 T + 2\eta\kappa \cdot \sum_{t=0}^T \inprod{g_t - \gradGxyt, \hatxt - \xt} + \left( \phi(\hatxt) - \phi(\xt) + \frac{\kappa}{2} \cdot \norm{\hatxt - \xt}^2 \right)
	\end{split}
\end{equation}

By rearranging this inequality, we obtain
\begin{equation}\label{eq:telescopic}
	\begin{split}
		\frac{1}{T+1} \sum_{t=0}^T \phi(\xt) - \phi(\hatxt) - \frac{\kappa}{2} \cdot \norm{\hatxt - \xt}^2 \le \frac{\phil(\theta_0) - \min_\theta \phi(\theta)}{2\eta\kappa T} + \frac{\eta L^2}{2} + \frac{1}{T+1} \cdot \sum_{t=0}^T \inprod{g_t - \gradGxyt, \hatxt - \xt}
	\end{split}
\end{equation}

Recall that $\mathcal{L}$ is $\kappa$-gradient Lipschitz. According to Fact~\ref{fact:weakcvx}, $\phi(\theta)$ is $\kappa$-weakly convex and $\phi(\theta) + \frac{\kappa}{2}\norm{\theta}^2$ is convex. Then, it is easy to see that $\Phi(\theta) = \phi(\theta) + \kappa \cdot \norm{\xt - \theta}^2$ is $\kappa$-strongly convex. Thus, we have
\begin{equation}\label{eq:moreau_res}
	\begin{split}
		\phi(\xt) - \phi(\hatxt) - \frac{\kappa}{2} \cdot \norm{\hatxt - \xt}^2 \ge &~\phi(\xt) + \kappa \cdot \norm{\xt - \xt}^2 - \phi(\hatxt) - \kappa \cdot \norm{\hatxt - \xt}^2 + \frac{\kappa}{2} \cdot \norm{\hatxt - \xt}^2 \\
		= &~ \phi(\xt) + \kappa \cdot \norm{\xt - \xt}^2 - \min_\theta \left(\phi(\theta) + \kappa \cdot \norm{\xt - \theta}^2\right) + \frac{\kappa}{2} \cdot \norm{\hatxt - \xt}^2 \\
		= &~ \Phi(\xt) - \min_\theta \Phi(\xt) + \frac{\kappa}{2} \cdot \norm{\hatxt - \xt}^2 \\
		\ge &~ \kappa \cdot \norm{\hatxt - \xt}^2 = \frac{1}{4\kappa} \norm{\nabla\phil(\xt)}^2
	\end{split}
\end{equation}
where the second last line comes from the definition of $\kappa$-strongly convex,
\begin{equation}
	\Phi(\xt) \ge \Phi(\hatxt) + \inprod{\nabla \Phi(\hatxt), \xt - \hatxt} + \frac{\kappa}{2} \cdot \norm{\xt - \hatxt}^2
\end{equation}
and the last line follows from an equation for the gradient of Moreau envelope \cite{equilibrium},
\begin{equation}
	\nabla\phi_{\lambda}(\theta) = \lambda^{-1}\left( \theta - \underset{\tilde{\theta}}{\arg\min} \left( \phi(\tilde{\theta}) + \kappa \cdot \norm{\theta - \tilde{\theta}}^2 \right) \right)
\end{equation}

Since $\max_{\theta \in \Theta}\abs{\phi(\theta) - \min_{\theta\in \Theta} \phi(\theta)} \le B$, we have
\begin{equation}
	\begin{split}
		\phi(\hatxt) + \kappa \cdot \norm{\hatxt - \xt}^2 \le &~\phi(\xt) \\
		\norm{\hatxt - \xt}^2 \le &~ \frac{\phi(\xt) - \phi(\hatxt)}{\kappa} \\
		\norm{\hatxt - \xt} \le &~ \sqrt{\frac{\phi(\xt) - \min_\theta \phi(\theta)}{\kappa}} \le \sqrt{\frac{B}{\kappa}}
	\end{split}
\end{equation}

Notice that $g_t$ is an unbiased estimation of $\gradGxyt$, namely we have $\E[g_t]=\gradGxyt$. Let $$G_t := \inprodG{t},$$ then we have
\begin{equation}
	\begin{split}
		\E\left[ G_t | G_{t-1}, G_{t-2}, \ldots, G_{1} \right] = 0, \\
		\E\left[ \sum_{t=0}^T G_t \right] = 0,
	\end{split}
\end{equation}
Namely, $\left\{ G_t \right\}_t$ is a martingale difference sequence. Then using Cauchy-Schwartz's inequality, we have
\begin{equation}
	\begin{split}
		\abs{G_t} \le& \norm{g_t - \gradGxyt} \cdot \norm{\hatxt - \xt} \\
		\le& \sqrt{\frac{B}{\kappa}}\left( \norm{g_t} + \norm{\gradGxyt} \right) \le 2L \sqrt{\frac{B}{\kappa}}
	\end{split}
\end{equation}

Via Azuma's inequality, with probability at least $1-\delta$ we have
\begin{equation}\label{eq:azuma_res}
	\sum_{t=0}^T G_t \le 2L \cdot \sqrt{\frac{B}{\kappa} \cdot 2(T+1)\log\frac{1}{\delta}}
\end{equation}

Substituting (\ref{eq:moreau_res}) and (\ref{eq:azuma_res}) into (\ref{eq:telescopic}), we have for any fixed $\eta>0$ the following inequality holds with probability at least $1-\delta$:
\begin{equation}\label{eq:ineq2}
	\begin{split}
		\frac{1}{T+1} \sum_{t=0}^T \frac{1}{4\kappa} \norm{\nabla\phil(\xt)}^2 \le \frac{\phil(\theta_0) - \min_\theta \phi(\theta)}{2\eta\kappa T} + \frac{\eta L^2}{2} + 2L \cdot \sqrt{\frac{B}{\kappa} \cdot \frac{2}{T+1}\log\frac{1}{\delta}}
	\end{split}
\end{equation}

To reach the tightest bound in this case, we then find $\eta^\star$ such that
\begin{equation}
	\eta^\star = \argmin_{\eta>0} M(\eta) = \frac{\phil(\theta_0) - \min_\theta \phi(\theta)}{2\eta\kappa T} + \frac{\eta L^2}{2}
\end{equation}

Note that
\begin{equation}
	\begin{split}
		M'(\eta) &= -\frac{\phil(\theta_0) - \min_\theta \phi(\theta)}{2\eta^2\kappa T} + \frac{L^2}{2} \\
		M''(\eta) &= \frac{\phil(\theta_0) - \min_\theta \phi(\theta)}{\eta^3\kappa T} \ge 0
	\end{split}
\end{equation}
Namely, $M(\eta)$ is convex. Hence we could solve $M'(\eta) = 0$ to obtain the optimal value, \ie,
\begin{equation}
	\eta^\star = \sqrt{\frac{\phil(\theta_0) - \min_\theta \phi(\theta)}{\kappa L^2 T}}
\end{equation}

Plugging $\eta^\star$ into (\ref{eq:ineq2}) proves the results.
\end{proof}

\end{document}